\documentclass[nohyperref]{article}

\usepackage{microtype}
\usepackage{graphicx}
\usepackage{subcaption}
\usepackage{booktabs} 
\usepackage{multicol}
\usepackage{multirow}
\usepackage{sistyle}
\usepackage{listings} 
\SIthousandsep{,}

\usepackage{hyperref}

\usepackage[accepted]{icml2022_v1}

\usepackage{amsmath}
\usepackage{amssymb}
\usepackage{mathtools}
\usepackage{amsthm}

\usepackage[capitalize,noabbrev]{cleveref}

\theoremstyle{plain}
\newtheorem{theorem}{Theorem}[section]

\theoremstyle{definition}

\theoremstyle{remark}

\usepackage[textsize=tiny]{todonotes}

\icmltitlerunning{Calibrating DNN by Annealing Double-Head (Preprint)}

\begin{document}

\twocolumn[
\icmltitle{\textit{Annealing Double-Head}: An Architecture  for Online Calibration of \\ Deep Neural Networks}

\begin{icmlauthorlist}
\icmlauthor{Erdong Guo}{ucsc,itp}
\icmlauthor{David Draper}{ucsc}
\icmlauthor{Maria De Iorio}{ucl,nus}
\end{icmlauthorlist}

\icmlaffiliation{ucsc}{University of California, Santa Cruz, California, USA}
\icmlaffiliation{itp}{Institute of Theoretical Physics, Chinese Academy of Sciences, Beijng, China}
\icmlaffiliation{ucl}{University College, London, London, UK}
\icmlaffiliation{nus}{National University of Singapore, Singapore}

\icmlcorrespondingauthor{Erdong Guo}{eguo1@ucsc.edu}

\icmlkeywords{Calibration, Neural Networks, Machine Learning, Bayesian Statistical Inference}

\vskip 0.3in
]



\printAffiliationsAndNotice{} 

\begin{abstract}
Model calibration, which is concerned with how frequently the model predicts correctly, not only plays a vital part in statistical model design, but also has substantial practical applications, such as optimal decision-making in the real world.
However, it has been discovered that modern deep neural networks are generally poorly calibrated due to the overestimation (or underestimation) of predictive confidence, which is closely related to overfitting.
In this paper, we propose Annealing Double-Head, a simple-to-implement but highly effective architecture for calibrating the DNN during training. 
To be precise, we construct an additional calibration head-a shallow neural network that typically has one latent layer-on top of the last latent layer in the normal model to map the logits to the aligned confidence.
Furthermore, a simple Annealing technique that dynamically scales the logits by calibration head in training procedure is developed to improve its performance. Under both the in-distribution and distributional shift circumstances, we exhaustively evaluate our Annealing Double-Head architecture on multiple pairs of contemporary DNN architectures and vision and speech datasets. We demonstrate that our method achieves state-of-the-art model calibration performance without post-processing while simultaneously providing comparable predictive accuracy in comparison to other recently proposed calibration methods on a range of learning tasks.
\end{abstract}

\section{Introduction}
\label{sec: intro}
Calibration Error (CE), a measure of predicted uncertainty, is a crucial metric for evaluating statistical learning models.
Despite the predictive accuracy success of Deep Neural Networks (DNNs) in various domains of learning tasks \citep{hinton2012imagenet, lecun2015deep, he2015deep}, it has been observed that DNNs are typically "overconfident" in their predictions, leading to miscalibration \citep{guo2017calibration}. 
A well-calibrated model provides accurate predictive confidence, which is the measurement of the likelihood that the prediction is correct, in addition to the prediction. 
Consequently, model calibration takes a greater role in real-world applications, particularly the deployment of decision-making systems.
For instances, in autonomous driving, human experts must determine whether the prediction can be trusted and whether they should participate in the controlling process based on the confidence predicted by the controlling system to maintain the safety of the driving procedure \citep{bojarski2016end}.
The same principle applies to medical diagnosis and other high-risk tasks \citep{jiang2012calibrating, crowson2016assessing, caruana2015intelligible}. Furthermore, calibrated confidence is beneficial for the out-of-distribution (OOD) analysis \citep{hendrycks2017a, liang2018enhancing, lee2018training, ovadia2019can} and interpretability of the model, as valuable information can be extracted from precise confidence and utilized to comprehend the model's decision. 

As observed by \citeauthor{guo2017calibration}, model capacity and lack of regularization may be major contributors to model miscalibration \citep{niculescu2005predicting}. 
In this direction, diverse methods have been proposed to calibrate the model (during training) using the regularization-like techniques, thereby constraining the prediction's overconfidence \citep{thulasidasan2019mixup, mukhoti2020calibrating, pereyra2017regularizing, muller2019does}. 
Aside from this, the focus of the calibration methods in the post-hoc categories is to learn a parameterized transformation to scale the model's logits using the validation set so that the confidence distribution is better calibrated \citep{platt1999probabilistic, niculescu2005predicting, zadrozny2002transforming, naeini2015obtaining, kull2019beyond, allikivi2020non, wenger2020non, gupta2020calibration}. Furthermore, the idea of adding a smooth approximation term to CE into the objective and then directly optimizing CE in training has recently received considerable interest \citep{kumar2018trainable, karandikar2021soft, cheng2022calibrating, krishnan2020improving}. 

\textbf{Contribution} \hspace{5pt} In this work, we propose an architecture referred as \textit{Annealing Double-Head} (ADH), which is simple to deploy but powerful, to calibrate the DNN during training. 
As depicted in \cref{fig: double_head}, we develop a second shallow neural network, known as the calibration head, to convert the model's last hidden layer's outputs, namely the logits, into calibrated confidences.
In the context that follows, we will refer to the shallow neural network as the calibration head and the original normal network as the main head. 

Our method is based on the observation that \textit{shallow neural networks are well calibrated} but cannot achieve outstanding predictive accuracy due to their limited capacity, while deep neural networks with more hidden layers possess remarkable predictive accuracy but are easily overconfident, resulting in miscalibration \citep{niculescu2005predicting, guo2017calibration}. 
In our architecture, the main head, which has powerful representation ability so to be designed to primarily focuses on predictive accuracy, directs the deep latent layers to provide critical latents for classification, while the calibration head, which is essentially a shallow neural network, converts the latents into calibrated confidences.
Further, we suggest \textit{Annealing}: an online (dynamical) scaling scheme to encourage the calibration head to learn accurate estimate of confidence, which extends the Temperature scaling methods into the in-training setting. 
To be more precise, the confidence estimated by calibration head is amplified by a constant (inverse of Temperature) which decreases or increases smoothly and ends to be one in the training procedure to alleviate the overconfidence or underconfidence effect. 
Finally, we investigate the performance of our architecture in Distributional Shift (DS) scenarios that include classification tasks on Gaussian noise-corrupted datasets and out-of-distribution detection.

\begin{figure}[ht!]
\vskip 0.2in
\begin{center}
\centerline{\includegraphics[width=\columnwidth]{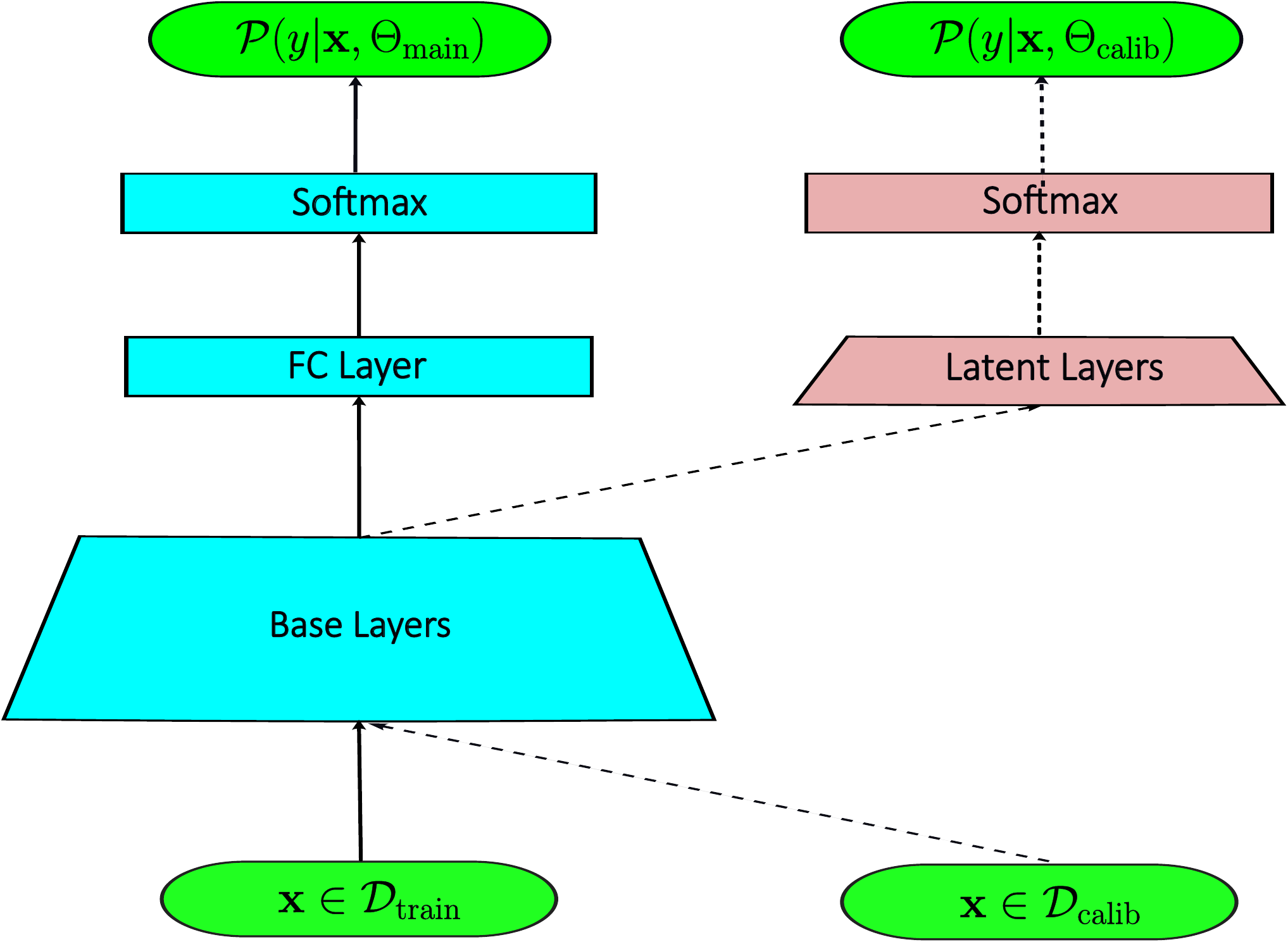}}
\caption{Illustration of Double-Head Annealing. The blue boxes indicate the main head, which is a standard neural network, and the ink boxes represent the calibration head, a shallow neural network. The pipeline of the feed-forward process is represented by dashed arrows, which indicate that the calibration head receives the logits from the main head and returns the calibrated confidences. We remark that calibration head is trained on a distinct dataset referred to as calibration set.}
\label{fig: double_head}
\end{center}
\vskip -0.26in
\end{figure}

Through a series of experiments on typical Vision and Speech datasets and network architectures (\cref{sec: results}), we analyze the property of our method and compare its performance to that of the recently proposed methods including MMCE, Brier Loss, and Focal Loss. 
On a variety of tasks and model architectures, we demonstrate that our method obtain the state-of-the-art results on model calibration. Moreover, we evaluate the performance of our method on model calibration under conditions of  distribution shift via various experiments on different datasets and model combinations. It is shown that our method provide competitive or even better results in comparison to the other methods listed in the experiments.   

%
%
 


\subsection{Related Works}
As it is vital to develop calibrated models for decision-making tasks, numerous novel methods have been proposed for the calibration of AI models recently. All of these methods can be roughly divided into the following classes:

\textbf{Regularization:}\hspace{5pt} This category contains, to our knowledge, Label Smoothing, AvUC loss in the SVI setting, Mixup, Soft Objective, Entropy Regularization (ER). Label smoothing and Mixup tend to regularize the DNN to prevent overconfidence \citep{muller2019does,  thulasidasan2019mixup}. \citeauthor{zhang2022and} show that Mixup becomes more effective in the high-dimensional setup. ER employs the negative entropy term to penalize centralized confidence distribution in order to estimate a well-calibrated confidence distribution \citep{pereyra2017regularizing}. The same reasoning follows in the Focal loss which is proved to be an upper bound for the entropy-penalized KL loss \citep{mukhoti2020calibrating}. In \citep{wang2021rethinking}, the authors investigate the property of the regularization based methods in depth. 

\textbf{Calibration Objectives:}\hspace{5pt} By incorporating CE-quantifying terms into the loss function, the methods in this class motivate the DNN to optimize both CE and predictive accuracy simultaniously. MMCE expresses CE as an optimizable probability measure utilizing RHKS \citep{kumar2018trainable}. In \citep{krishnan2020improving}, the authors derive the so-called AvUC loss function from a differential measure they proposed in their paper. It is stated that Mean Square Error loss is superior to Cross-Entropy in the work by \citeauthor{hui2020evaluation}. In addition, In \citep{karandikar2021soft}, the authors propose a differential loss based on a continuous binning operation. \citeauthor{zhang2022and} suggest that the loss function with pairwise constraints is advantageous for calibration. 

\textbf{Post-Processing Methods:}\hspace{5pt} 
Post-Processing techniques typically fit a separate module on a validation dataset after training to rescale the confidences. Temperature scaling, which is a special case of Platt scaling \citep{guo2017calibration, platt1999probabilistic}, is the most effective approach among them, but it fails in the DS case \citep{ovadia2019can}.
Other popular ways include histogram binning, isotonic regression, and Bayesian binning into quaniles, all of which were initially proposed for binary classification but can be easily extended to multi-class classification \citep{niculescu2005predicting, zadrozny2002transforming, naeini2015obtaining, kull2019beyond, allikivi2020non, wenger2020non, gupta2020calibration}. 

\textbf{Bayesian Model Averaging:}\hspace{5pt} The Bayesian framework is a potent instrument for evaluating predictive uncertainty due to its versatility in uncertainty analysis \citep{draper1995assessment}. In addition, it is demonstrated that coherent bayesian is inherently calibrated \citep{dawid1982well}. Despite the computational cost, Bayesian Neural Networks that take architecture uncertainty into account are promising for providing accurate uncertainty estimates. For approximated bayesian computation, a variety of techniques, such as variational approximation, are recommended for reducing computational complexity. 
In \citep{louizos2017multiplicative}, the authors examine the application of variational approximation to the calibration task and report that model averaging achieves better performance than a single model in terms of both accuracy and calibration.
Incorporating uncertainty into predictions is also possible by averaging the predictions of multiple models, also known as ensembles \citep{lakshminarayanan2017simple, wen2020batchensemble, dusenberry2020efficient}. 
Due to space constraints, we are unable to discuss all of the fascinating studies on model calibration published in recent years; thus, we will provide a list of other relevant works: \cite{gawlikowski2021survey, wang2021energy, jansen2004calibration, datta2021cal, minderer2021revisiting, ma2021improving, tomani2021towards}.




\section{Background}
\label{sec: bg}
\textbf{Setup}\hspace{5pt} This work focuses on the multi-class classification.
The training dataset is denoted by $\mathcal{D}_{\text{train}} = \{(\mathbf{x}^{(i)}, y^{(i)})\}_{i=1}^{M}$, which consists of $M$ i.i.d. samples following the ground truth distribution $\mathcal{P}(\mathbf{X}, Y)$. 
Here $\mathbf{x}^{(i)}\in\mathbb{R}^{N}$ and $y^{(i)}\in\{1, \cdots n\}$ represent the feature vector and the class label of the $i$'th sample, respectively. 
On a training set, a $n$-classes classification DNN learns a transformation $f_{\theta}(\cdot)$ mapping a feature vector $\mathbf{x}^{(i)}$ to a class probability vector $p^{(i)}$ by optimizing the loss function, which is cross-entropy in our setup. 
In the prediction stage, a predicted tuple $(\hat{y}, \hat{p})$ for a given sample $\mathbf{x}$ is provided, where $\hat{y} = \arg\!\max_j{\hat{p}_{j}(\mathbf{x})}$, and $\hat{p} = f_{\theta, \hat{y}}(\mathbf{x})$. $\hat{y}$ is the predicted label, and $\hat{p}$ is its confidence estimate. Intuitively, the confidence provided by the well-calibrated model matches up precisely to the probability that the prediction is correct, therefore we define the exact calibration criterion as 
\begin{align}
\label{eq: calib}
    \mathcal{P}(\hat{Y} = y | f_{\theta, \hat{y}}(\mathbf{X})=p) = p, \quad \forall p\in [0, 1].
\end{align}

Although we can use Def. \ref{eq: calib} to determine whether a model is perfectly calibrated in theory, we still need a metric to evaluate how far a model is from being perfectly calibrated in practice.

\textbf{Metrics}\hspace{5pt} Various novel metrics have been proposed to measure the distance between the left and right side of Eq. \ref{eq: calib} \citep{roelofs2022mitigating, kumar2019verified, gupta2020calibration}. One frequently used metric is the Expected Calibration Error (ECE), which is defined as the expectation of the $L^{1}$ norm of the difference between the confidence and accuracy.
\begin{align}
\label{eq: ece}
    \text{ECE}[f_{\theta}(\cdot)] = \mathbb{E}_{f_{\theta, \hat{y}}(\mathbf{X})}[|p - \mathcal{P}(\hat{Y}=y|f_{\theta, \hat{y}}(\mathbf{X})=p)|].
\end{align}
Due to the discrete nature of the practical dataset, binning approximation of \cref{eq: ece} is commonly used,
\begin{align}
\label{eq: ece_bin}
    \text{ECE}_{\mathcal{D}_{\text{test}}} = \sum_{l=1}^{L}\frac{|B_{l}|}{|\mathcal{D}_{\text{test}}|}|\text{acc}(B_{l}) - \text{conf}(B_{l})|, 
\end{align}
where $\text{acc}(\text{B}_{l})$ and $\text{conf}(\text{B}_{l})$ represent, respectively, the average accuracy and average confidence of the samples located in $\text{B}_{l}$. $|\text{B}_{l}|$ denotes the number of samples in the $\text{B}_{l}$ bin. 

\textbf{Empirical Calibration Histogrm}\hspace{5pt} The empirical calibration histogram, which depicts the gap between the model's confidence and accuracy, is a valuable depiction of model calibration for understanding the distribution of miscalibration \citep{degroot1981assessing, niculescu2005predicting}.
It plots the accuracy against confidence, with perfect calibration resulting in a diagonal line.  

\section{Annealing Double-Head}
\label{sec: doublehead}
In this section, we provide a detailed explanation of Annealing Double-Head and discuss its properties. 
\subsection{Architecture}
\label{sub: architecture}
In the Double-Head architecture, the logits $\mathbf{z}(\mathbf{x})$ supplied by the latent layers preceding the output layer are transformed into predictive probabilities by the calibration head.
In our design, the calibration head is implemented by a fully connected neural network with three layers: an input layer whose dimension is the same as the logits, a latent layer whose dimension is one half of the dimension of the input layer, and an output layer whose dimension is the number of classes. In our experiments, this structure performe remarkably well (\cref{sec: results}). 
We use the training dataset to optimize the parameters of the main head, while the calibration head is trained on a separate data set, known as the calibration set, to prevent overfitting. 
Since the calibration set is never viewed by the main head during training, it can be reused as the validation set to fine-tune the main head's structure and hyper-parameters after training. In our experiments, we choose the validation set to be the calibration set.

The calibration head and main head are trained simultaneously on their respective calibration set and training set.
In particular, we interleave the optimization of parameters in the main head and the calibration head by updating the calibration head's parameters with one step for every $k$ steps of updates to the main head's parameters.

We observe that the learning rate, which can be fine-tuned on a separated validation set, is a crucial hyperparameter for the calibration head's performance. 
The empirical rule for the ratio between the learning rates of the main head and the calibration head is approximately $1$ to $100$, according to our numerical analysis.

\subsection{Annealing and Augmentation} 
During the optimization of the calibration head, we also employ an evolving constant $\beta$ to dynamically rescale the confidence during the learning of the model in order to prevent overconfidence (or underconfidence). We do not want the learned map to be overscaled, as this would result in biased estimation, so that the absolute difference between the scaling factor $\beta$ and $1$ decreases and converges to $0$ as the number of training steps increases in each epoch.

In our scheme, the scaling factor $\beta_{t}$ at step $t$ is defined by the following \cref{eq: dt}, 
\begin{align}
\label{eq: dt}
    \beta_{t} = \beta_{0} - (\beta_{0} - 1) \times \frac{t}{s},
\end{align}
where $\beta_{0}$ is the initial scaling factor and $s$ is the number of training steps in one epoch. 
We summarize our Anneal Double-Head method in \cref{alg: adh}.

\begin{algorithm}[tb]
\caption{Annealing Double-Head}
\label{alg: adh}
\begin{algorithmic}[1]
   \STATE{\bfseries Input:}\\
   Training and Validation Dataset: $\mathcal{D}_{\text{training}}$, $\mathcal{D}_{\text{valid}}$;\\ 
   Learning Rates: $\lambda_{\text{main}}$, $\lambda_{\text{calib}}$;\\
   Initial Scaling Factor $\beta_{0}$;\\
   Clibration Period: $k$;\\
   \STATE{\bfseries Initialization:} \\
   Parameters $\Theta_{\text{main}}$ and $\Theta_{\text{calib}}$ in main and calibration head
   \FOR{$epoch=1$ {\bfseries to} $\text{EPOCHS}$}
   \FOR{$t = 1$ {\bfseries to} $\text{STEPS}$}
   \STATE Sample a batch $\mathbf{X}_{\text{train}}$ from $\mathcal{D}_{\text{train}}$
   \STATE update $\Theta_{\text{main}}^{t+1} \xleftarrow{} \Theta_{\text{main}}^{t} - \lambda_{\text{main}}\nabla\mathcal{L}_{\text{main}}$
   \IF{$t \text{ mod } k = 0$}
   \STATE update $\beta_{t}$ by \cref{eq: dt}
   \STATE Sample a batch $\mathbf{X}_{\text{calib}}$ from $\mathcal{D}_{\text{calib}}$
   \IF{Turn on Augmentation}
   \STATE $\mathbf{X}_{\text{calib}} \xleftarrow{} g(\mathbf{X}_{\text{calib}})$ \\
   where $g(\cdot)$ is the augmentation map
   \ENDIF
   \STATE $\mathbf{Z}_{\text{calib}} \xleftarrow{} \beta_{t} \mathbf{Z}_{\text{calib}}$
   \STATE update $\Theta_{\text{calib}}^{t+1} \xleftarrow{} \Theta_{\text{calib}}^{t} - \lambda_{\text{calib}}\nabla\mathcal{L}_{\text{calib}}$
   \ENDIF
   \ENDFOR
   \ENDFOR
\end{algorithmic}
\end{algorithm}

For the Distribution Shift case, random perturbations are added to the inputs of the calibration head to reduce the confidence of the OOD samples. As a result of separating the calibration component committed for calibration from the model's main architecture, our architecture has the advantage that the training calibration head with the augmented samples does not influence the learning progress of the main head.



\subsection{Why Annealing Double-Head works?}
\label{subsec: why_work}
As is known, shallow neural networks are properly calibrated but have limited predictive power due to their restricted capacity, whereas modern deep neural networks offer superior predictive accuracy but poor calibration. How therefore can the two benefits of shallow and deep neural networks, namely excellent model calibration and strong predictive power, be combined into a single model? 
To overcome this challenge, we develop the Annealing Double-Head architecture, in which the confidence provided by the deep layer component, also known as the main head, is calibrated by a shallow layer component, also known as the calibration head. 
Due to the powerful deep neural layers in the main head, state-of-the-art classification accuracy can be reached, whilst the calibration head, made of shallow layers, enable estimation of precise predictive uncertainty based on the logits from the main head. 

For the DNNs using Negative Log Likelihood (NLL) as the objective, such as cross-entropy, it tries to boost the logits during the later period of training, when the training accuracy is high, by amplifying the weights in order to reduce the training NLL even further \citep{mukhoti2020calibrating}. For this reason, regularization-related methods that tend to constrain the weight norm to be small could be advantageous for the calibration task. 

Our method for addressing overconfidence (or underconfidence) entails dynamically scaling up (or down) the logits produced by calibration head by a varying factor $\beta$ to align the predictive confidence distribution with the ground truth confidence distribution.
The scale factor $\beta$ is set to asymptotically converge to $1$ at the completion of training in order to prevent a global scaling of confidence magnitude. In addition, dynamic rescaling of logits results in adaptive gradients during training, as evidenced by the following statement:
\begin{theorem}    
\label{the: annealing_gradient}
Let $\mathcal{L}(\mathbf{z}, y)$, denote the cross-entropy loss of a sample pair, 
$(\mathbf{x}, y)$ denote a sample and its label, and $\mathbf{z} =f_{\theta}(\mathbf{x})$ represent the logit of the sample $\mathbf{x}$.  
Given the assumption that $y = \arg\max_{j}(\mathbf{z}_{j})$, we have
\begin{align*}
&\partial_{\theta}{\mathcal{L}(\beta \mathbf{z}, y)} = \sum_{j}\gamma_{j}\beta\frac{\partial{\mathcal{L}(\mathbf{z}, y)}}{\partial{z_{j}}}\partial_{\theta}{z_{j}}, \\
&\begin{cases}
c_{0, i} \leq \gamma_{i} \leq c_{1, i} \quad \text{ if } i = y \\
c_{2, i} \leq \gamma_{i} \leq c_{3, i} \quad \text{ if } i \neq y
\end{cases}
\end{align*}
where
\begin{align*}
c_{0, i} = &\frac{1}{n}\exp{[(1-\beta)(z_{i} - \mathbf{z}_{(1)}) - (\mathbf{z}_{(n-1)} - \mathbf{z}_{(1)})]} \\
c_{1, i} = &n\exp{[(1-\beta)(z_{i} - \mathbf{z}_{(n-1)}) + (\mathbf{z}_{(n-1)} - \mathbf{z}_{(1)})]} \\
c_{2, i} = &(\frac{1}{n} + \frac{n-1}{n}\exp{(\mathbf{z}_{(1)} - \mathbf{z}_{(n)})}) \\
&\times\exp{[(1 - \beta)(\mathbf{z}_{(n)} - z_{i})]} \\
c_{3, i} = &(1 + (n - 1)\exp{(\mathbf{z}_{(n-1)} - \mathbf{z}_{(n)})})\\
&\times\exp{[(1 - \beta)(\mathbf{z}_{(n)} - z_{i})]}.
\end{align*}
\end{theorem}
We note that $\mathbf{z}_{(k)}$ means the $k^{th}$ order statistics of logit vector $\mathbf{z}$. Proof of above proposition can be found in \cref{subsec: grad_proof}. 

In the preceding \cref{the: annealing_gradient}, we assume that $y=\arg\max_{j}{\mathbf{z}_{j}}$, and we argue that this assumption is justified since, at least in later periods of training, when the calibration head converges to the high-accuracy state, the assumption is a good approximation of the actual scenario.
Since $\mathbf{z}_{(n)}$ dominates $\mathbf{z}$, \cref{the: annealing_gradient} demonstrates that the upper bound and lower bound of the rescaled gradient $||\partial_{\theta}{\mathcal{L}(\beta \mathbf{z}, y)}||$ are primarily governed by the factor $\exp{[(1-\beta)\mathbf{z}_{(n)}]}$. 
Therefore, if $\beta > 1$, the norm of the gradients $||\partial_{\theta}{\mathcal{L}(\beta \mathbf{z}, y)}||$ is rescaled downwards, otherwise it is scaled upwards. This behavior is numerically verified in Fig. \ref{fig: grads_hist_end}, and Appendix \ref{subsec: varying_beta} provides more discussion.

As depicted in Fig. \ref{fig: entropy_vs_T}, the adapted gradient during the training of the calibration head results in the varied entropy of the converged confidence distribution: an increase in the annealing factor $\beta$ will lead to a greater entropy.
\cref{subsec: annealing} provides a more exhaustive numerical analysis.

The idea of calibrating neural networks by penalizing the entropy of the networks' confidence distributions has been explored in several works, such as \citep{pereyra2017regularizing} and \citep{mukhoti2020calibrating}. 
And these findings are compatible with the following theorem, in which we demonstrate that $\text{ECE}_{2}$, the second-order ECE, is bounded by a constant term $\text{C}[\hat{p}(x]$, which is a function of the confidence distribution $\hat{p}(x)$, minus the entropy $\text{H}[\hat{p}(x)]$ of the confidence distribution $\hat{p}(x)$.

\begin{theorem}
\label{the: ece_upper_bound}
Let $\hat{p}(\cdot)$ denote a function which maps a sample $\mathbf{x}$ to its confidence $\hat{p}(\mathbf{x})$, then the second-order ECE: $\text{ECE}_{2}$, which is defined as the square root of the second moment of absolute difference between the confidence and accuracy can be bounded as follows,  
\begin{align*}
    \text{ECE}_{2}[\hat{p}(\mathbf{x})] \leq 
    & (\text{C}[\hat{p}(x)] - 2 H[\hat{p}(x)])^{1/2}, 
\end{align*}
where 
\begin{align*}
&\text{C}[\hat{p}(x)] = 3 - 2\mathbb{E}_{\hat{p}(x)}[\log{\hat{p}(x)}] - \mathbb{E}_{\hat{p}(x)}[f(\hat{p}(x))], \\
&\text{H}[\hat{p}(\mathbf{x})] = \mathbb{E}_{\hat{p}(\mathbf{x})}[\log{f(\hat{p}(\mathbf{x})}].
\end{align*}
Actually $H[\hat{p}(\mathbf{x})]$ is just the shannon entropy of $\hat{p}(\mathbf{x})$, and $f(\hat{p}(\mathbf{x}))$ stands for the probability mass function (PMF) of random variable $\hat{p}(\mathbf{x})$.
\end{theorem}
Proof can be found in \cref{subsec: bound_proof}.
This theorem implies that the confidence distribution with a greater entropy tends to have a smaller ECE. However, it should be noted that higher entropy does not necessarily correspond to improved model calibration, as intuitively excessive entropy might lead to underconfidence issue. In \cref{subsec: main_vs_calib}, we examine the entropy of the outputs distribution of both the main head and the calibration head. The entropy of the confidence by calibration head, which is better calibrated, is indeed greater than the confidence by main head.  



\section{Experiments and Results}
\label{sec: results}

\begin{table*}[!t]
\vskip 0.15in
\begin{centering}
\begin{small}
\begin{sc}
\resizebox{\textwidth}{!}{
\begin{tabular}{c c cc cc cc cc cc}
\toprule
\multirow{2}{*}{\textbf{Dataset}} & \multirow{2}{*}{\textbf{Model}} & \multicolumn{2}{c}{\textbf{Cross Entropy}} & \multicolumn{2}{c}{\textbf{MMCE}} & \multicolumn{2}{c}{\textbf{Brier Loss}} & \multicolumn{2}{c}{\textbf{Focal Loss}} & \multicolumn{2}{c}{\textbf{Double-Head}} \\
 &  & pre T & post T & pre T & post T & pre T & post T & pre T & post T & No Anneal & Anneal \\
\midrule
\multirow{4}{*}{CIFAR-10} & ResNet 50 & 4.28 & 2.17 (1.8) & 3.18 & 2.08 (1.4) & 2.56 & 1.58 (1.1) & 2.73 & 2.25 (1.1) & 1.22 & \bf{0.50} \\
& ResNet 101 & 5.69 & 2.98 (2.0) & 3.71 & 2.03 (1.5) & 2.75 & 2.04 (1.1) & 2.88 & 2.37 (1.1) & 1.32 & \bf{0.65} \\
& DenseNet 121 & 3.43 & 3.12 (1.1) & 3.13 & 2.40 (1.2) & 1.22 & 1.22 (1.0) & 2.65 & 2.02 (1.3) & 1.40 & \bf{1.18} \\
& Wide ResNet 28-10 & 3.22 & 2.38 (1.2) & 2.42 & 1.65 (1.2) & 1.09 & 1.09 (1.0) & 0.84 & 0.84(1.0) & 1.28 & \bf{0.71}\\
 \midrule
\multirow{4}{*}{CIFAR-100} & ResNet 50 & 10.48 & 6.32 (1.2) & 11.29 & 5.67 (1.3) & 5.54 & 5.54 (1.0) & 7.66 & 4.99 (1.1) & 2.52 & \bf{1.04} \\
& ResNet 101 & 10.40 &  6.82 (1.2) & 10.11 & 4.97 (1.3) & 3.43 & 1.57 (3.0) & 5.79 & 2.54 (1.5) & 2.04 & \bf{0.88} \\
& DenseNet 121 & 8.12 & 4.39 (1.2) & 8.97 & 4.82 (1.2) & 4.14 & 4.14 (1.0) & 6.14 & 3.56 (1.1) & 3.38 & \bf{1.51}\\
& Wide ResNet 28-10 & 5.35 & 4.87 (1.1) & 6.32 & 4.65 (1.2) & 3.25 & 3.25 (1.0) & 3.61 & 3.19 (1.1) & 2.40 & \bf{1.53} \\ \midrule
\multirow{4}{*}{SVHN} & ResNet 50 & 2.90 & 1.29 (1.7) & 2.48 & 0.85 (1.6) & 0.90 & 0.90 (1.0) & 1.93 & 1.55 (1.2) & 0.78  & \bf{0.54}\\
& ResNet 101 & 3.10 & 1.09 (1.9) & 2.83 & 1.09 (1.6) & 0.83 & 0.83 (1.0) & 2.22 & 1.07 (1.2) & \bf{0.65} &  0.82\\
& DenseNet 121 & 2.31 & 0.89 (1.5) & 2.33 & 1.07 (1.5) & 1.74 & 1.74 (1.0) & 1.34 & 1.34 (1.0) & 1.17 & \bf{0.76}\\
& Wide ResNet 28-10 & 1.95 & 0.89 (1.5) & 1.90 & 1.76 (1.2) & 1.04 & 1.07 (0.9) & 0.79 & 0.79 (1.0) & 1.28 & \bf{0.65}\\ \midrule
\multirow{2}{*}{ImageNet}
& ResNet 50 & 24.96 & 6.78 (0.6) & 4.56 & 1.41 (1.2) & 1.23 & 1.23 (1.0) & 3.81 & 2.58 (0.8) & 1.26 & \bf{1.13}\\
& ResNet 152 & 12.83 & 3.64 (0.8) & 4.10 & 2.15 (1.1) & 2.72 & \bf{1.54} (1.1) & 6.84 & 2.34 (0.9) & 2.63 & 1.65\\
\midrule
SST Fine Grained & TreeLSTM & 1.91 & 1.43 (1.2) & 1.72 & 1.72 (1.0) & 4.92 & 3.99 (0.9) & 7.37 & 1.73 (0.6) & 1.79 & \bf{1.26}\\
SST Binary & GP CNN & 10.30 & 2.85 (2.5) & 4.49 & 2.29 (1.3) & 2.52 & 2.42 (1.1) & 3.13 & 2.81 (1.1) & \bf{1.43} & 1.84\\
20 Newsgroups & GP CNN & 4.10 & 2.55 (1.1) & 3.59 & 3.06 (1.3) & 2.75 & \bf{2.01} (0.9) & 3.47 & 2.65 (1.2) & 2.46 & 2.24\\ 
\bottomrule
\end{tabular}
}

\end{sc}
\end{small}
\end{centering}
\caption{ECE ($\%$) of multiple methods: CE, MMCE, Brier, FL, and ours with and without TS on different model and dataset combinations is reported. The used temperature is specified between brackets. In the last two columns, the ECE of our own approach without Annealing and with Annealing is displayed.}
\label{tab: ece_max}
\vskip -0.2in
\end{table*}

\begin{figure*}[!t]
\vskip 0.2in
\begin{subfigure}{0.19\textwidth}
 \centering  
 \includegraphics[width=0.95\linewidth]{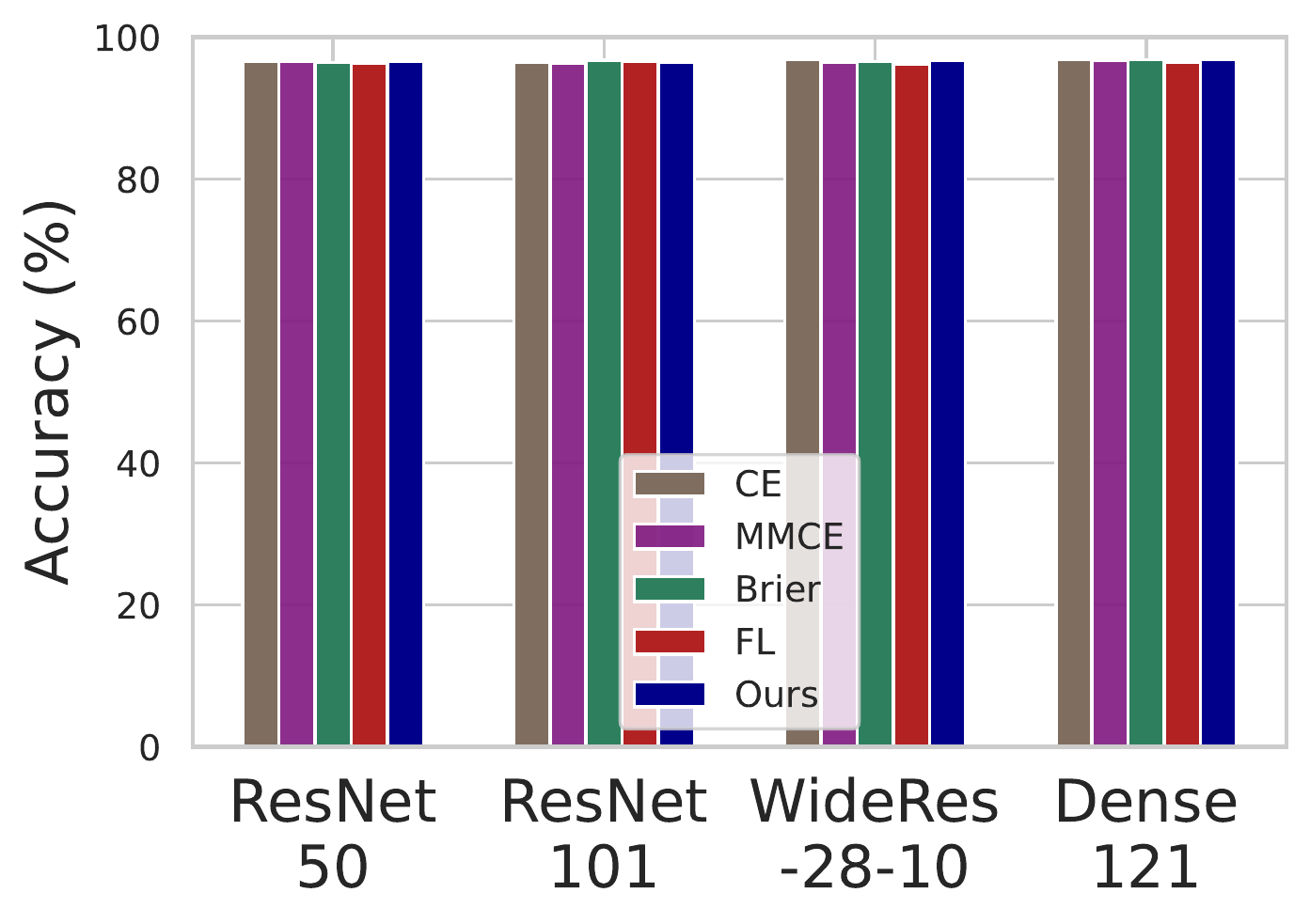}  
  \caption{CIFAR-10}
  \label{fig: acc_cifar10} 
\end{subfigure}
\begin{subfigure}{0.19\textwidth}
  \centering
  \includegraphics[width=0.95\linewidth]{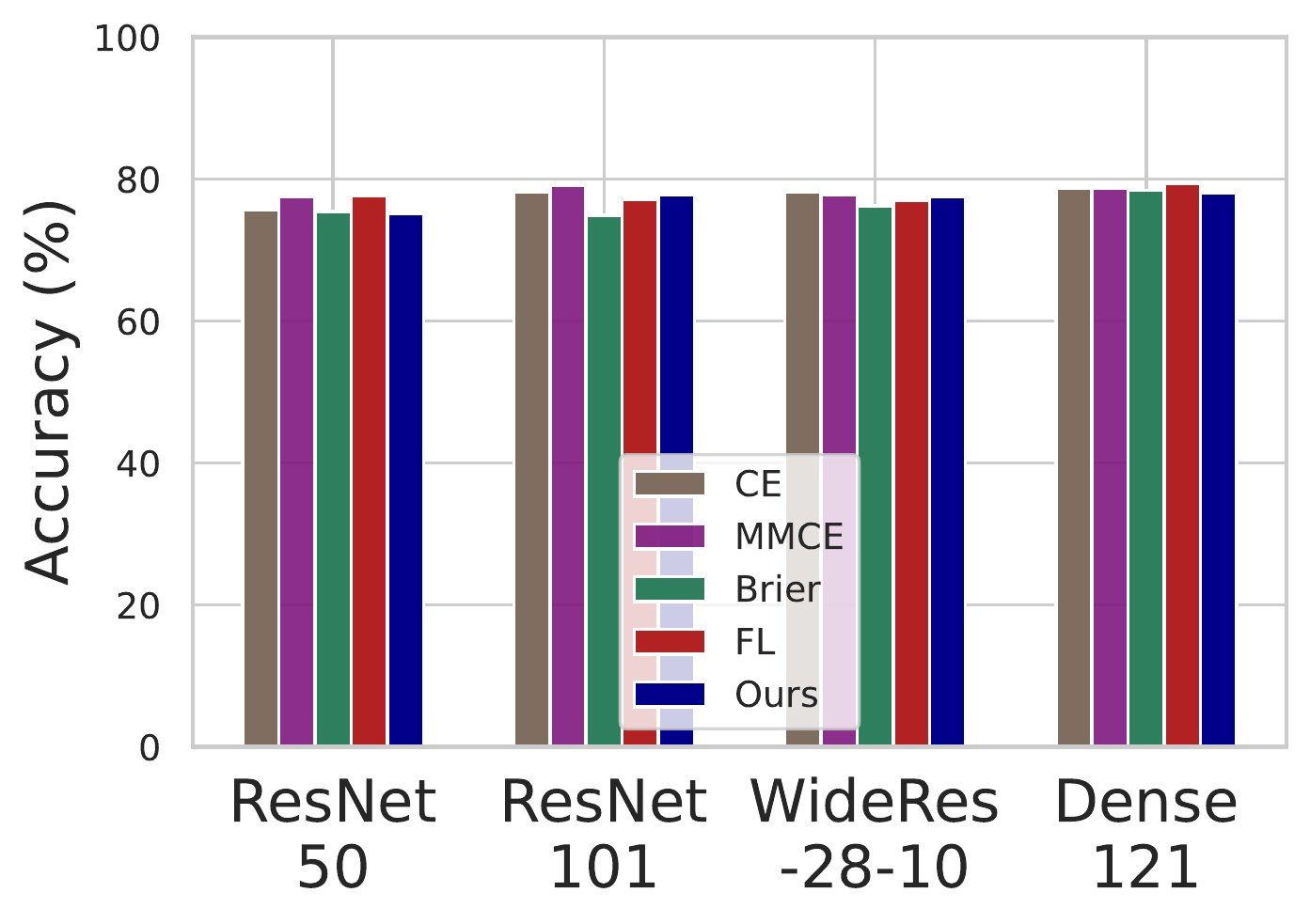} 
  \caption{CIFAR-100}
  \label{fig: acc_cifar100}
\end{subfigure}
\begin{subfigure}{0.19\textwidth}
  \centering
  \includegraphics[width=0.95\linewidth]{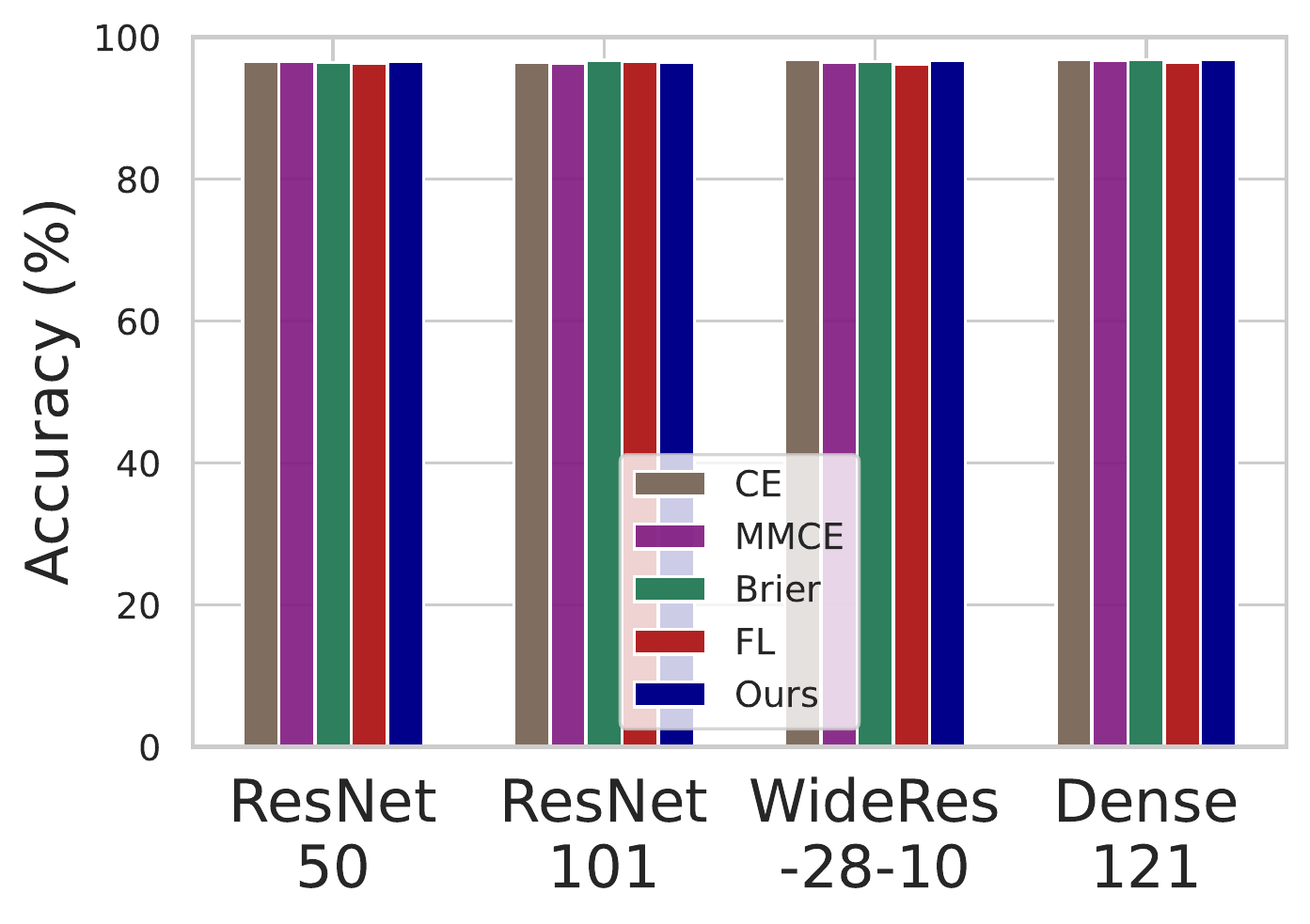} 
  \caption{SVHN}
  \label{fig: acc_svhn}
\end{subfigure}
\begin{subfigure}{0.19\textwidth}
  \centering
  \includegraphics[width=0.95\linewidth]{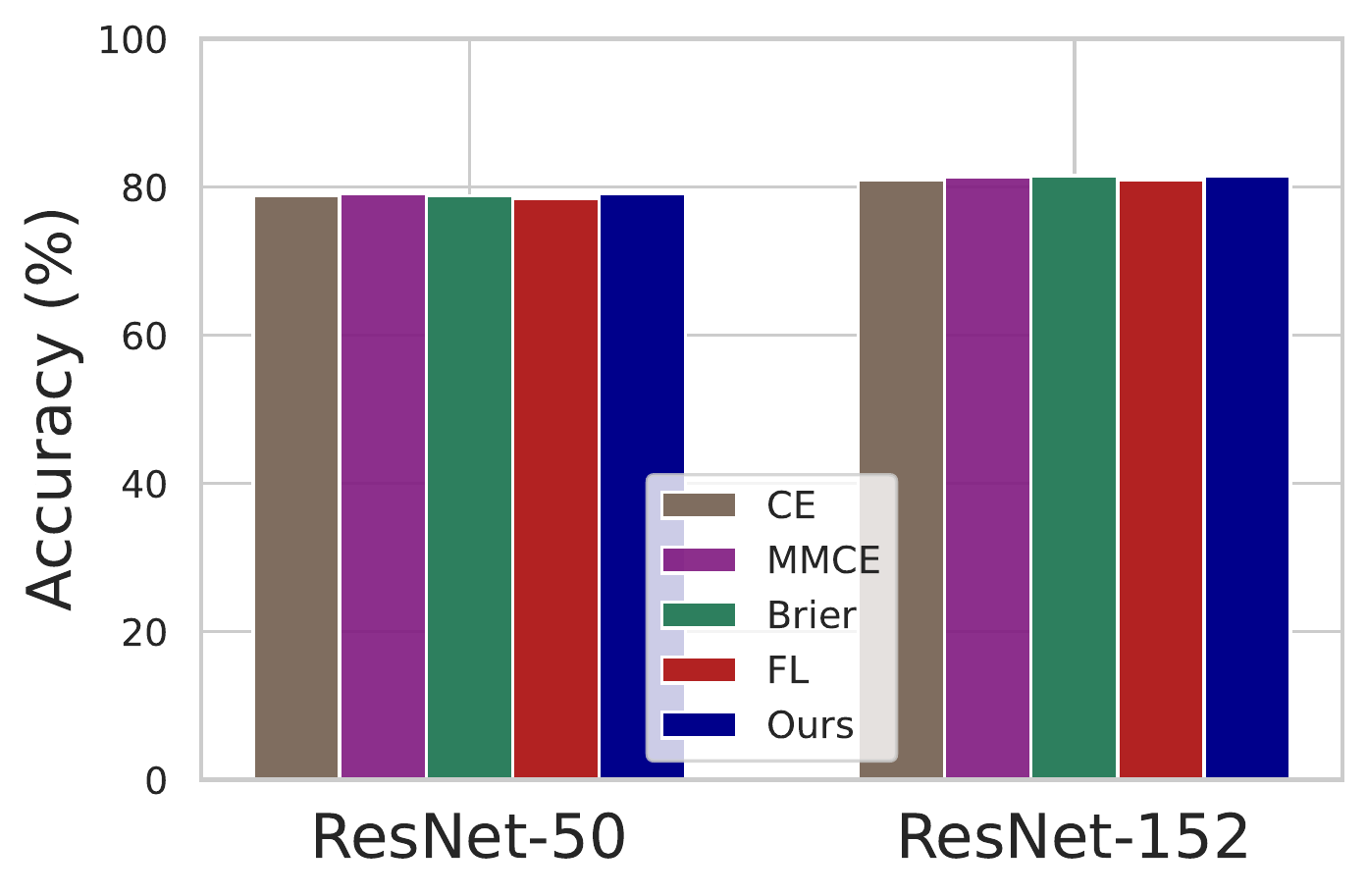} 
  \caption{ImageNet 2012}
  \label{fig: acc_4}
\end{subfigure}
\begin{subfigure}{0.19\textwidth}
  \centering
  \includegraphics[width=0.95\linewidth]{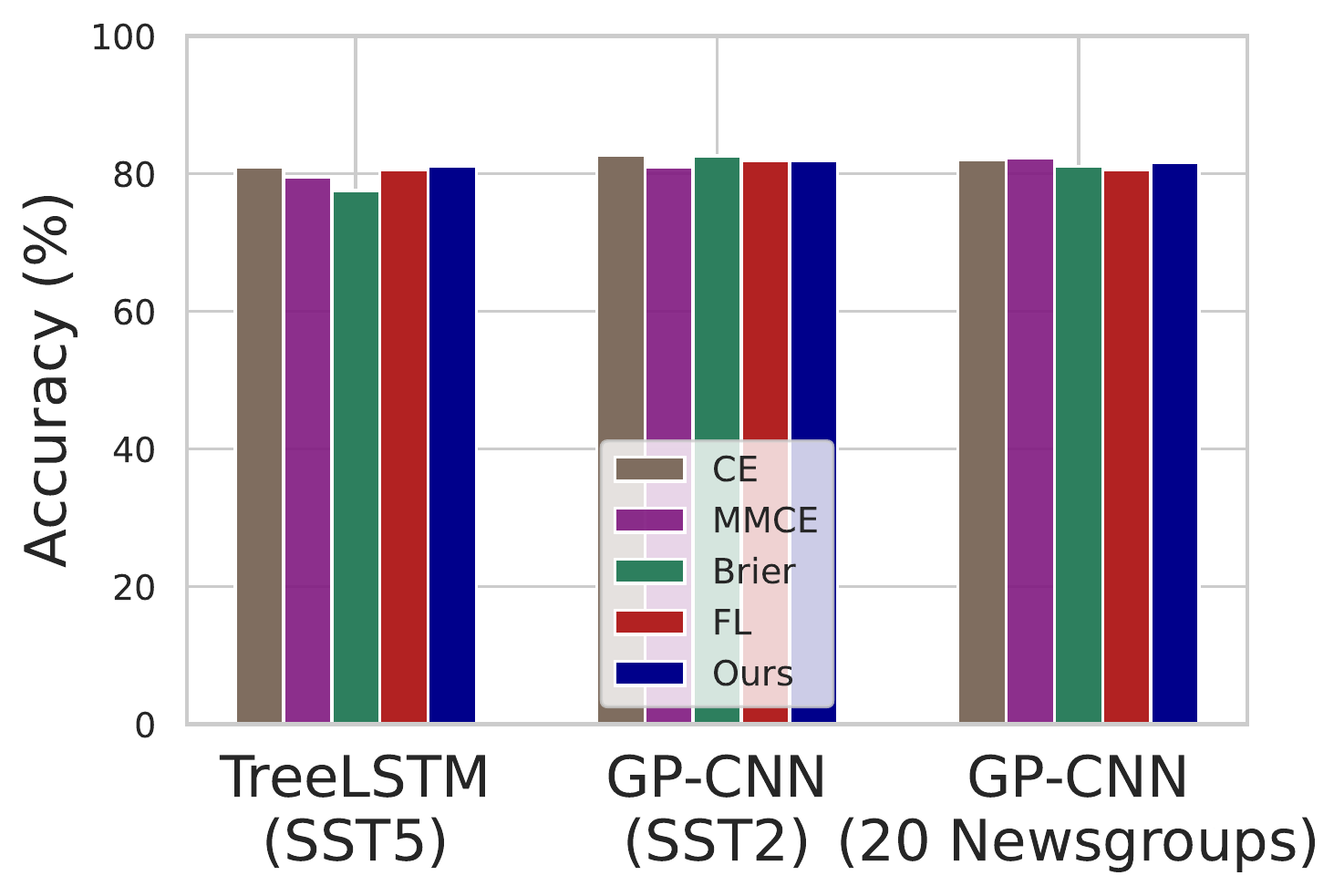} 
  \caption{NLP Datasets}
  \label{fig: acc_nlp}
\end{subfigure}
\caption{Barplots of accuracy on vision and speech datasets. (a), (b) and (c) show the accuracies provided by aformentioned models on vision datasets including: CIFAR-10, CIFAR-100 and SVHN. (d) display the accuracies of ResNet-50 and ResNet-152 on ImageNet, (e) summarizes the accuracies of TextCNN and TreeLSTM on speech datasets: SST Binary, SST Fine Grained and 20 Newsgroups.}
\label{fig: acc_barplot}
\vskip -0.2in
\end{figure*}

To validate the performance of Anneal Double-Head architecture on the calibration of DNNs, we compare our method to some other existing calibration methods including MMCE, Focal Loss, Soft Objective on both Vision and NLP datasets. In our experiments, we scale neural networks architectures from modest to large in order to fully comprehend the scalability of our method. Experiments demonstrate that our method delivers state-of-the-art ECE on multiple datasets and modern architectures with allowable computational overhead, which is controlled by the calibration period hyper-parameter $k$ in our algorithm. 
Detailed analysis of the computational complexity is presented in \cref{subsec: complexity}.

\textbf{Datasets}\hspace{5pt}
Popular vision datasets such as Street View House Numbers (SVHN) \citep{netzer2011reading}, CIFAR-$10$/CIFAR-$100$ \citep{krizhevsky2009learning}, and ImageNet $2012$ \citep{deng2009imagenet} are employed in our studies, as well as frequently used speech datasets such as Stanford Sentiment Treebank (SST) \citep{socher2013recursive}and 20 Newsgroups \citep{Lang95}.
In \cref{sec: exp_app}, you may find further information about these datasets.


\textbf{Models} \hspace{5pt} 
In vision tasks, we employ modern Convolutional Neural Networks (CNN): Residual Networks \citep{he2015deep} with $50$, $101$, and $152$ layers (Res-$50$, Res-$101$, and Res-$152$ respectively), Wide Residual Networks \citep{zagoruyko2016wide} with $28$ layers and a widening factor of $10$ (Res-28-10), and Dense Networks \citep{huang1608densely} with $121$ layers (dense-$121$). Text CNN \citep{kim2014convolutional} and Tree Long Short Term Memory (LSTM) \citep{hochreiter1997long, tai2015improved} are utilized for NLP tasks, with CNN-static architecture and Glove word embedding technique \citep{pennington2014glove} applied to Text CNN. We describe the setup of training in detail in \cref{sec: exp_app}.

\textbf{Comparison Methods}\hspace{5pt}
We assess the performance of our method in comparison to the following recently proposed "end-to-end" calibration techniques:
\begin{itemize}
    \item Maximum Mean Calibration Error (MMCE) \citep{kumar2018trainable} is a trainable measure of the calibration error based on the RKHS kernel;

    \item Brier Loss (BL) \citep{brier1950verification, degroot1983comparison} is the mean square error between the predicted confidence vector and the one-hot encoding vector of the ground truth label, and it is shown that Brier Loss can be decomposed into calibration and refinement which is related to the area under the ROC curve;  
    
    \item Focal Loss (FL) \citep{mukhoti2020calibrating} is shown to be an upper bound on the regularized KL-divergence, and thus encouraging the model to increase entropy so to prevent overfitting during training procedure;
\end{itemize}

\subsection{Calibration Results}
In Table. \ref{tab: ece_max}, we systematically evaluate the ECE ($\%$) of the aforementioned methods (with and without TS) and ours on varies architectures and datasets. With the exception of our own approach, we report the post-processed ECE by TS and the corresponding optimal temperature to it. We demonstrate that \textit{our method, without post-processing, may generally yield the best ECE among those given}. For our method, we additionally present the calibration error without Annealing and find that it still provide fairly modest ECE in comparison to other methods stated. Despite our method's outstanding calibration performance, it may also produce classification accuracies that are comparable to or even better than those of other methods when applied to multiple tasks, as illustrated in Fig. \ref{fig: acc_barplot}.


Fig. \ref{fig: ece_acc_vs_epoch} displays the ECE and accuracy curves versus the training Epochs for ResNet-50 on Cifar-10 in order to visualize how the ECEs and accuracies evolve during the training procedure. Our method decreases ECE rapidly and converges to the minimum value steadily, whereas the other three methods oscillate heavily and even increase considerably during the initial phase of training. As depicted in Fig. \ref{fig: ece_vs_epoch}, our approach reaches the smallest ECE among all after around $150$ epochs, whereas the ECE of MMCE increases during the last stage of training.
As seen in Fig. \ref{fig: acc_vs_epoch}, our method achieves the highest accuracy.

\begin{figure}[t!]
\vskip 0.2in
\centering
\begin{subfigure}{0.49\columnwidth}
  \centering
  \includegraphics[width=0.95\linewidth]{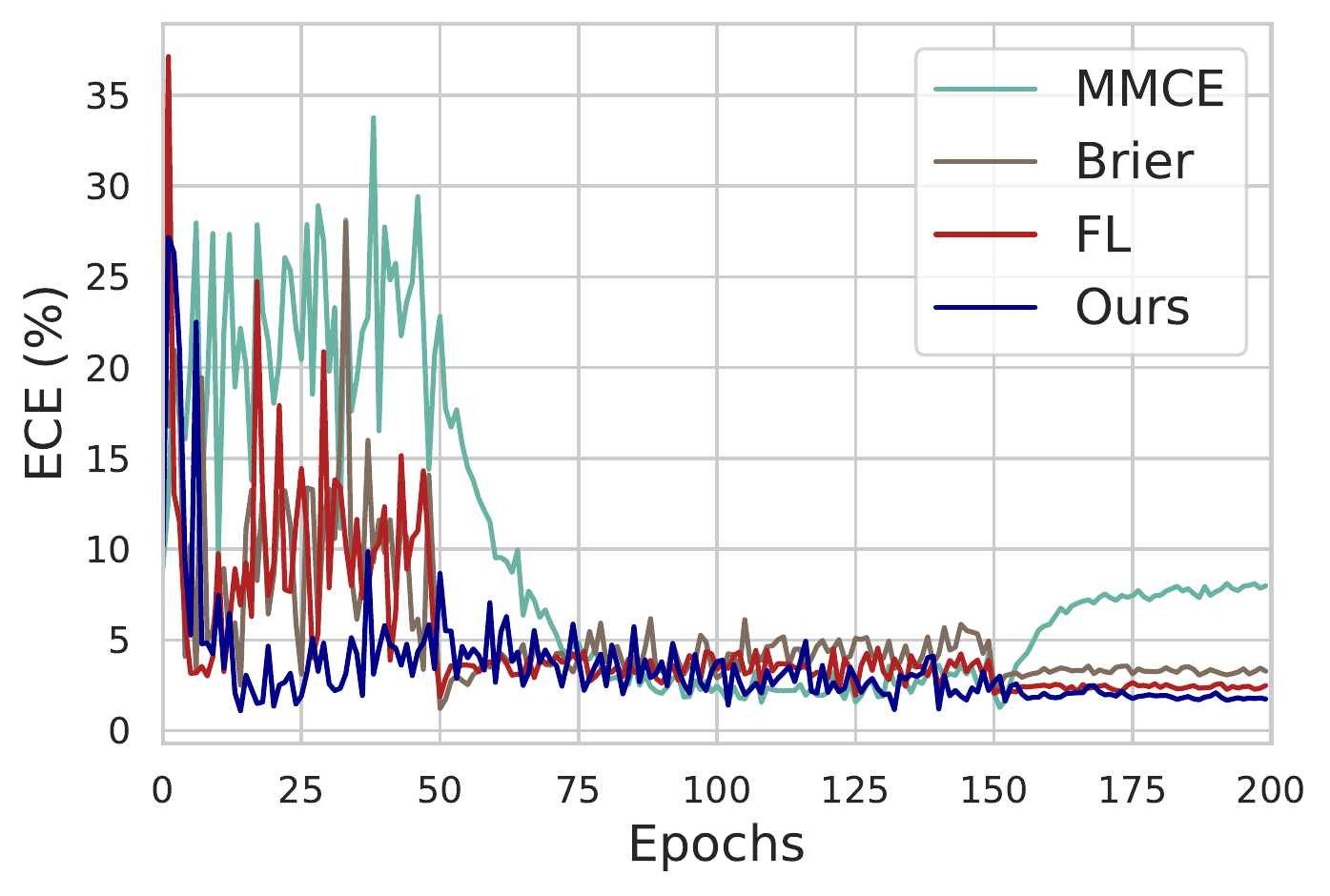}  
  \caption{ECE vs Epoch}
  \label{fig: ece_vs_epoch}
\end{subfigure}
\begin{subfigure}{0.49\columnwidth}
  \centering
  \includegraphics[width=0.95\linewidth]{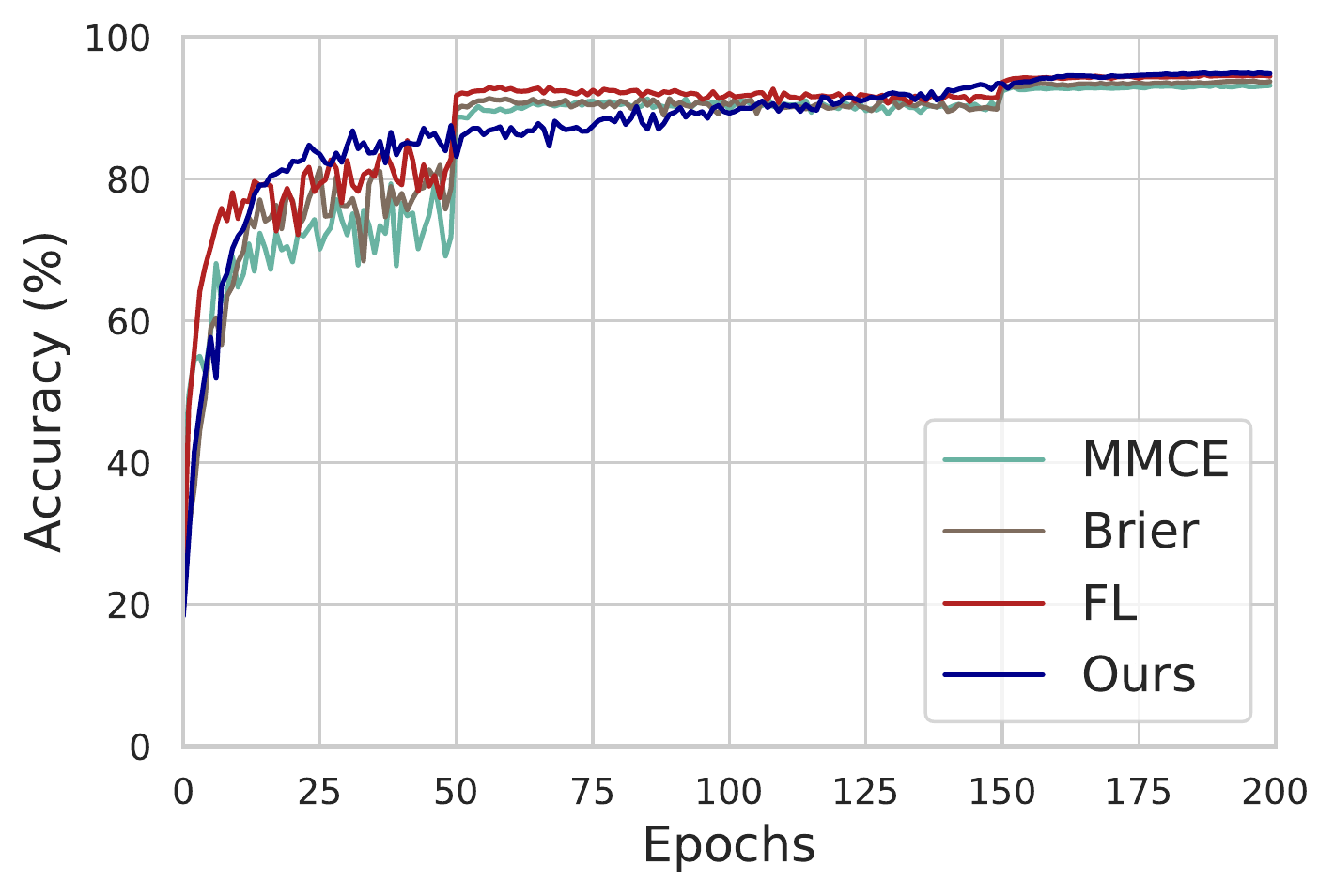} 
  \caption{Accuracy vs Epoch}
  \label{fig: acc_vs_epoch}
\end{subfigure}
\caption{ECE and accuracy for ResNet-50 on CIFAR-10 during training procedure. (a) and (b) show the ECE and accuracy versus the varying epoch respectively.}
\label{fig: ece_acc_vs_epoch}
\vskip -0.2in
\end{figure}

To further understanding the behaviors of these methods on the model calibration, we also explore the distribution of the confidences and also the reliability diagram of the predictive confidences. The results are reported in Fig. \ref{fig: relib_methods_compare} of \cref{subsec: methods_comparison}. 

\subsection{Main Head vs Calibration Head}
\label{subsec: main_vs_calib}
In this part, we comprehensively explore the dynamical behaviors of the main head and calibration head during training. 

As seen in Fig. \ref{fig: ece_main_calib}, the ECE of the calibration head decreases nearly monotonically, but the ECE of the main head decreases rapidly during the course of the first few epochs, before increasing for approximately $50$ epochs and then decreasing.

The behavior of the entropy as a function of the training epoch is also investigated. As the training epoch rises, the entropy of the predictive confidence distribution drops. However, as demonstrated in Fig. \ref{fig: entropy_main_calib}, the entropy of the confidence distribution by calibration head is always larger than the entropy of the main head. In addition, we observe that the calibration head generalizes better than the main head, as shown in Fig. \ref{fig: nll_main_calib}, where the Test NLL of the calibration head is less than the Test NLL of the main head. As presented in Fig. \ref{fig: acc_main_calib}, the classification accuracy of the calibration head is roughly equal to that of the main head.
In Fig. \ref{fig: relib_hist_calib} and Fig. \ref{fig: relib_hist_main}, we display the reliability diagram of the confidences by calibration head and main head respectively. It is shown that the calibration head is much more better calibrated than the main head which is typically overconfident. 
In Appendix \ref{subsec: conf_relib_calib_vs_main}, the confidence distributions of both components are displayed in Fig. \ref{fig: relib_calib_vs_main}. 


\begin{figure}[t!]
\vskip 0.2in
\centering
\begin{subfigure}{0.49\columnwidth}
  \centering
  \includegraphics[width=0.95\linewidth]{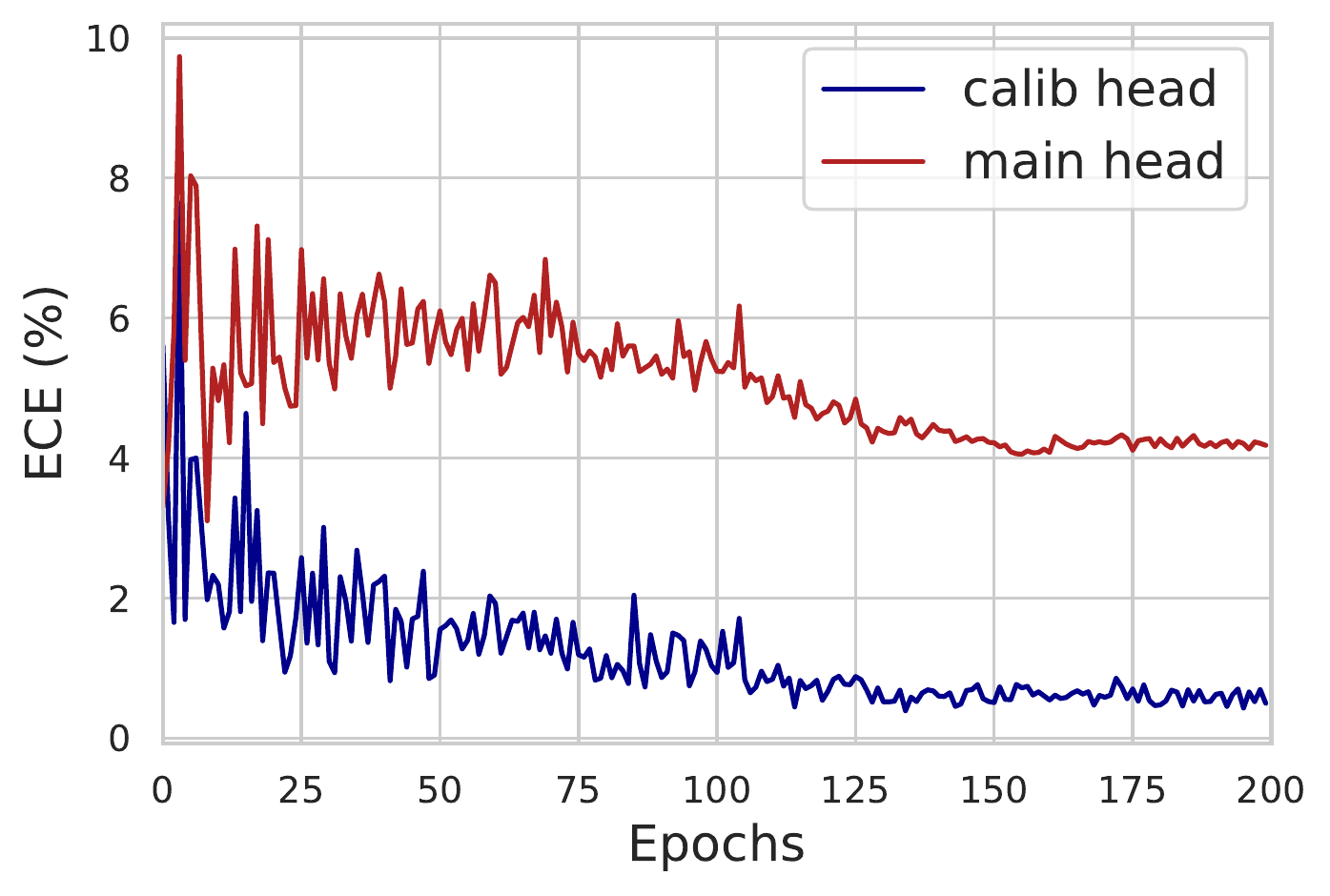}  
  \caption{ECE}
  \label{fig: ece_main_calib}
\end{subfigure}
\begin{subfigure}{0.49\columnwidth}
  \centering
  \includegraphics[width=0.95\linewidth]{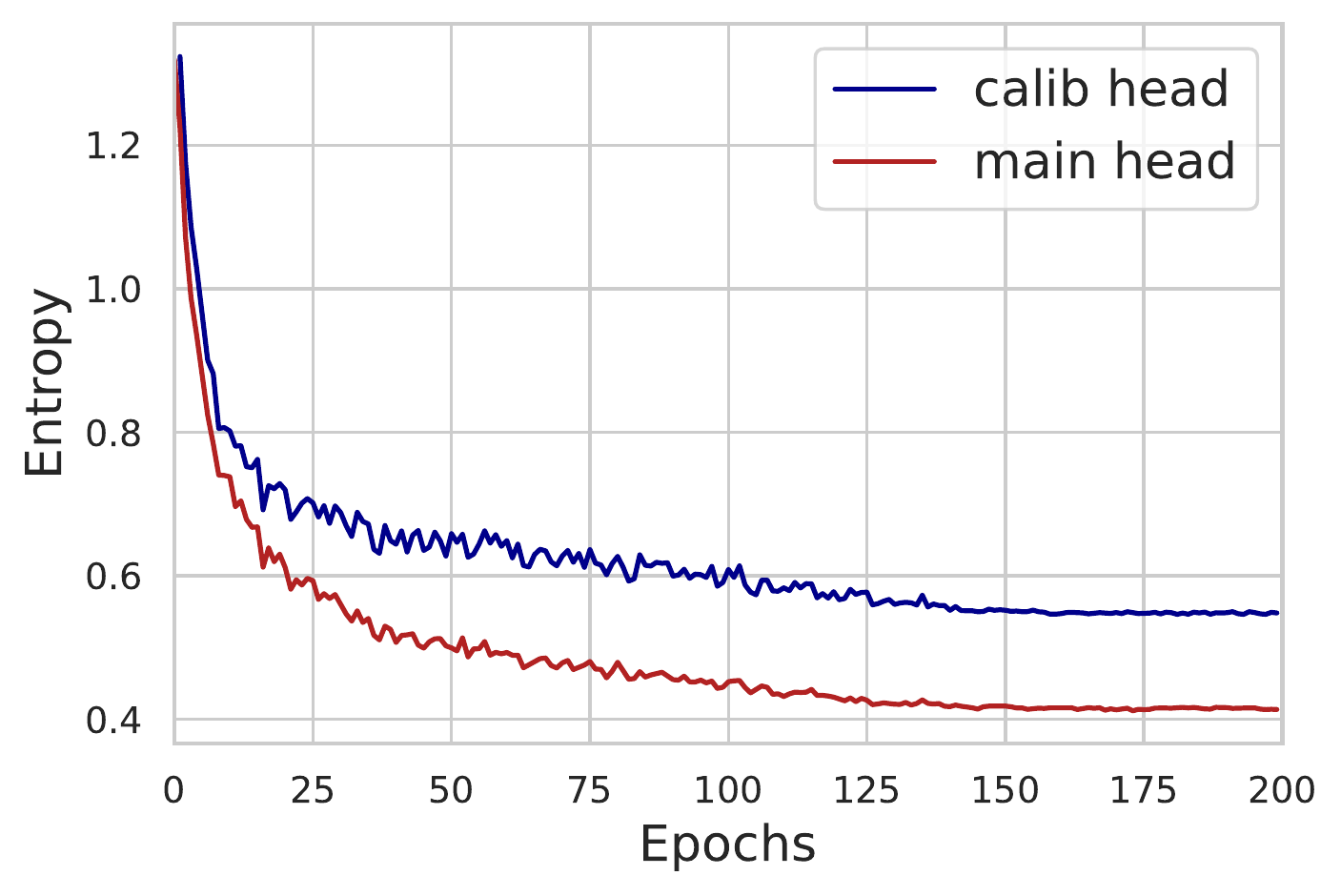} 
  \caption{Entropy}
  \label{fig: entropy_main_calib}
\end{subfigure}\\
\centering
\begin{subfigure}{0.49\columnwidth}
  \centering
  \includegraphics[width=0.95\linewidth]{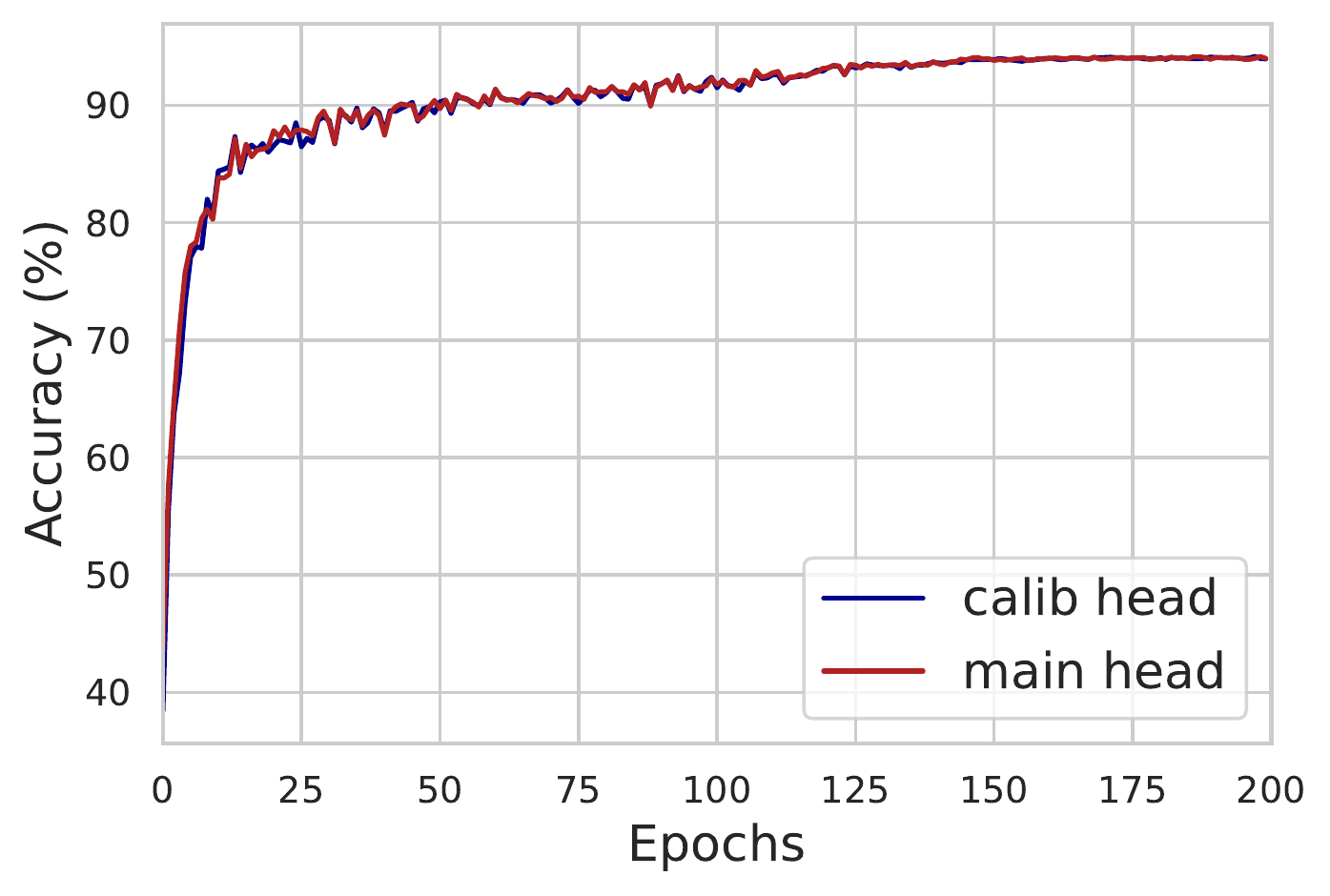} 
  \caption{Accuracy}
  \label{fig: acc_main_calib}
\end{subfigure}
\begin{subfigure}{0.49\columnwidth}
  \centering
  \includegraphics[width=0.95\linewidth]{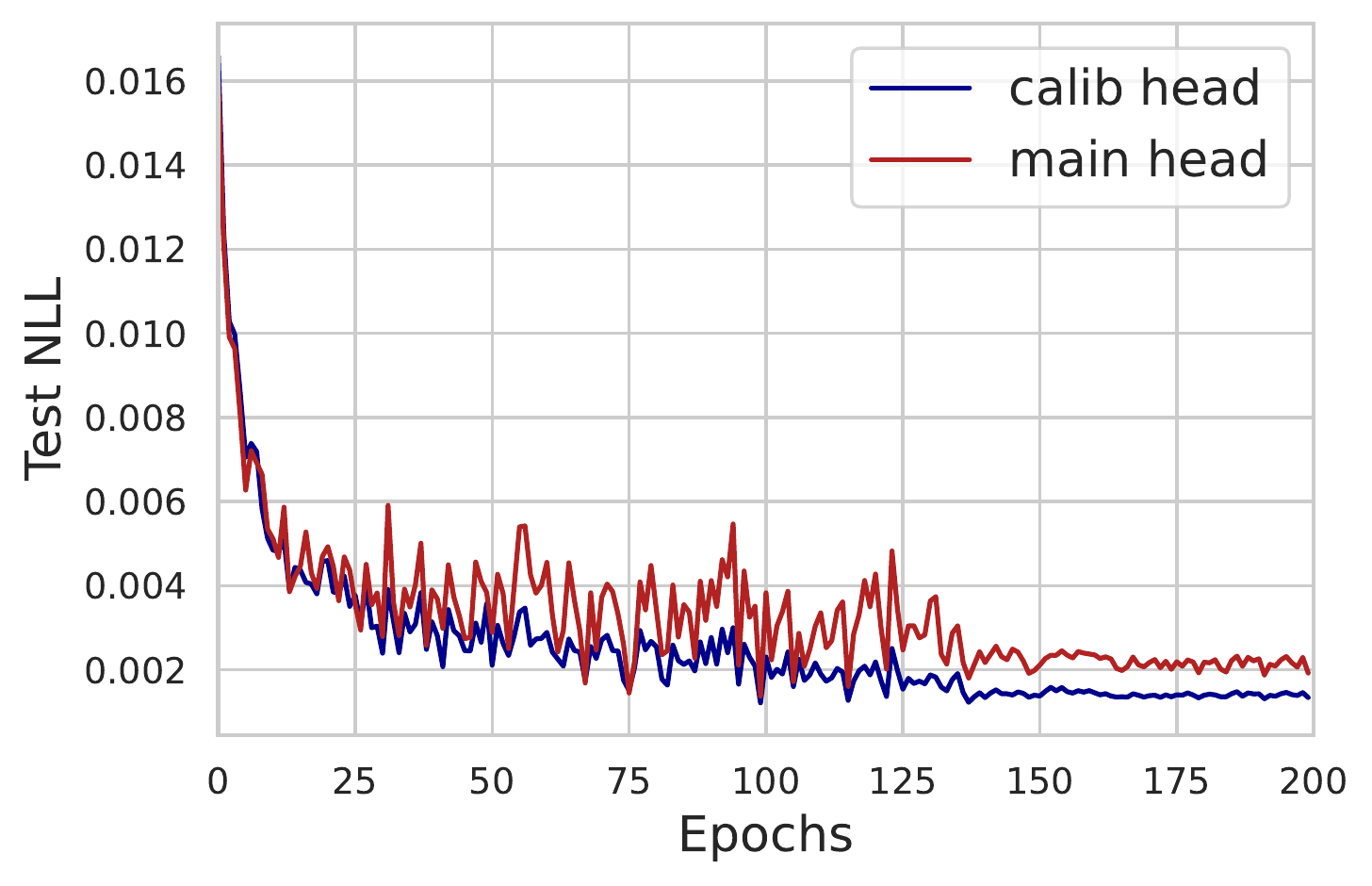} 
  \caption{NLL}
  \label{fig: nll_main_calib}
\end{subfigure}\\
\centering
\begin{subfigure}{0.49\columnwidth}
  \centering
  \includegraphics[width=0.95\linewidth]{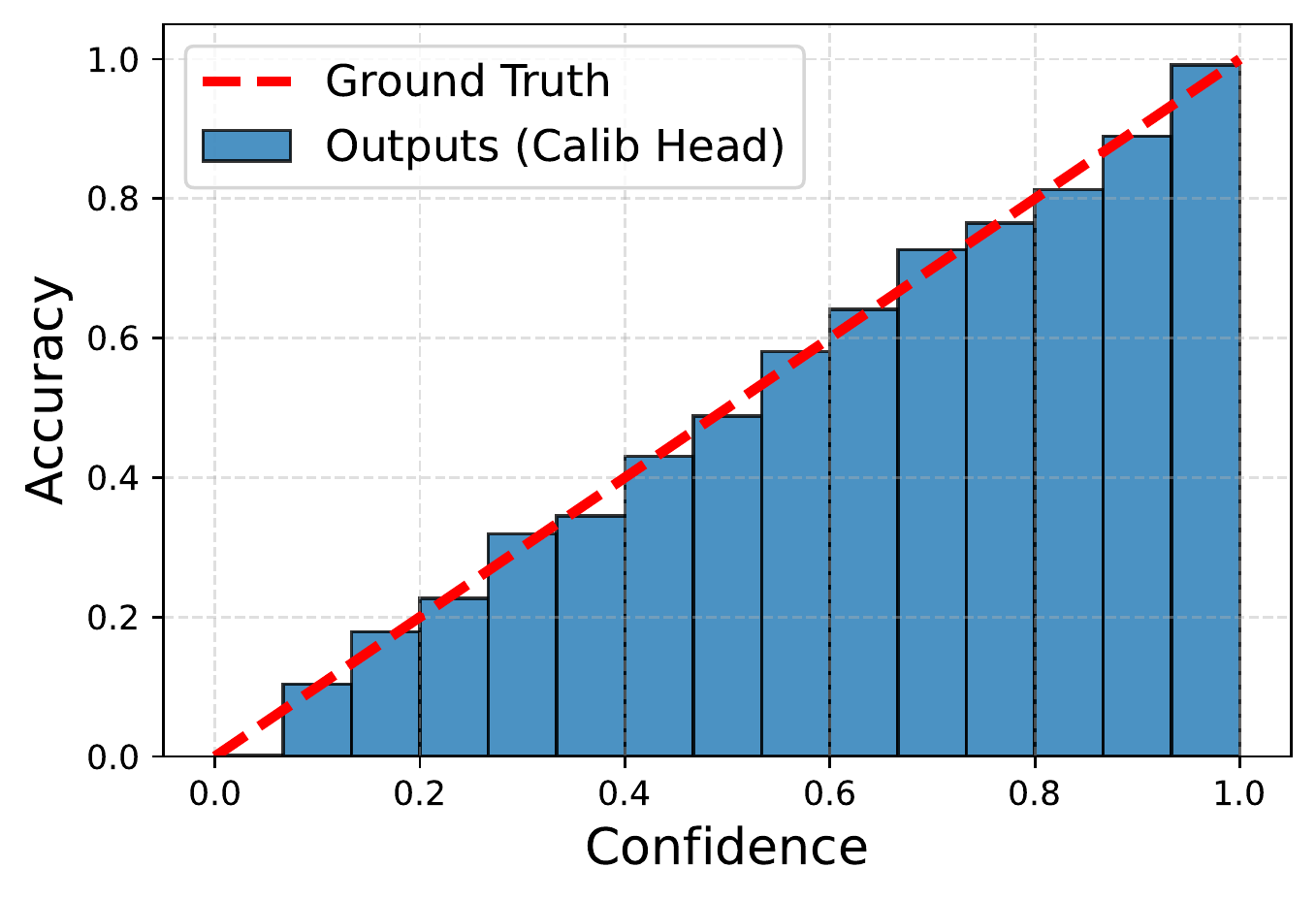} 
  \caption{Calibration Head}
  \label{fig: relib_hist_calib}
\end{subfigure}
\begin{subfigure}{0.49\columnwidth}
  \centering
  \includegraphics[width=0.95\linewidth]{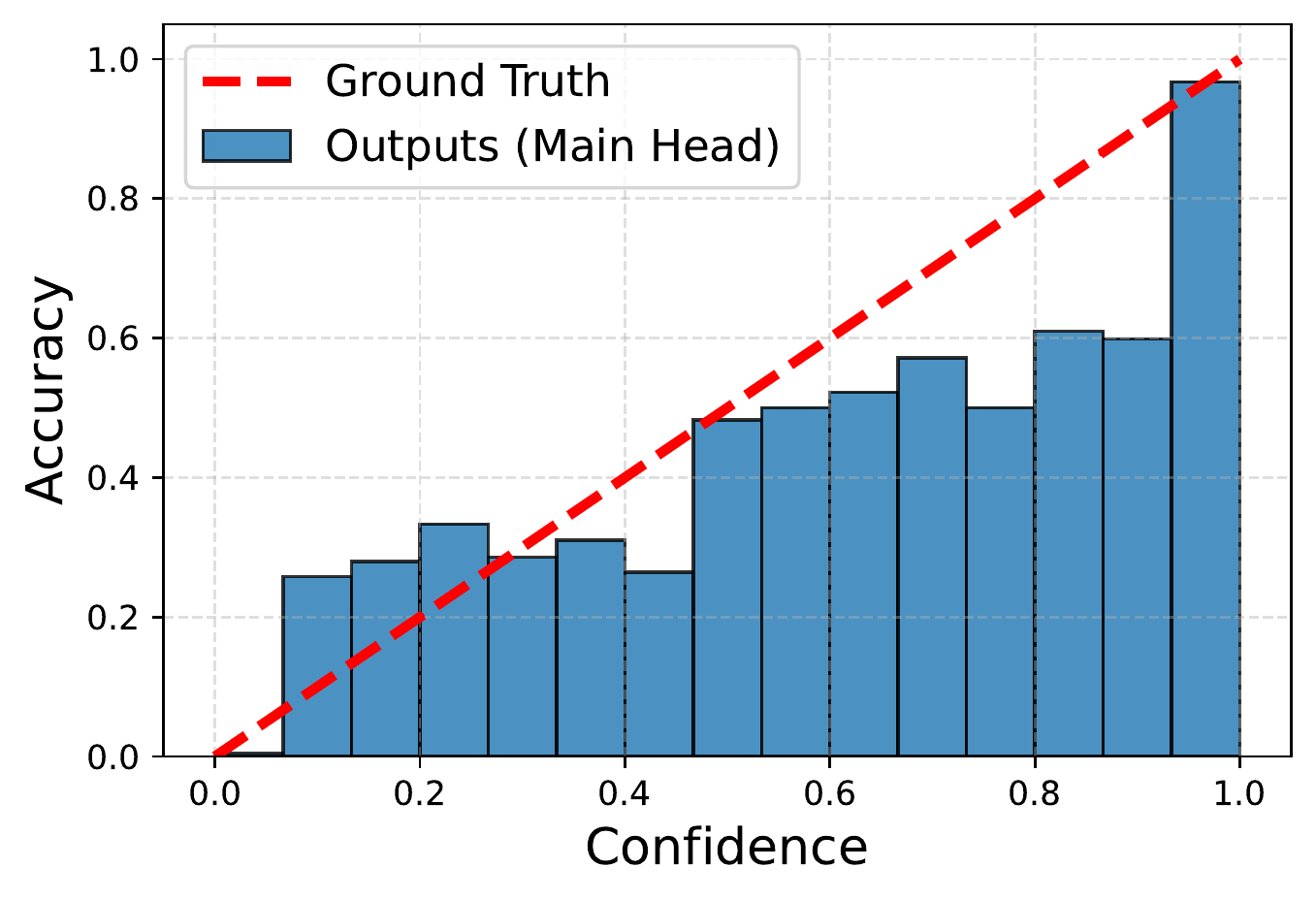} 
  \caption{Main Head}
  \label{fig: relib_hist_main}
\end{subfigure}
\caption{ECE, entropy, accuracy and NLL of ResNet-50 on CIFAR-10 versus epoch are plotted in (a), (b), (c) and (d) respectively. In (c) and (d), the reliability diagrams of calibration head and main head are displayed.}
\label{fig: main_vs_calib}
\vskip -0.2in
\end{figure}


\subsection{Annealing}
\label{subsec: annealing}
\begin{table*}[t!]
\begin{centering}
\begin{small}
\begin{sc}
\resizebox{\textwidth}{!}{
\begin{tabular}{c c cc cc cc cc cc}
\toprule
\multirow{2}{*}{\textbf{Dataset}} & \multirow{2}{*}{\textbf{Model}} & \multicolumn{2}{c}{\textbf{Cross Entropy}} & \multicolumn{2}{c}{\textbf{MMCE}} & \multicolumn{2}{c}{\textbf{Brier Loss}} & \multicolumn{2}{c}{\textbf{Focal Loss}} & \multicolumn{2}{c}{\textbf{ADH}} \\
&  & pre T & post T & pre T & post T & pre T & post T & pre T & post T & pre T & post T \\
\midrule
\multirow{2}{*}{CIFAR-10/CIFAR-10-C} 
& ResNet 101 & 73.83 & 73.99 & 84.31 & \bf{84.77} & 73.68 & 75.95 & 81.19 & 81.28 & 82.76 & 82.93 \\
& Wide ResNet 28-10 & 74.73 & 75.87 & 83.54 & 84.54 & 91.40 & 92.23 & 91.57 & 91.69 & 92.60 & \bf{92.63} \\
\midrule
\multirow{2}{*}{CIFAR-10/SVHN}
& ResNet 101 & 86.04 & 85.43 & 87.50 & 89.02 & 92.89 & 94.45 & 89.50 &  91.27 & 93.25 & \bf{94.88} \\
& Wide ResNet 28-10 & 87.00 & 85.48 & 90.94 & 90.64 & 70.88 & 74.68 & 87.25 & 87.38 & 93.75 & \bf{94.08} \\
\bottomrule
\end{tabular}
}

\end{sc}
\end{small}
\end{centering}
\caption{AUROC (\%) for aforementioned approaches on ResNet-50 and ResNet-152 in DS tasks. The training set is CIFAR-10 and the DS sets are the CIFAR-10-C and SVHN.}
\label{tab: auroc}
\vskip -0.2in
\end{table*}

We investigate the property of the Annealing technique in our Annealing Double-Head approach by experimenting with various annealing factors $\beta$. 
In Fig. \ref{fig: annealing}, ECE and entropy of ResNet-50 on CIFAR-10 with different $\beta$ are illustrated. 
The ECE and entropy of the calibration head exhibit the following behaviours. 

As indicated by the preceding plots and the confidence histogram in Fig. \ref{fig: relib_annealing} of Appendix \ref{subsec: varying_beta}, the model is overconfident in its predictions when $\beta$ is small, which leads to poor calibration.  
As $\beta$ increases, so does the entropy of the predictive distribution depicted in Fig. \ref{fig: entropy_vs_T}, since the model tends to distribute the confidences more uniformly among the erroneous categories in the predicted label vectors, resulting in a better calibrated model.

As the annealing factor $\beta$ surpasses the critical point, which is $1.2$ in our experiment (as shown in Fig. \ref{fig: ece_vs_T}), ECE will increase once more, as the model becomes underconfident in its predictions, as evidenced by the high entropy of the predicted distribution in Fig. \ref{fig: entropy_vs_T} and also the reliability plots Fig. \ref{fig: relib_annealing} of Appendix \ref{subsec: varying_beta}. 
Although we only present the results for ResNet-50 on CIFAR-10, these behaviours hold to all models.

\begin{figure}[b!]
\centering
\begin{subfigure}{0.49\columnwidth}
  \centering
  \includegraphics[width=0.95\linewidth]{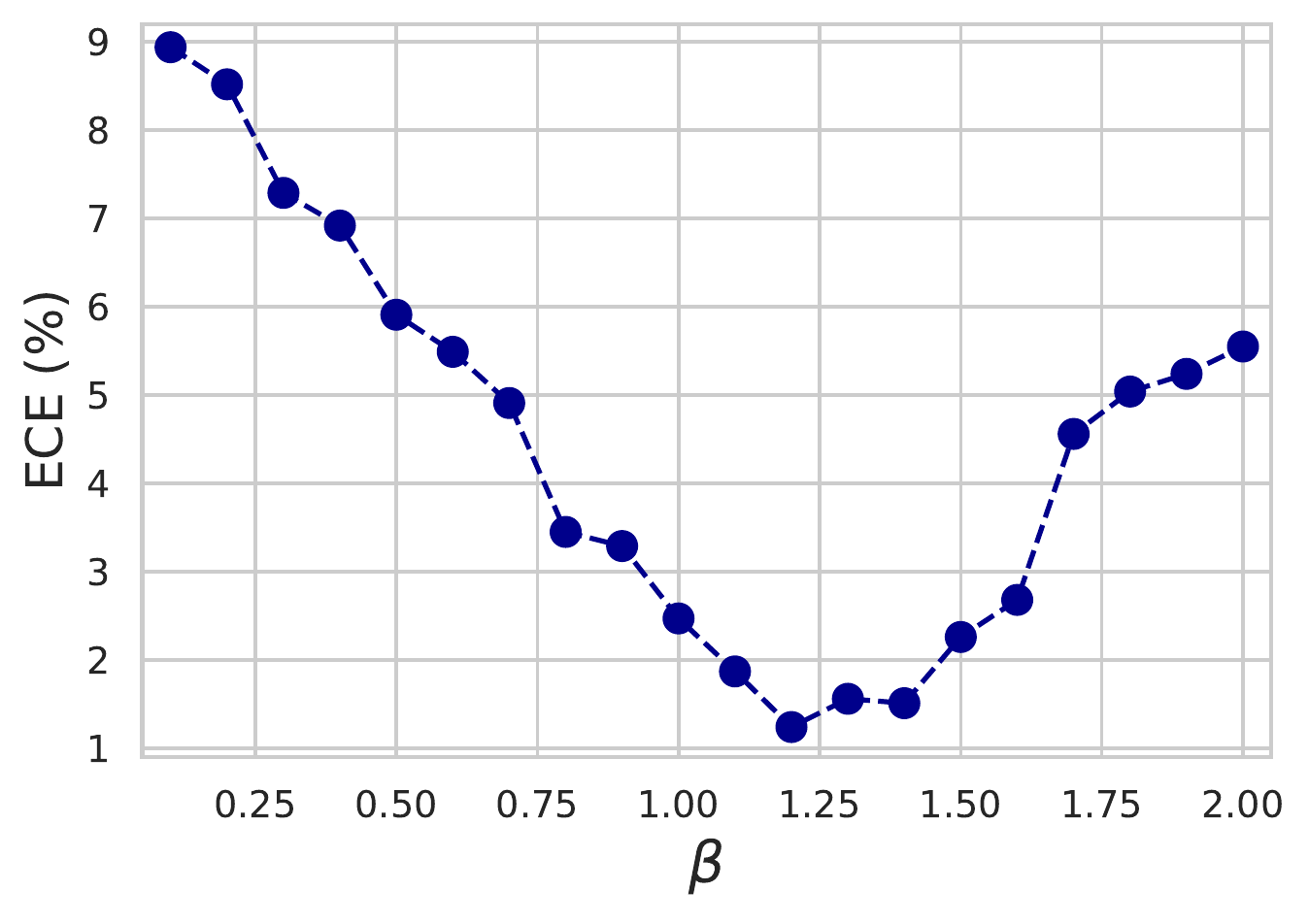}  
  \caption{ECE vs $\beta$}
  \label{fig: ece_vs_T}
\end{subfigure}
\begin{subfigure}{0.49\columnwidth}
  \centering
  \includegraphics[width=0.95\linewidth]{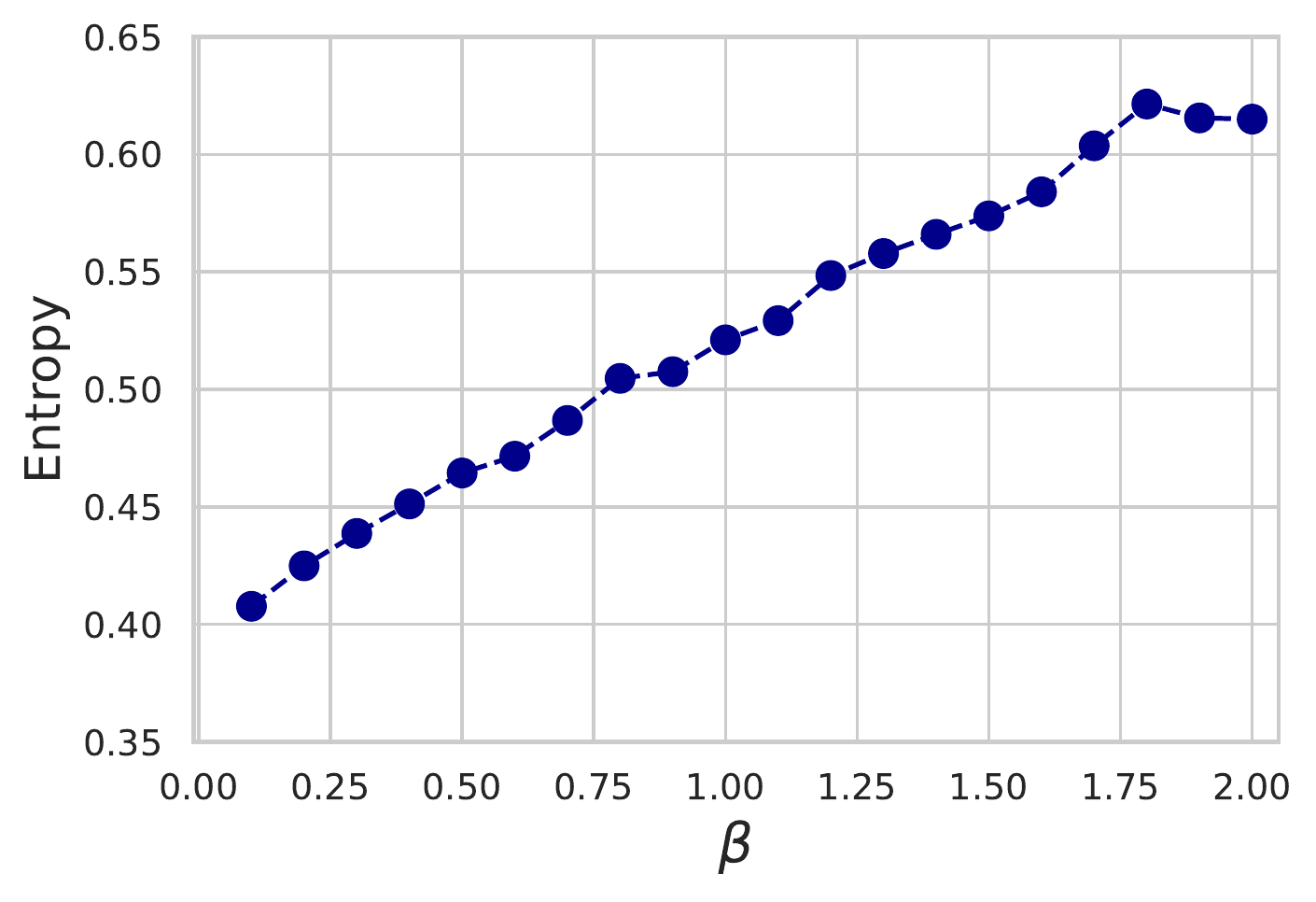} 
  \caption{Entropy vs $\beta$}
  \label{fig: entropy_vs_T}
\end{subfigure}
\caption{(a) and (b) illustrate the ECE and entropy, respectively, as a function of increasing $\beta$.}
\label{fig: annealing}
\vskip -0.23in
\end{figure}

\subsection{Calibration under Distributional Shift} 

Precise estimates of predictive confidence should be provided not only for in-distribution samples, but also for out-of-distribution samples, which are frequently encountered in practical applications.
We investigate the model calibration by previously discussed approaches and our own method in the DS scenario.

In our implementation, we adopt two types of distributional shift: (a) shifting the test inputs (CIFAR-10) with a Gaussian noise perturbation of severity $5$, specifically CIFAR-10-C, \citep{hendrycks2018benchmarking}, and (b) using an entirely different test dataset, SVHN. In both instances, we train ResNet-50 and Wide-ResNet-28-10 on CIFAR-10 using the listed calibration methods and obtain the AUROC before and after Temperature Scaling (TS). The findings are presented in Table. \ref{tab: auroc}. As stated in \citep{ovadia2019can}, TS does not perform well in the distributional shift situation, even if the shift is modest, and our results confirm this, as TS post-processing does not always result in a significant rise in AUROC. Even though our method is not based on regularization, its performance is comparable to the other three methods and generally excellent in all four cases.

Our method provides an efficient way to construct ensemble of models by introducing multiple heads into the base model, which is a potential method for model uncertainty summarization (bayesian model averaging).
\subsection{Time Complexity and Implementation}
\label{subsec: complexity}
We alternate the training of the main head and the calibration head by interleaving every $k$ steps of training for the main head with one step of training for the calibration head. 
Therefore, it is straightforward to estimate that the additional computing complexity cost for model calibration does not exceed $O(\frac{1}{k})$ of the normal training process.
In Appendix \ref{subsec: complexity_app}, we examine the performance of our ADH method as a function of the calibration period $k$ and demonstrate that setting $k=70$ still yields excellent performance while only causing a $1.3\%$ increase in computation cost.

Our Annealing Double-Head is quite simple to implement, and Fig. \ref{fig: code_adh} of Appendix \ref{subsec: complexity_app} shows an implementation with about thirty lines of code in Pytorch.



%
%



\section{Conclusion and Future Work}
\label{sec: conclusion}
In this paper, we present Annealing Double-Head, a simple but effective architecture for calibrating DNNs. 
We extensively evaluate the performance of our approach for model calibration by comparing it to various existing high-performance approaches, such as MMCE, Brier Loss, and Focal Loss, utilizing a variety of recent neural network architectures and widely used vision and speech datasets, and demonstrate that our approach achieves the state-of-art model calibration performance, which is the best among the aforementioned methods.
We empirically examine the dynamical behaviours of both the calibration head and the main head during training and quantitatively validate our model design reasoning.
We demonstrate theoretically that the Annealing technique leads to adaptive gradients during training and thoroughly investigated its properties through a variety of experiments involving altering annealing factor.
As is noticed by previous works such as \citep{pereyra2017regularizing, mukhoti2020calibrating}, larger entropy of confidence distribution may result in improved model calibration. 
We verify this assumption by showing that $\text{ECE}_{2}$ is bounded by the negative entropy of the predictive confidence distribution and a constant term which is a function of predictive confidence.
Finally, we investigate the performance of our approach in the setting of distributional shift and demonstrate that it yields comparable performance to other regularization-based methods.

\section*{Acknowledgements}
We acknowledge J. Xavier Prochaska and Ila Fiete for their insightful discussions.
Erdong Guo would like to thank ITP, CAS, for their gracious hospitality while this study was being completed, Pan Zhang for providing the computational resources and Akhilan Boopathy for useful discussions.
\nocite{}

\bibliography{references}
\bibliographystyle{icml2022}

\newpage
\appendix
\onecolumn
\section{Proofs}
\label{sec: pfs}
\subsection{Proof of Theorem \ref{the: annealing_gradient}}
\label{subsec: grad_proof}
\begin{theorem}    
Let $\mathcal{L}(\mathbf{z}, y)$, denote the cross-entropy loss of a sample pair, 
$(\mathbf{x}, y)$ denote a sample and its label, and $\mathbf{z} =f_{\theta}(\mathbf{x})$ represent the logit of the sample $\mathbf{x}$. 
Given the assumption that $y = \arg\max_{j}(\mathbf{z}_{j})$, we have
\begin{align*}
&\partial_{\theta}{\mathcal{L}(\beta \mathbf{z}, y)} = \sum_{j}\gamma_{j}\frac{\partial{\mathcal{L}(\mathbf{z}, y)}}{\partial{z_{j}}}\partial_{\theta}{z_{j}}, \\
&\begin{cases}
c_{0, i} \leq \gamma_{i} \leq c_{1, i} \quad \text{ if } i = y \\
c_{2, i} \leq \gamma_{i} \leq c_{3, i} \quad \text{ if } i \neq y
\end{cases}
\end{align*}
where
\begin{align*}
&c_{0, i} = \frac{1}{n}\exp{[(1-\beta)(z_{i} - \mathbf{z}_{(1)}) - (\mathbf{z}_{(n-1)} - \mathbf{z}_{(1)})]} \\
&c_{1, i} = n\exp{[(1-\beta)(z_{i} - \mathbf{z}_{(n-1)}) + (\mathbf{z}_{(n-1)} - \mathbf{z}_{(1)})]} \\
&c_{2, i} = (\frac{1}{n} + \frac{n-1}{n}\exp{(\mathbf{z}_{(1)} - \mathbf{z}_{(n)})})\exp{[(\beta - 1)z_{i}]}\\
&c_{3, i} = (1 + (n - 1)\exp{(\mathbf{z}_{(n-1)} - \mathbf{z}_{(n)})})\exp{[(\beta - 1)z_{i}]}.
\end{align*}
\end{theorem}
\begin{proof}
The definition of cross-entropy loss $\mathcal{L}(\mathbf{z}, y)$ is as 
\begin{align}
\mathcal{L}(\mathbf{z}, y) = \sum_{i}y_{i}(z_{i} - \log{\sum_{k}\exp{(z_{k})}}),
\end{align}
here we use the one-hot encoding of $y$ and assign the true class component of the encoding vector $\mathbf{y}$ to $1$. Then we have
\begin{align}
\frac{\partial{\mathcal{L}(\mathbf{z}, y)}}{\partial{z_{j}}} = 
\begin{cases}
\frac{\sum_{i\neq j}\exp{(z_{i}})}{\sum_{i}\exp{(z_{i})}} \quad \text{ if } j = y \\
-\frac{\exp{(z_{j}})}{\sum_{i}\exp{(z_{i})}} \quad \text{ if } j \neq y.
\end{cases}
\end{align}
For the rescaled confidence, we can derive following relation, 
\begin{align}
\partial_{\theta}{\mathcal{L}(\beta\mathbf{z}, y)} &= \sum_{j}\beta\frac{\partial{\mathcal{L}(\beta\mathbf{z}, y)}}{\partial{\beta z_{j}}}\partial_{\theta}{z_{j}}\\
&= \sum_{j}\gamma_{j}\beta\frac{\partial{\mathcal{L}(\mathbf{z}, y)}}{\partial{z_{j}}}\partial_{\theta}{z_{j}}, 
\end{align}
where
\begin{align*}
\gamma_{j} = \frac{\partial{\mathcal{L}(\beta\mathbf{z}, y)}}{\partial{\beta z_{j}}} /\frac{\partial{\mathcal{L}(\mathbf{z}, y)}}{\partial{z_{j}}} = 
\begin{cases}
\frac{\sum_{i\neq j}\exp{(\beta z_{i}})\sum_{i}\exp{(z_{i})}}{\sum_{i}\exp{(\beta z_{i})}\sum_{i\neq j}\exp{(z_{i}})} & \text{ if } j = y, \\
\exp{[(\beta - 1)z_{j}]}\frac{\sum_{i}\exp{(z_{i})}}{\sum_{i}\exp{(\beta z_{i})}} & \text{ if } j \neq y.
\end{cases}
\end{align*}

First, consider the situation when $j=y$, 
\begin{align}
\gamma_{j} &\geq \frac{(n - 1)\exp{(\beta \mathbf{z}_{(1)})}}{\exp{(\beta z_{j})} + (n - 1)\exp{(\beta\mathbf{z}_{(n-1)})}}\frac{\exp{(z_{j})} + (n - 1)\exp{(\mathbf{z}_{(1)})}}{(n - 1)\exp{(\mathbf{z}_{(n-1)})}} \\
&\geq \exp{(\beta\mathbf{z}_{(1)} - \mathbf{z}_{(n-1)})}\frac{1}{n}\exp{[(1 - \beta)z_{j}]} \\
&= \frac{1}{n}\exp{[(1-\beta)(z_{j} - \mathbf{z}_{(1)}) - (\mathbf{z}_{(n-1)} - \mathbf{z}_{(1)})]}, \\
\gamma_{j} &\leq \frac{(n - 1)\exp{(\beta \mathbf{z}_{(n-1)})}}{\exp{(\beta z_{j})} + (n - 1)\exp{(\beta\mathbf{z}_{(n-1)})}}\frac{\exp{(z_{j})} + (n - 1)\exp{(\mathbf{z}_{(n-1)})}}{(n - 1)\exp{(\mathbf{z}_{(1)})}} \\
&\leq \exp{(\beta\mathbf{z}_{(n-1)} - \mathbf{z}_{(1)})}n\exp{[(1 - \beta)z_{j}]} \\
&= n\exp{[(1-\beta)(z_{j} - \mathbf{z}_{(n-1)}) + (\mathbf{z}_{(n-1)} - \mathbf{z}_{(1)})]}.
\end{align}
Similarly, in the case when $j \neq y$, we have
\begin{align}
\gamma_{j} &\geq \exp{[(\beta - 1)z_{j}]}\frac{\exp{(\mathbf{z}_{(n)})} + (n - 1)\exp{(\mathbf{z}_{(0)})}}{n\exp{(\beta \mathbf{z}_{(n)})}}\\ 
&= (\frac{1}{n} + \frac{n-1}{n}\exp{(\mathbf{z}_{(1)} - \mathbf{z}_{(n)})}) \exp{[(1 - \beta)(\mathbf{z}_{(n)} - z_{i})]}, \\
\gamma_{j} &\leq \exp{[(\beta - 1)z_{j}]}\frac{\exp{(\mathbf{z}_{(n)})} + (n - 1)\exp{(\mathbf{z}_{(n-1)})}}{n\exp{(\beta \mathbf{z}_{(0)})}}\\
& = (1 + (n - 1)\exp{(\mathbf{z}_{(n-1)} - \mathbf{z}_{(n)})})\exp{[(1 - \beta)(\mathbf{z}_{(n)} - z_{i})]}.
\end{align}
\end{proof}

\subsection{Proof of Theorem \ref{the: ece_upper_bound}}
\label{subsec: bound_proof}
\begin{theorem}
\label{the: ece_upper_bound_v1}
Let $\hat{p}(\cdot)$ denote a function which maps a sample $\mathbf{x}$ to  its confidence $\hat{p}(\mathbf{x})$, then the second order ECE: $\text{ECE}_{2}$, which is defined as the square root of the second moment of absolute difference between the confidence and accuracy can be bounded as follows,  
\begin{align*}
    \text{ECE}_{2}[\hat{p}(\mathbf{x})] \leq 
    & (\text{C}[\hat{p}(x)] - 2\text{H}[\hat{p}(x)])^{1/2}, 
\end{align*}
where 
\begin{align*}
&\text{C}[\hat{p}(x)] = 3 - 2\mathbb{E}_{\hat{p}(x)}[\log{\hat{p}(x)}] - \mathbb{E}_{\hat{p}(x)}[f(\hat{p}(x))], \\
&\text{H}[\hat{p}(\mathbf{x})] = \mathbb{E}_{\hat{p}(\mathbf{x})}[\log{f(\hat{p}(\mathbf{x})}].
\end{align*}
\end{theorem}
\begin{proof}
According to the definition the second order ECE, its square is derived as 
\begin{align}    
\label{eq: ece_square}
(\text{ECE}_{2}[\hat{p}(\mathbf{x})])^{2} = & \sum_{\hat{p}(\mathbf{x})}f(\hat{p}(\mathbf{x}))[f(\hat{y}(\mathbf{x}) = y(\mathbf{x})|\hat{p}(\mathbf{x})) - \hat{p}(\mathbf{x})]^{2}\\
=& \sum_{\hat{p}(\mathbf{x})}f(\hat{p}(\mathbf{x}))(f(\hat{y}(\mathbf{x}) = y(\mathbf{x})|\hat{p}(\mathbf{x}))^{2} - 2f(\hat{p}(\mathbf{x}))f(\hat{y}(\mathbf{x}) = y(\mathbf{x})|\hat{p}(\mathbf{x}))\hat{p}(\mathbf{x}) )\\
& + f(\hat{p}(\mathbf{x}))\hat{p}(\mathbf{x})^{2}).
\end{align}
For the first term, we get
\begin{align}
\label{eq: first_term}
\sum_{\hat{p}(\mathbf{x})}f(\hat{p}(\mathbf{x}))(f(\hat{y}(\mathbf{x}) = y(\mathbf{x})|\hat{p}(\mathbf{x}))^{2} &\leq \sum_{\hat{p}(\mathbf{x})}f(\hat{p}(\mathbf{x}))(f(\hat{y}(\mathbf{x}) = y(\mathbf{x})|\hat{p}(\mathbf{x}))\\
&=\text{Acc}.
\end{align}
For the third term, we have 
\begin{align}
\label{eq: third_term}
\sum_{\hat{p}}{f(\hat{p}(\mathbf{x}))\hat{p}(\mathbf{x})^{2}} \leq 1. 
\end{align}
We have, for the second term, 
\begin{align}
&\sum_{\hat{p}(\mathbf{x})}2f(\hat{p}(\mathbf{x}))f(\hat{y}(\mathbf{x}) = y(\mathbf{x})|\hat{p}(\mathbf{x}))\hat{p}(\mathbf{x}) \\
= &\sum_{\hat{p}(\mathbf{x})}2f(\hat{p}(\mathbf{x}))\frac{f(\hat{y}(\mathbf{x}) = y(\mathbf{x}), \hat{p}(\mathbf{x}))}{f(\hat{p}(\mathbf{x}))}\hat{p}(\mathbf{x})\\
\geq & 
\label{eq: three_term}
\sum_{\hat{p}(\mathbf{x})}2f(\hat{p}(\mathbf{x}))\log{[f(\hat{y}(\mathbf{x}) = y(\mathbf{x}), \hat{p}(\mathbf{x}))]} - \sum_{\hat{p}(\mathbf{x})}2f(\hat{p}(\mathbf{x}))\log{[f(\hat{p}(\mathbf{x}))]}\\
& + \sum_{\hat{p}(\mathbf{x})}2f(\hat{p}(\mathbf{x}))\log{[\hat{p}(\mathbf{x})]}.
\end{align}
Let's analyse each term in \cref{eq: three_term}. 
By the Bonferroni inequalities, we obtain the lower bound of the first term as 
\begin{align*}
& \sum_{\hat{p}(\mathbf{x})}2f(\hat{p}(\mathbf{x}))\log{[f(\hat{y}(\mathbf{x}) = y(\mathbf{x}), \hat{p}(\mathbf{x}))]} \\
\geq & 2\sum_{\hat{p}(\mathbf{x})}f(\hat{p}(\mathbf{x}))(1 - \frac{1}{\text{Acc} + f(\hat{p}(\mathbf{x})) - 1}) \\ 
\geq & \sum_{\hat{p}(\mathbf{x})}f(\hat{p}(\mathbf{x}))\text{Acc} + (f(\hat{p}(\mathbf{x})))^{2} - 2 \\
= & \text{Acc} - 2 + \sum_{\hat{p}(\mathbf{x})}(f(\hat{p}(\mathbf{x})))^{2}.
\end{align*}
Let's combine all three term in \cref{eq: three_term}, and get 
\begin{align}
\label{eq: second_term}
&\sum_{\hat{p}(\mathbf{x})}2f(\hat{p}(\mathbf{x}))f(\hat{y}(\mathbf{x}) = y(\mathbf{x})|\hat{p}(\mathbf{x}))\hat{p}(\mathbf{x}) \\
\geq & \text{Acc} + 2\text{H}[\hat{p}(x)] - 2 + \sum_{\hat{p}(\mathbf{x})}(f(\hat{p}(\mathbf{x})))^{2} + 2\sum_{\hat{p}(\mathbf{x})}f(\hat{p}(\mathbf{x}))\log{[\hat{p}(\mathbf{x})]}
\end{align}
Finally, we plug \cref{eq: first_term}, \cref{eq: second_term} and \cref{eq: third_term} back into \cref{eq: ece_square}, and get
\begin{align}
(\text{ECE}_{2}[\hat{p}(\mathbf{x})])^{2} \leq 3 - 2\mathbb{E}_{\hat{p}(x)}[\log{\hat{p}(x)}] - \mathbb{E}_{\hat{p}(x)}[f(\hat{p}(x))] - 2 \text{H}[\hat{p}(\mathbf{x})].
\end{align}
\end{proof}

\section{Datasets and Setup for Experiments}
\label{sec: exp_app}
\textbf{Datasets}\hspace{5pt} The datasets used in our experiments are
\begin{itemize}
    \item Street View House Numbers (SVHN) \citep{netzer2011reading}: $32 \times 32$ colored images centered around a single character from house numbers in Google Street View images with $10$ classes. 60,000/13,257/26,032 images for train/validation(calibration)/test.
    \item CIFAR-$10$/CIFAR-$100$ \citep{krizhevsky2009learning}: $32 \times 32$ colored images from $10$ classes. The train/validation(calibration)/test is devided as: 45,000/5,000/10,000. 
    \item ImageNet $2012$ \citep{deng2009imagenet}: Web images from $1000$ classes. 1.3millon/25,000/25,000 images for train/validation(calibration)/test.
    \item Stanford Sentiment Treebank (SST) \citep{socher2013recursive}: $11,855$ single sentences extracted from Movie Reviews, parsed as sentence trees and annotated by sentiment. Each sample includes a binary and a fine grained $5$-classes label. Training/validation/test sets contain 6920/872/1821 documents.
    \item 20 Newsgroups \citep{Lang95}: 20,000 newsgroups documents divided across 20 newsgroups. 15,098/900/3,999 documents for train/validation/test.
\end{itemize}

\textbf{Setup of Experiments}\hspace{5pt}
Normalization, random cropping and padding, and random horizontal flips are the data augmentation procedures utilized for all vision dataset experiments. In the training period, SGD optimizer is applied with Nesterov momentum $0.9$ and weight decay of $5 \times 10^{-4}$ for main head and $5 \times 10^{-5}$ for calib head, respectively. 
For CIFAR-10/100 and SVHN, the cosine annealing schedule \citep{loshchilov2016sgdr} is used for learning rate adaptation and the models are trained with $200$ epochs for the ADH method and the baseline CE, whereas the Multi-Step schedule is adopted and the models are trained with $350$ for MMCE, Brier Loss, and Focal Loss. 
For the ImageNet-2012 dataset, we employ the pretrained models provided by \citep{facebookarchive} and retrain them for $100$ epochs. All other settings are identical to those discussed previously.
In NLP tasks, all models are trained primarily following the training setups outlined in the papers where the models were first proposed. In our experiments, the Adam optimizer \citep{kingma2014adam} is used and trained for $50$ epochs.

\section{Confidence Distribution and Reliability Diagram}
\label{sec: conf_and_relib}
\subsection{Methods Comparison}
\label{subsec: methods_comparison}
In Fig. \ref{fig: relib_methods_compare}, we display the reliability and confidence histograms of ResNet-50 trained on CIFAR-10 and calibrated using the MMCE, BL, FL, and ADH methods. 
In the first row of the figures matrix, the reliability diagrams are displayed. 
Our method yields the most calibrated model, as demonstrated. The models calibrated by the other three approaches continue to be slightly overconfident in their predictions.
In the second row are the histograms of confidence. 
We observe that, among the four calibration methods, ADH provides the most uniform confidence distribution.
To better represent the left bins with small confidences, we exclude the bin with the highest confidence from the confidence histograms in the final row.

\begin{figure}[!htb]
\vskip 0.2in
\centering
\begin{subfigure}{0.24\columnwidth}
  \centering
  \includegraphics[width=0.95\linewidth]{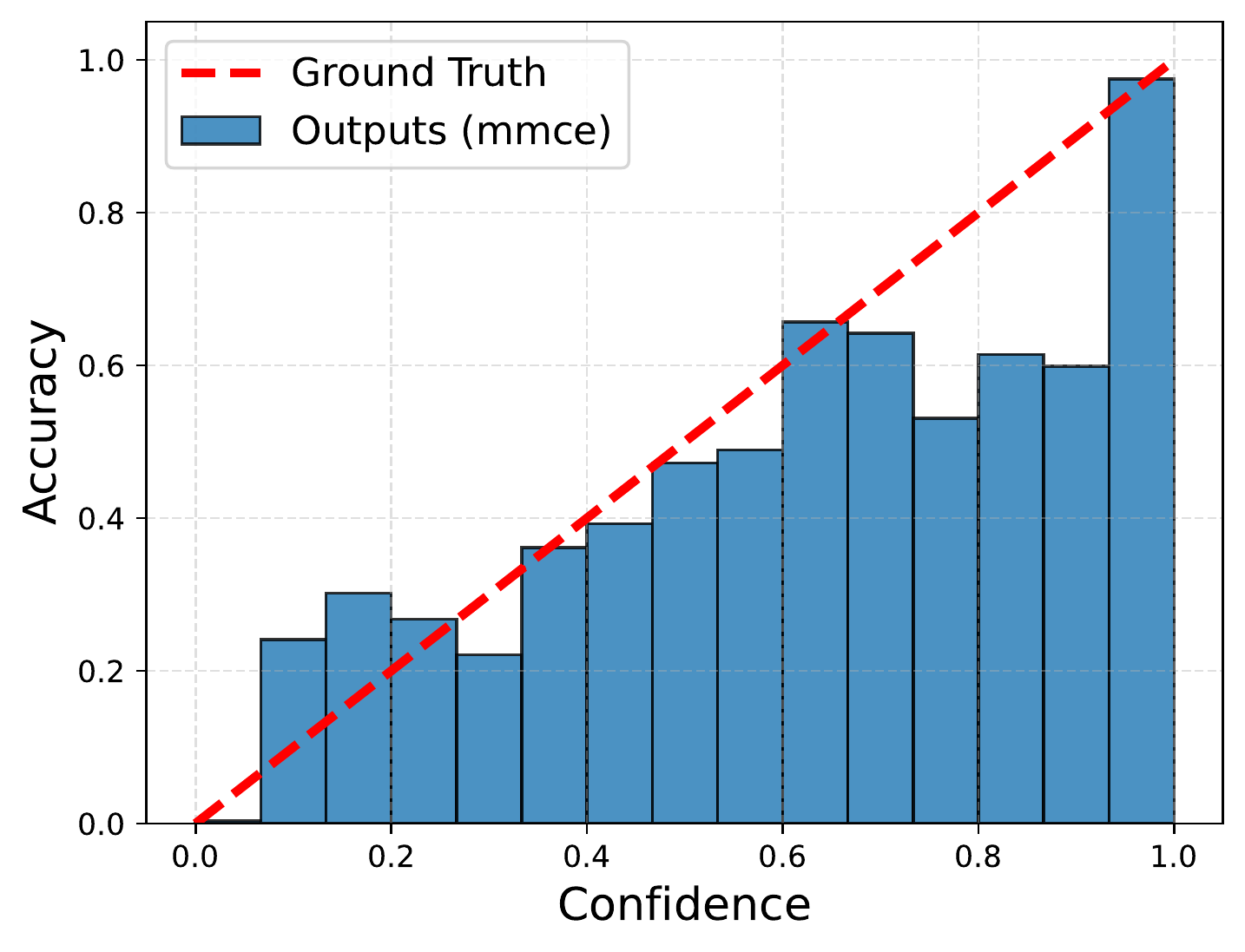}  
  \caption{MMCE}
  \label{fig: mmce_relib}
\end{subfigure}
\begin{subfigure}{0.24\columnwidth}
  \centering
  \includegraphics[width=0.95\linewidth]{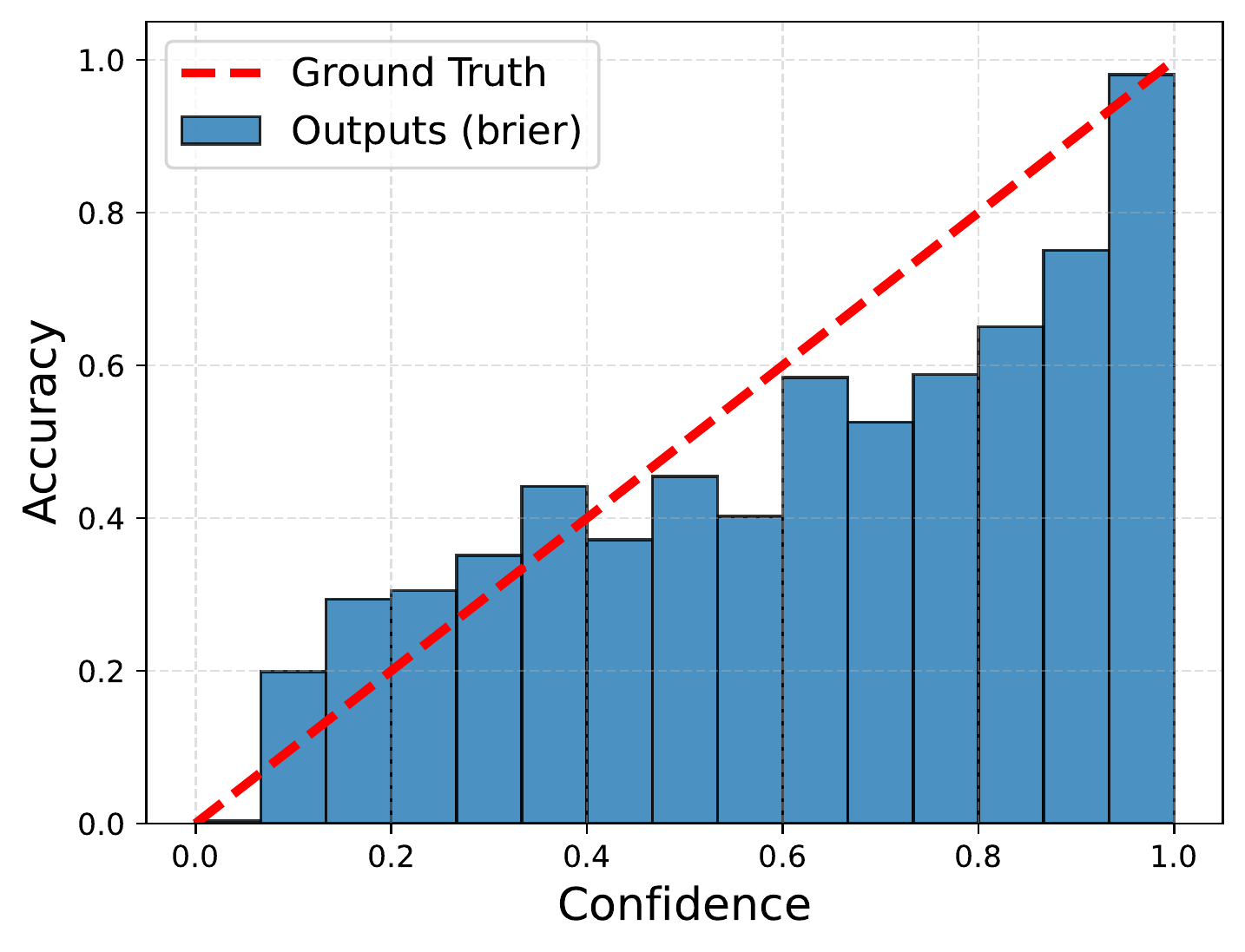} 
  \caption{Brier}
  \label{fig: brier_relib}
\end{subfigure}
\begin{subfigure}{0.24\columnwidth}
  \centering
  \includegraphics[width=0.95\linewidth]{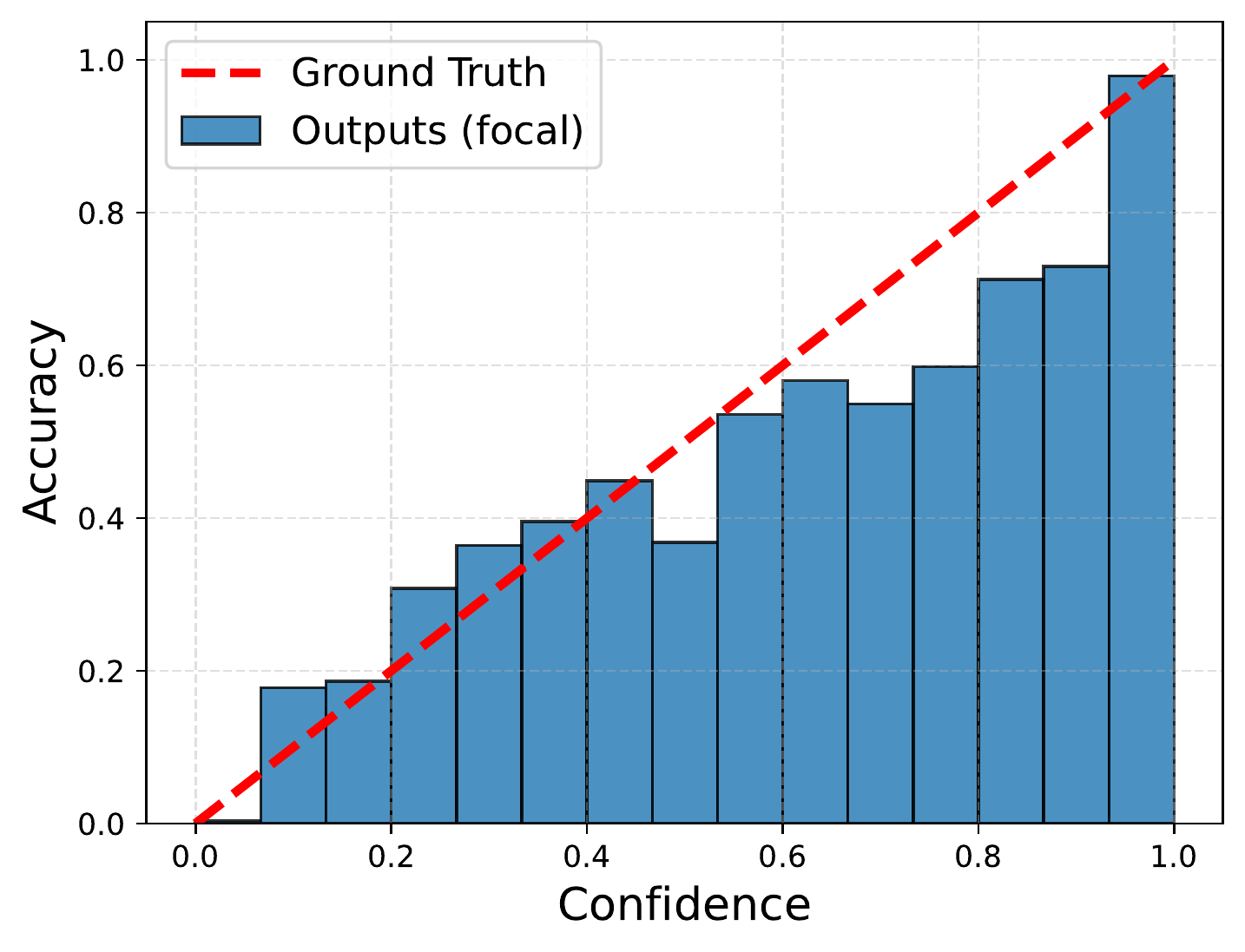}  
  \caption{Focal Loss}
  \label{fig: focal_relib}
\end{subfigure}
\begin{subfigure}{0.24\columnwidth}
  \centering
  \includegraphics[width=0.95\linewidth]{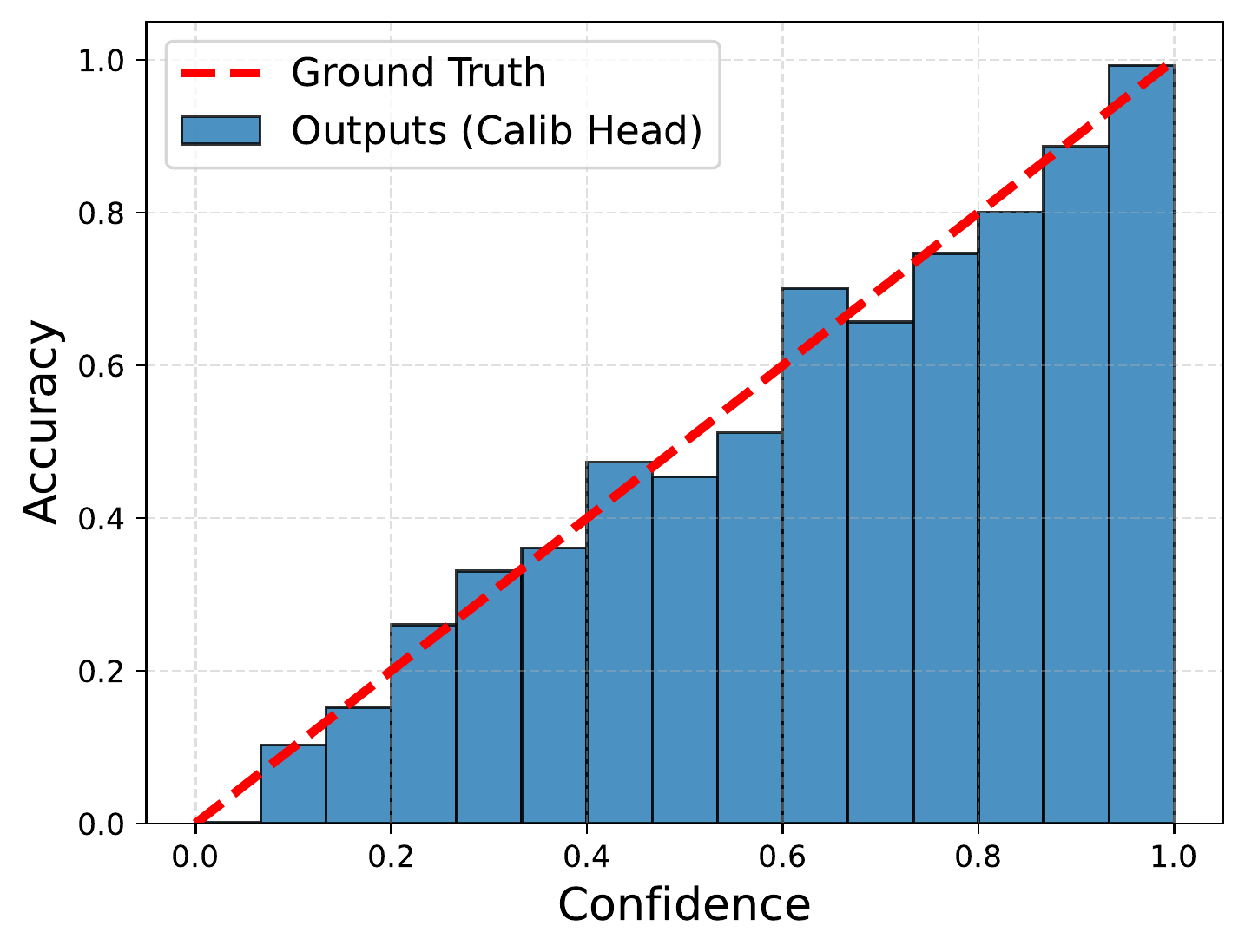} 
  \caption{ADH}
  \label{fig: adh_relib}
\end{subfigure}\\
\centering
\begin{subfigure}{0.24\columnwidth}
  \centering
  \includegraphics[width=0.95\linewidth]{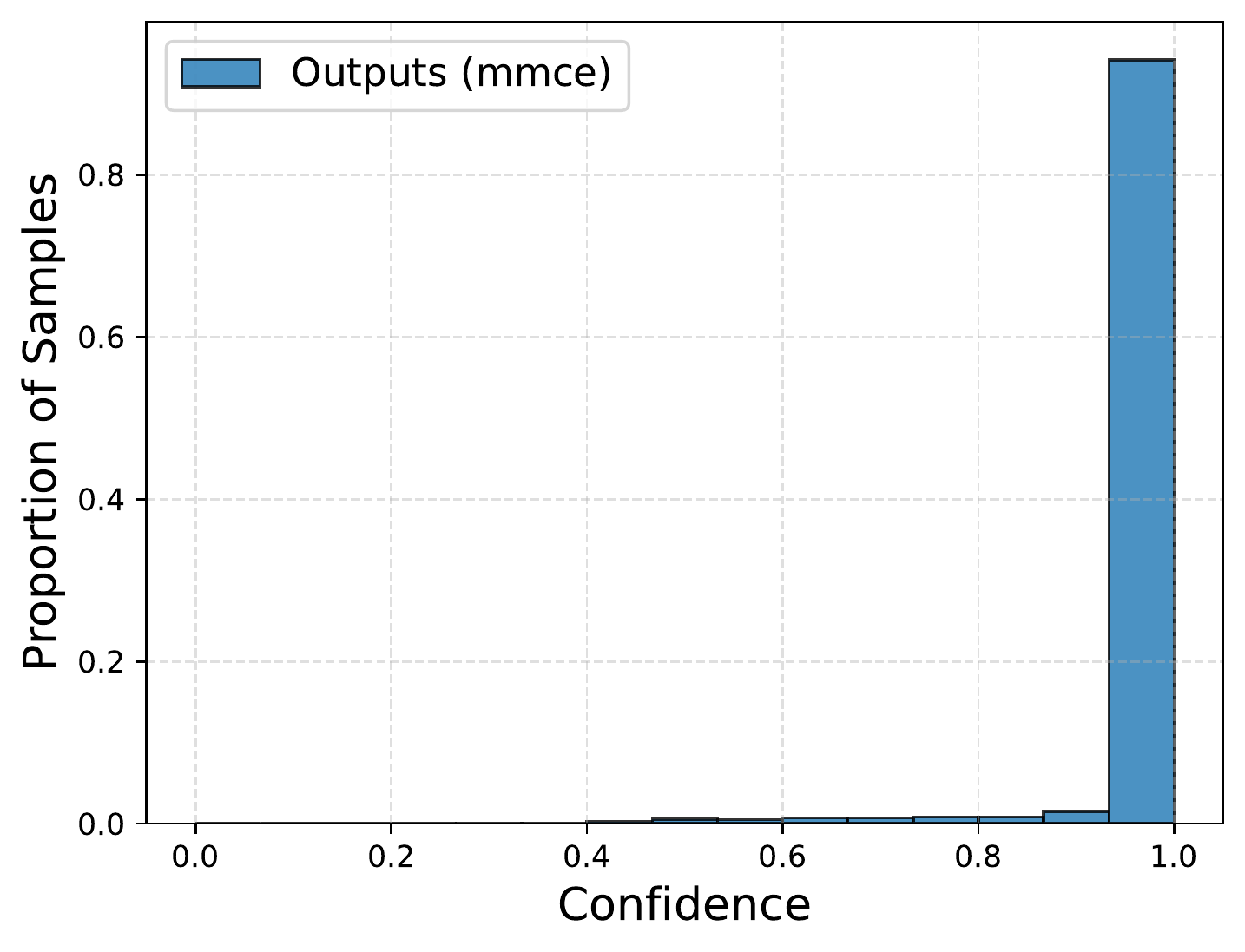}  
  \caption{MMCE}
  \label{fig: mmce_conf}
\end{subfigure}
\begin{subfigure}{0.24\columnwidth}
  \centering
  \includegraphics[width=0.95\linewidth]{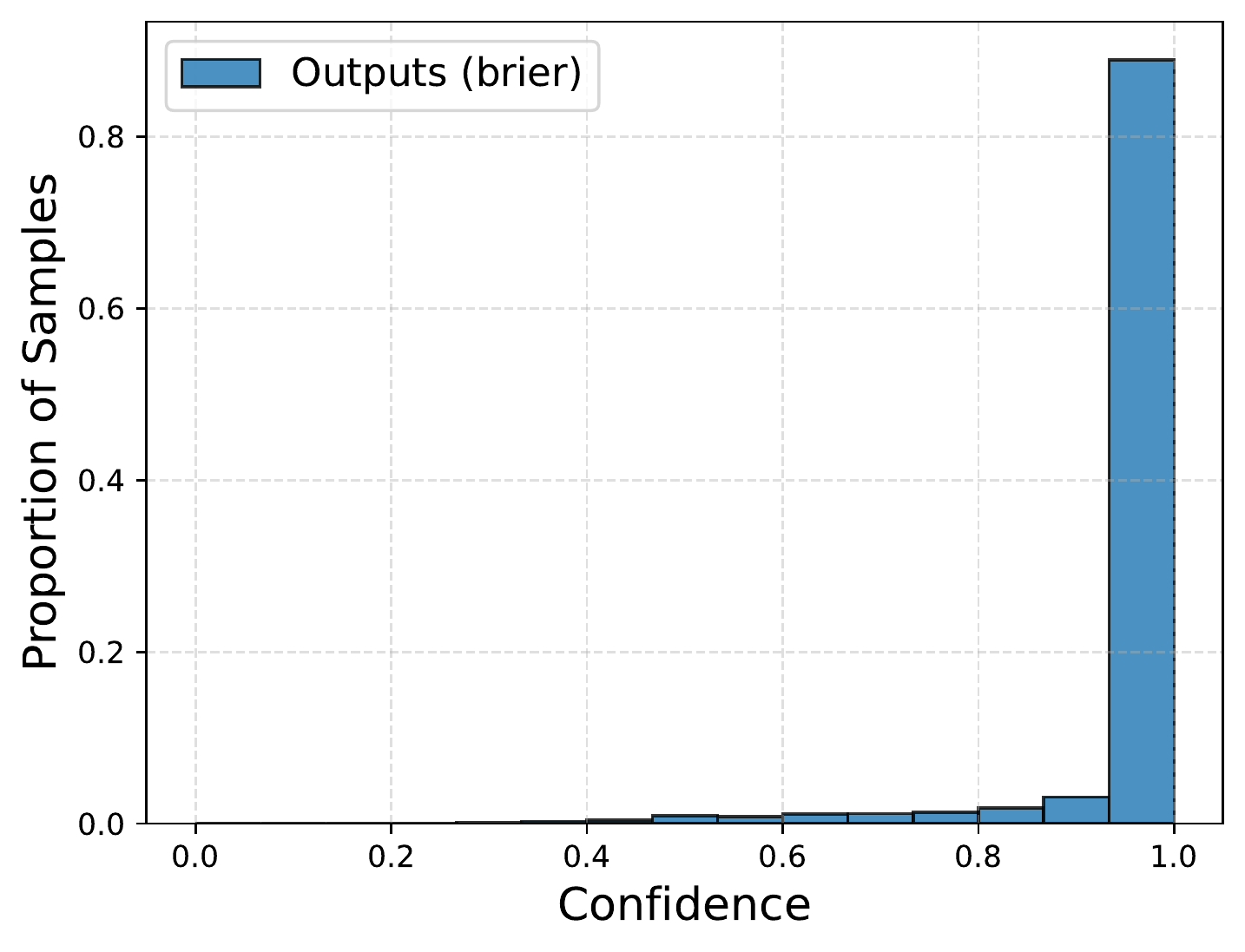} 
  \caption{Brier}
  \label{fig: brier_conf}
\end{subfigure}
\begin{subfigure}{0.24\columnwidth}
  \centering
  \includegraphics[width=0.95\linewidth]{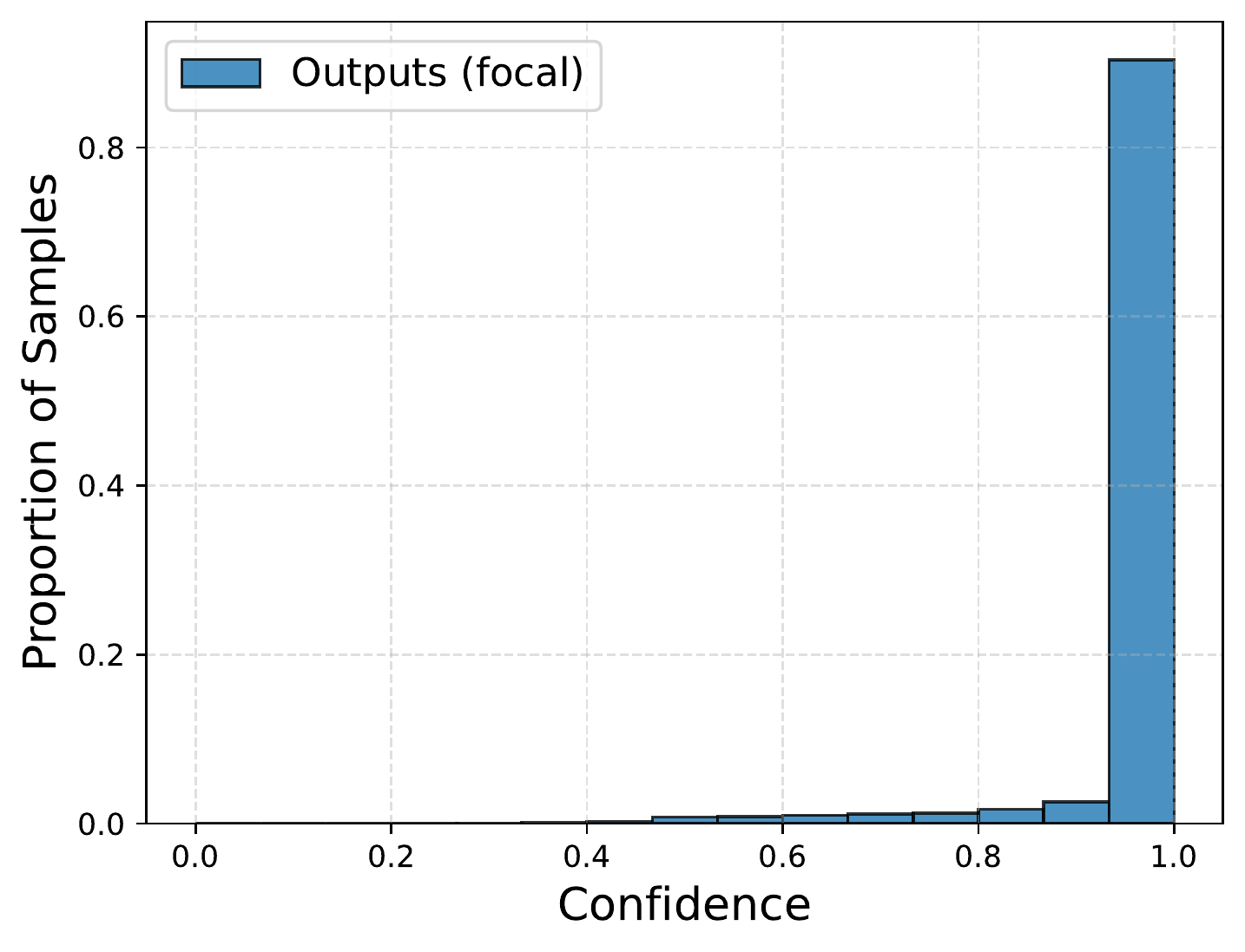}  
  \caption{Focal Loss}
  \label{fig: focal_conf}
\end{subfigure}
\begin{subfigure}{0.24\columnwidth}
  \centering
  \includegraphics[width=0.95\linewidth]{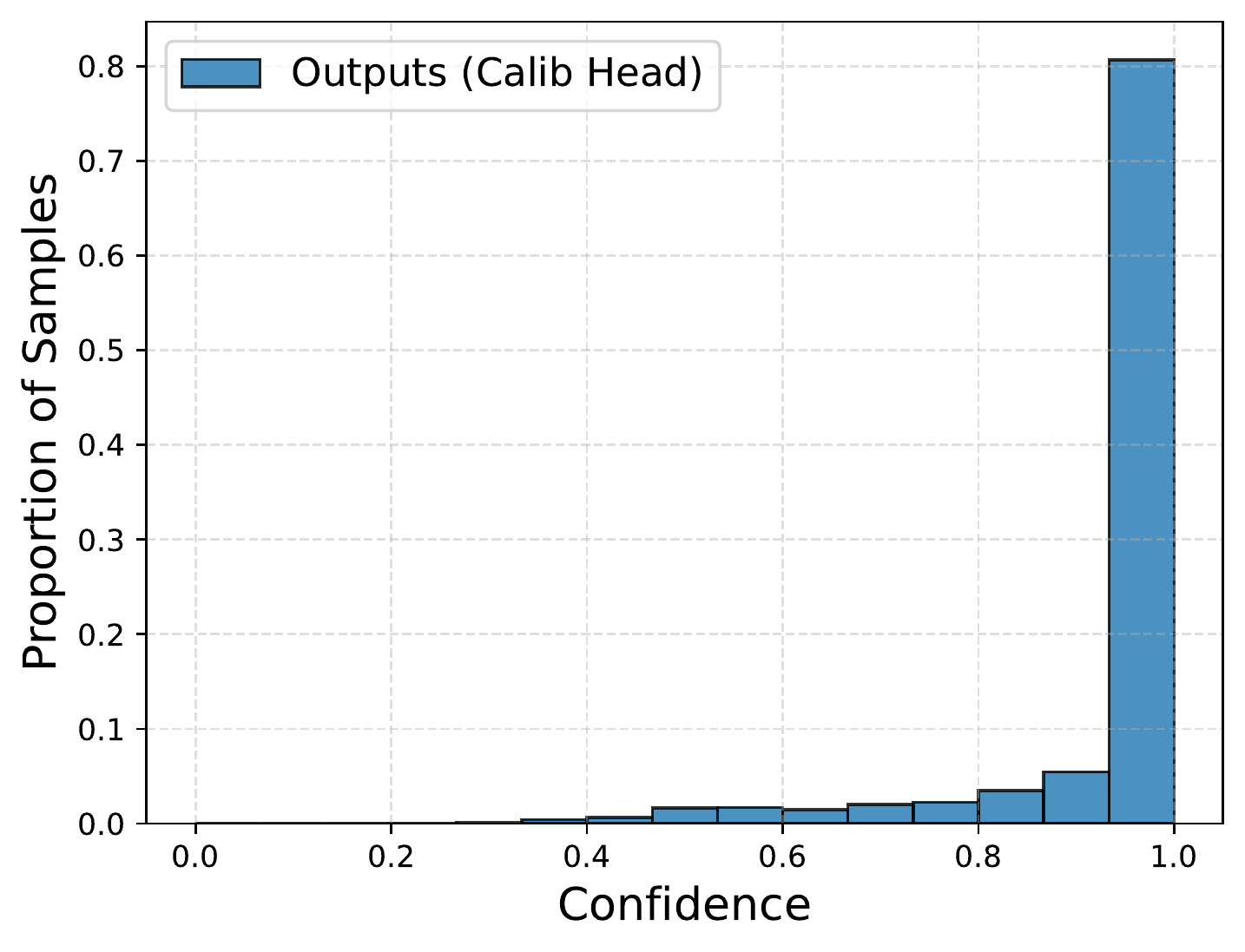} 
  \caption{ADH}
  \label{fig: adh_conf}
\end{subfigure}\\
\centering
\begin{subfigure}{0.24\columnwidth}
  \centering
  \includegraphics[width=0.95\linewidth]{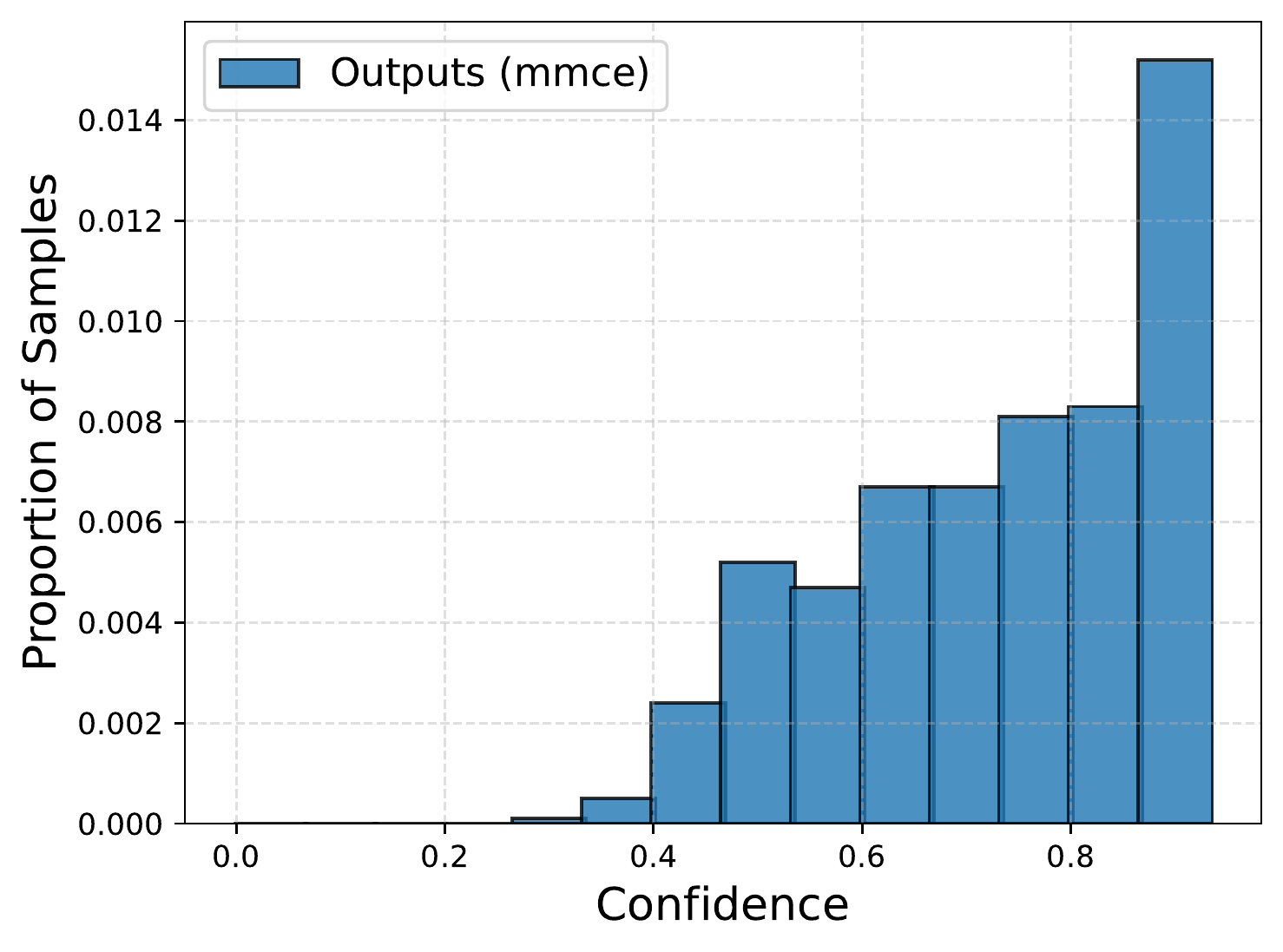}  
  \caption{MMCE}
  \label{fig: mmce_conf_red}
\end{subfigure}
\begin{subfigure}{0.24\columnwidth}
  \centering
  \includegraphics[width=0.95\linewidth]{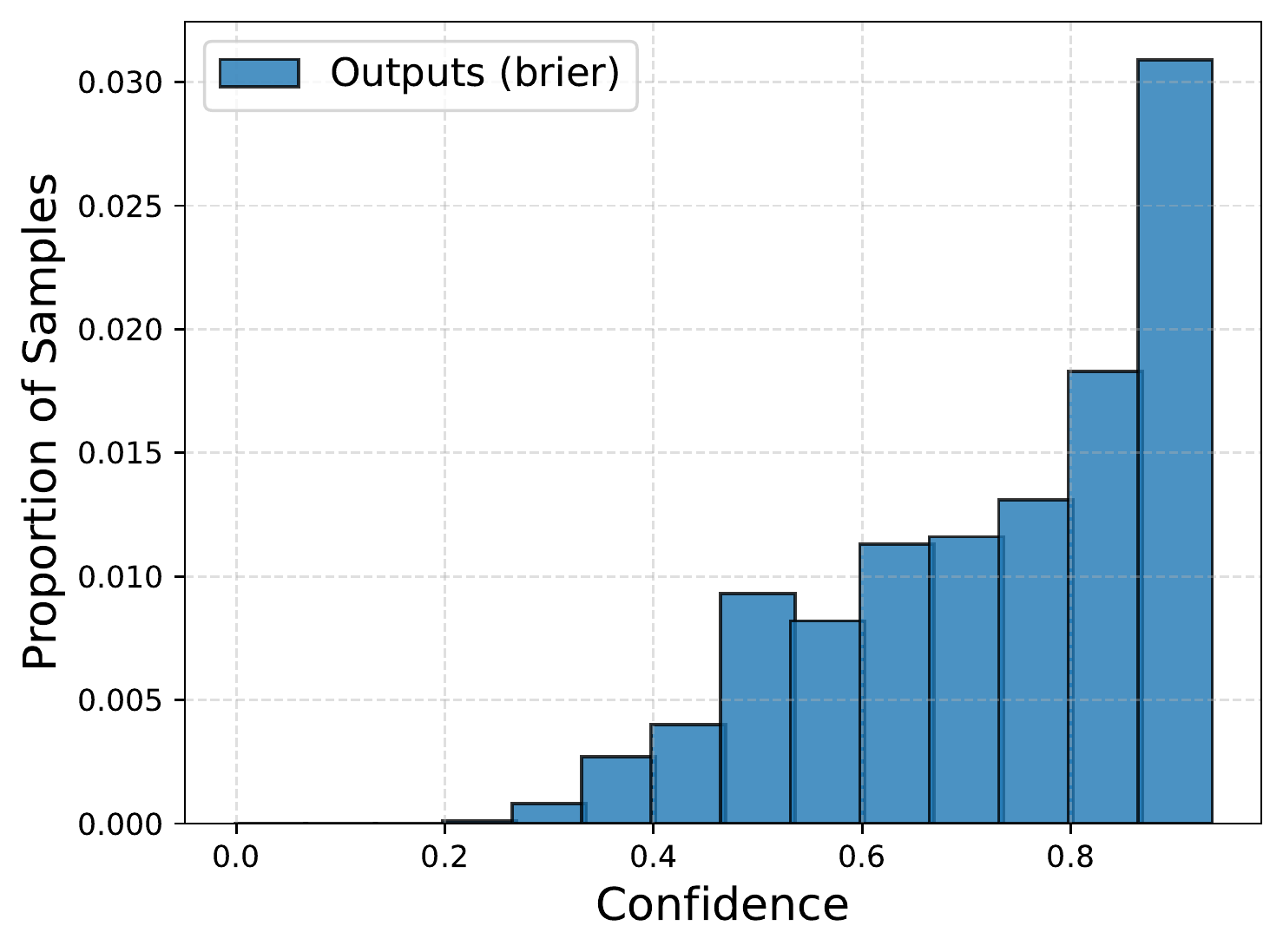} 
  \caption{Brier}
  \label{fig: brier_conf_red}
\end{subfigure}
\begin{subfigure}{0.24\columnwidth}
  \centering
  \includegraphics[width=0.95\linewidth]{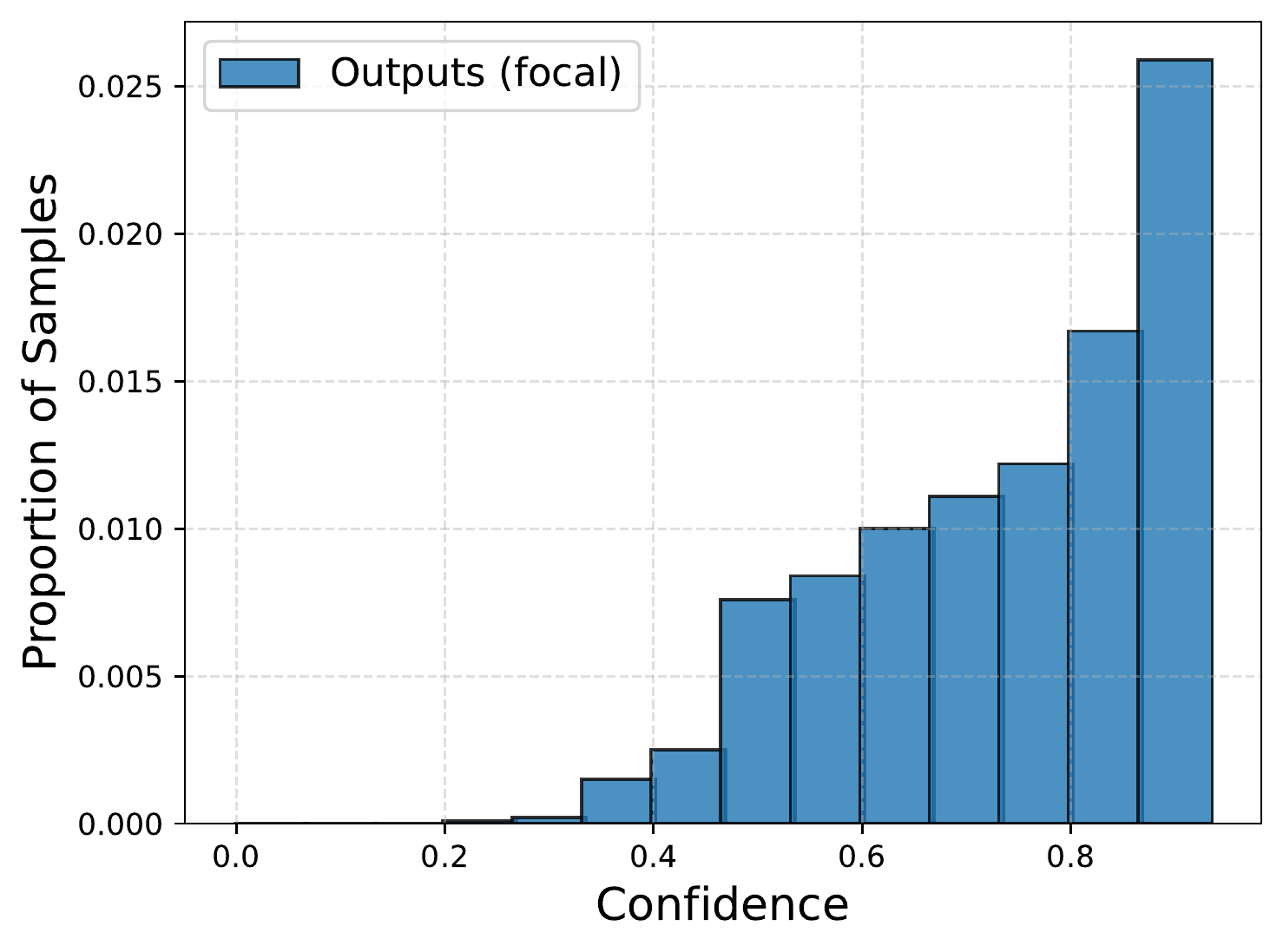}  
  \caption{Focal Loss}
  \label{fig: focal_conf_red}
\end{subfigure}
\begin{subfigure}{0.24\columnwidth}
  \centering
  \includegraphics[width=0.95\linewidth]{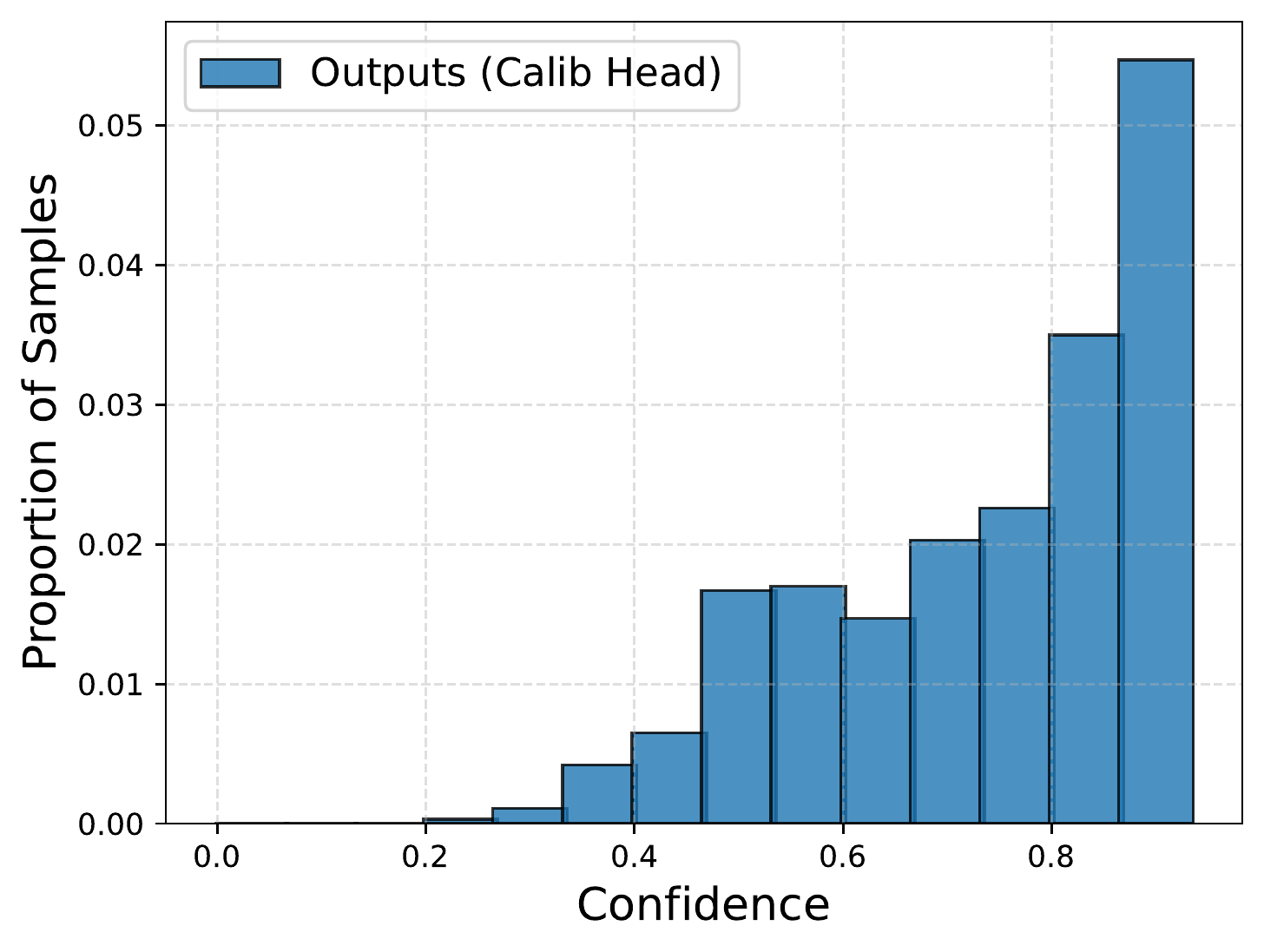} 
  \caption{ADH}
  \label{fig: adh_conf_red}
\end{subfigure}
\caption{The reliability diagrams of MMCE, BL, FL, and ADH are shown in (a), (b), (c), and (d), respectively. The confidence histograms of MMCE, BL, FL, and ADH are plotted in (e), (f), (g), and (h). The last bins of the confidence histograms of aforementioned four methods are eliminated in (i), (j), (k), and (l).}
\label{fig: relib_methods_compare}
\vskip -0.2in
\end{figure}

\subsection{Calibration Head vs Main Head}
\label{subsec: conf_relib_calib_vs_main}
In Fig. \ref{fig: relib_calib_vs_main}, the reliability and confidence histograms of the calibration head and main head of ResNet-50 trained on CIFAR-10 are displayed.
In the first row are plotted reliability diagrams. The histograms of confidence appear in the second row. 
 In the final row of the confidence histograms, the bins with the highest confidences are eliminated. 
 The figures reveals that the calibration head transfers samples from the high confidence bins to the low confidence bins, hence increasing predictive uncertainty.
 
\begin{figure}[!ht]
\vskip 0.2in
\centering
\begin{subfigure}{0.49\columnwidth}
  \centering
  \includegraphics[width=0.8\linewidth]{./results/figures/relib_hist_calib.pdf}  
  \caption{Calibration Head}
  \label{fig: relib_hist_calib_v1}
\end{subfigure}
\begin{subfigure}{0.49\columnwidth}
  \centering
  \includegraphics[width=0.8\linewidth]{./results/figures/relib_hist_main.pdf} 
  \caption{Main Head}
  \label{fig: relib_hist_main_v1}
\end{subfigure}\\
\centering
\begin{subfigure}{0.49\columnwidth}
  \centering
  \includegraphics[width=0.8\linewidth]{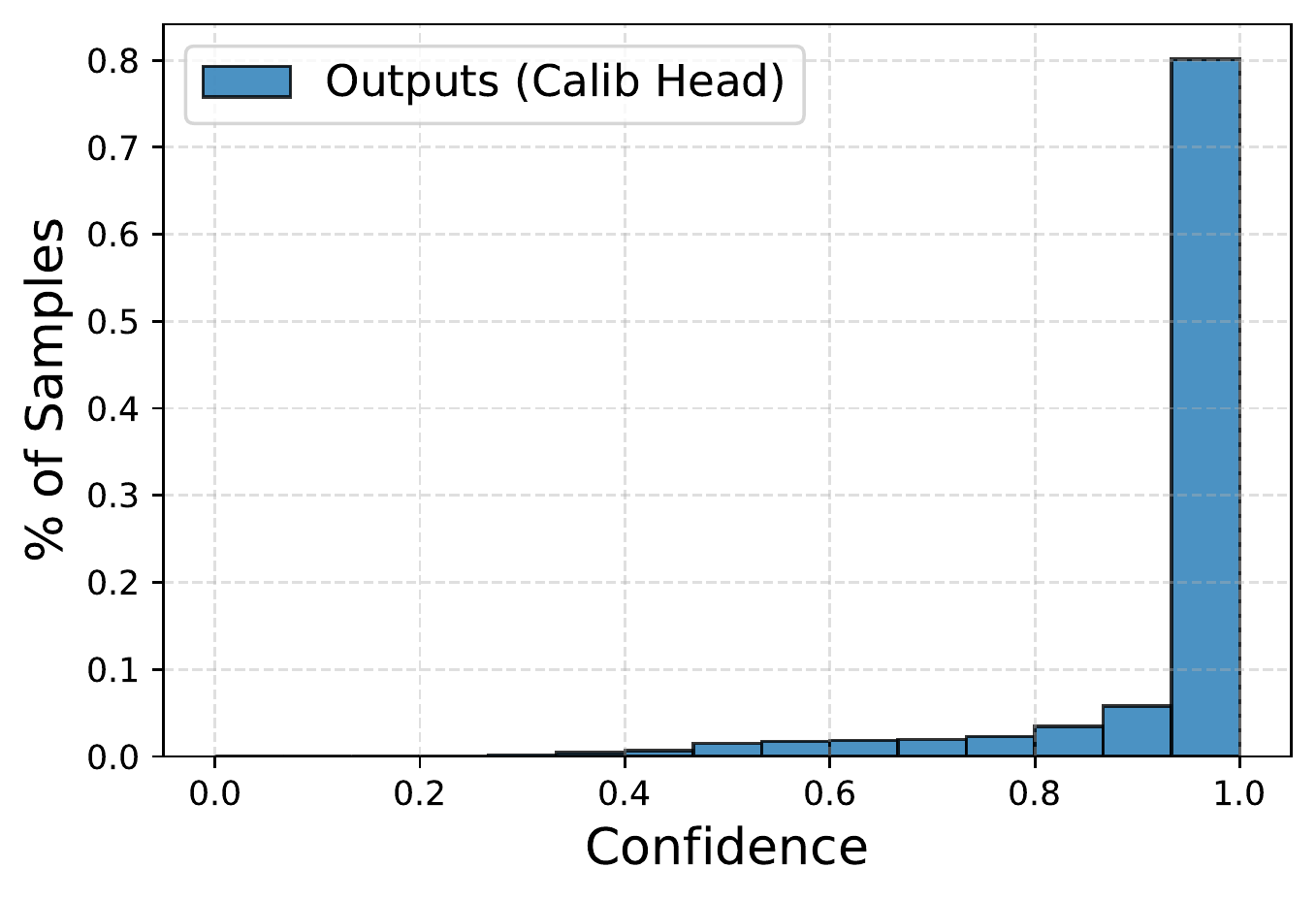}  
  \caption{Calibration Head}
  \label{fig: confid_hist_calib}
\end{subfigure}
\begin{subfigure}{0.49\columnwidth}
  \centering
  \includegraphics[width=0.8\linewidth]{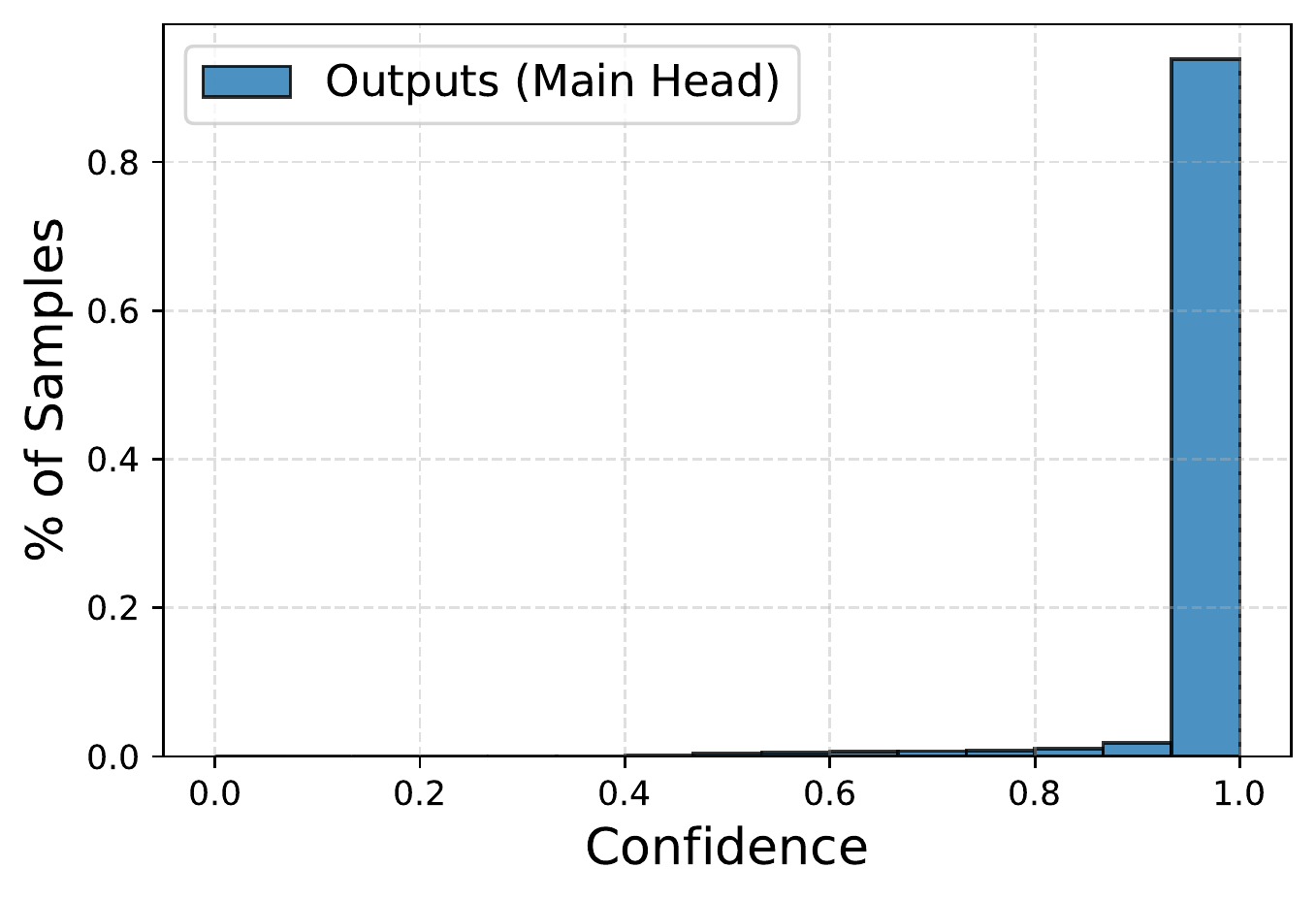} 
  \caption{Main Head}
  \label{fig: conf_hist_main}
\end{subfigure}\\
\centering
\begin{subfigure}{0.49\columnwidth}
  \centering
  \includegraphics[width=0.8\linewidth]{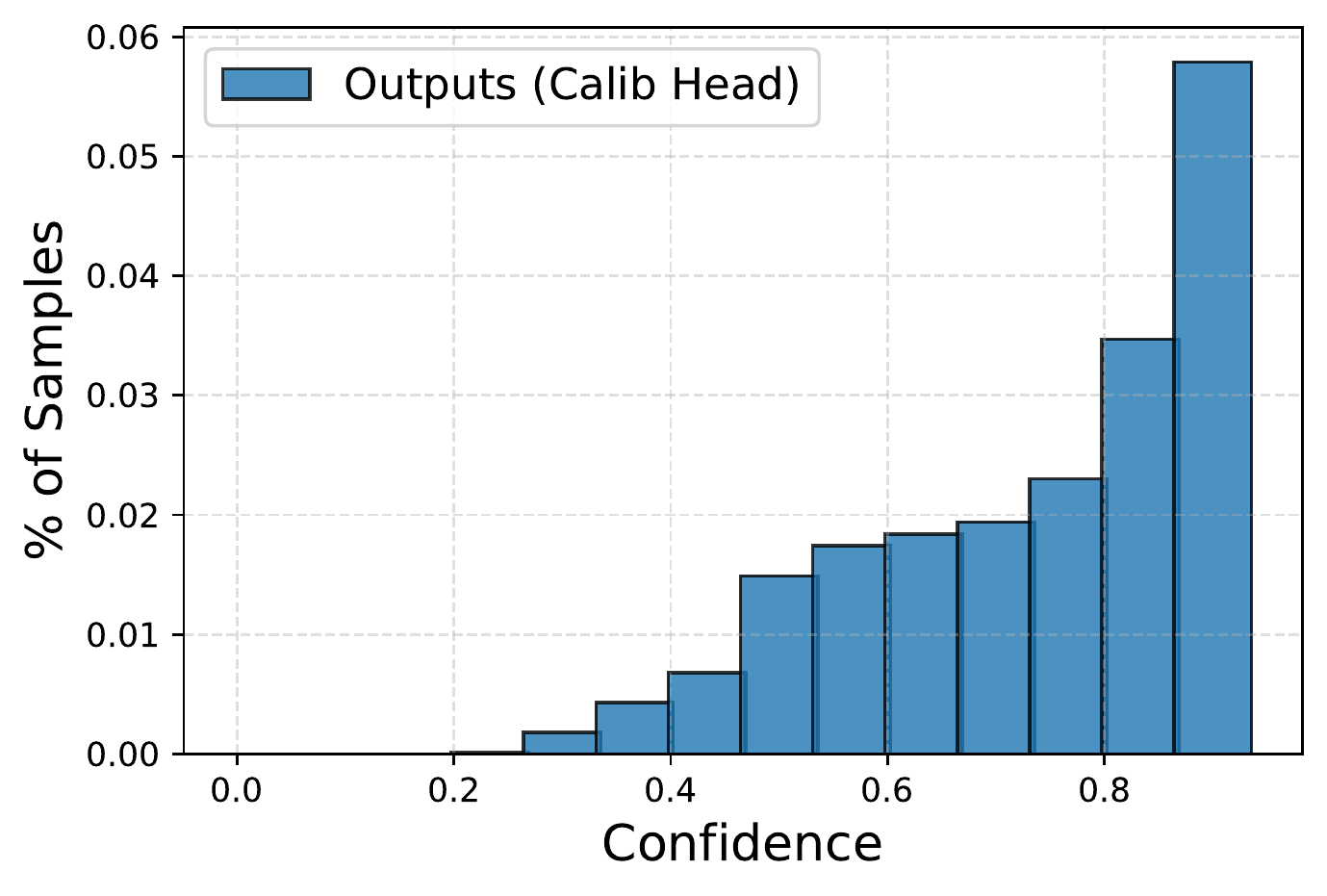}  
  \caption{Calibration Head}
  \label{fig: conf_hist_calib_red}
\end{subfigure}
\begin{subfigure}{0.49\columnwidth}
  \centering
  \includegraphics[width=0.8\linewidth]{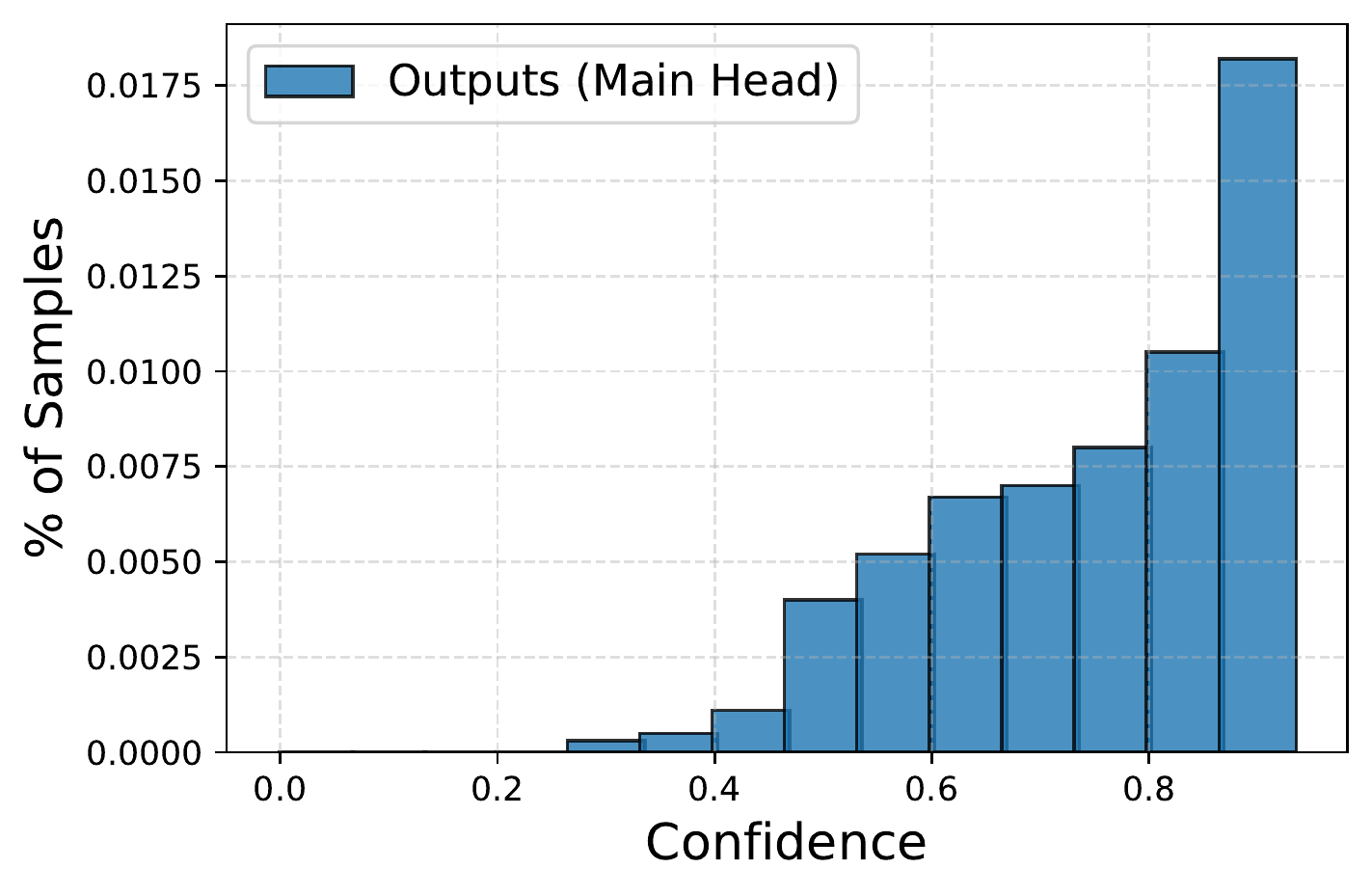} 
  \caption{Main Head}
  \label{fig: conf_hist_main_red}
\end{subfigure}
\caption{(a) and (b) depict the reliability histograms for the calibration head and the main head, respectively. (c) and (d) exhibit, respectively, the confidence histograms of the calibration head and the main head. In (e) and (f), the confidence histograms from (c) and (d) are displayed without the final bin.}
\label{fig: relib_calib_vs_main}
\vskip -0.2in
\end{figure}

\subsection{Annealing}
\label{subsec: varying_beta}
In Fig. \ref{fig: relib_annealing}, we display the reliability histograms and the confidence histograms of calibration head with varies $\beta$ for ResNet-50 trained on CIFAR-10. 
As studied in \cref{subsec: annealing}, an increase in $\beta$ leads to an increase in the entropy of confidence distribution, as more samples from the higher confidence bins are shifted to the lower confidence bins, resulting in higher predictive uncertainty.
The ECE will decrease and then increase as $\beta$ grow, because the model with smaller $\beta$ attempts to be overconfident in its prediction while the model with a larger $\beta$ appears to be underconfident in its prediction as shown in the first row of Fig. \ref{fig: relib_annealing}.

In Fig. \ref{fig: grads_hist}, we display the histograms of the gradient norm in the last (the last $10^{\text{th}}$ step, the last $5^{\text{th}}$ step, and the final step) and initial (the first step, the $5^{\text{th}}$ step, the $10^{\text{th}}$ step) steps from Epoch $60$ to Epoch $200$ to demonstrate that rescaling the logits dynamically leads to the rescaled $2$-norm of gradients, as proved in Theorem \ref{the: annealing_gradient}. 
Consistent with the discussion in \cref{subsec: why_work}: $\beta$ less than $1$ or higher than $1$ will magnify or reduce the gradient norm respectively, Fig. \ref{fig: grads_hist_end} indicates that the distribution of the gradient norm with $\beta=0.6$ resides in the right side of the distribution of the gradient norm with $\beta=1.5$.
From left to right in Fig. \ref{fig: grads_hist_end}, which corresponds to rising $\beta$, we notice that the discrepancy between the two distributions decreases until it nearly disappears, as intended by our design.
However, in the first few steps, factor $\beta$ dominates the rescaling factor since all components of vector $\mathbf{z}$ are almost equally spread, so $\beta$ greater than $1$ or less than $1$ will result in magnified or reduced gradient norm, which is the reverse of the previously mentioned case, as illustrated in Fig. \ref{fig: grads_hist_start}. 

\begin{figure}[!ht]
\vskip 0.2in
\centering
\begin{subfigure}{0.24\columnwidth}
  \centering
  \includegraphics[width=0.95\linewidth]{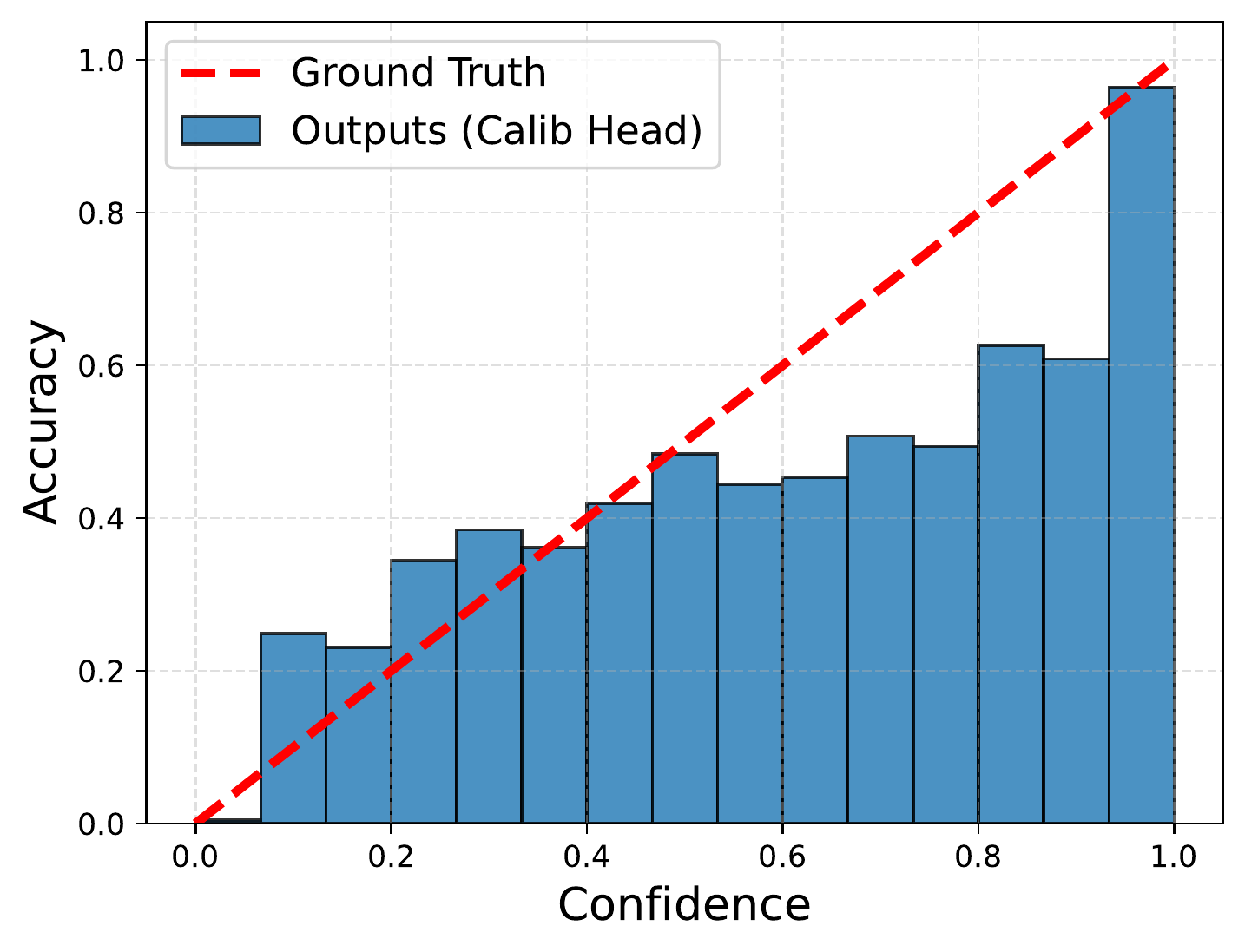}  
  \caption{$\beta = 0.1$}
  \label{fig: relib_beta_0.1}
\end{subfigure}
\begin{subfigure}{0.24\columnwidth}
  \centering
  \includegraphics[width=0.95\linewidth]{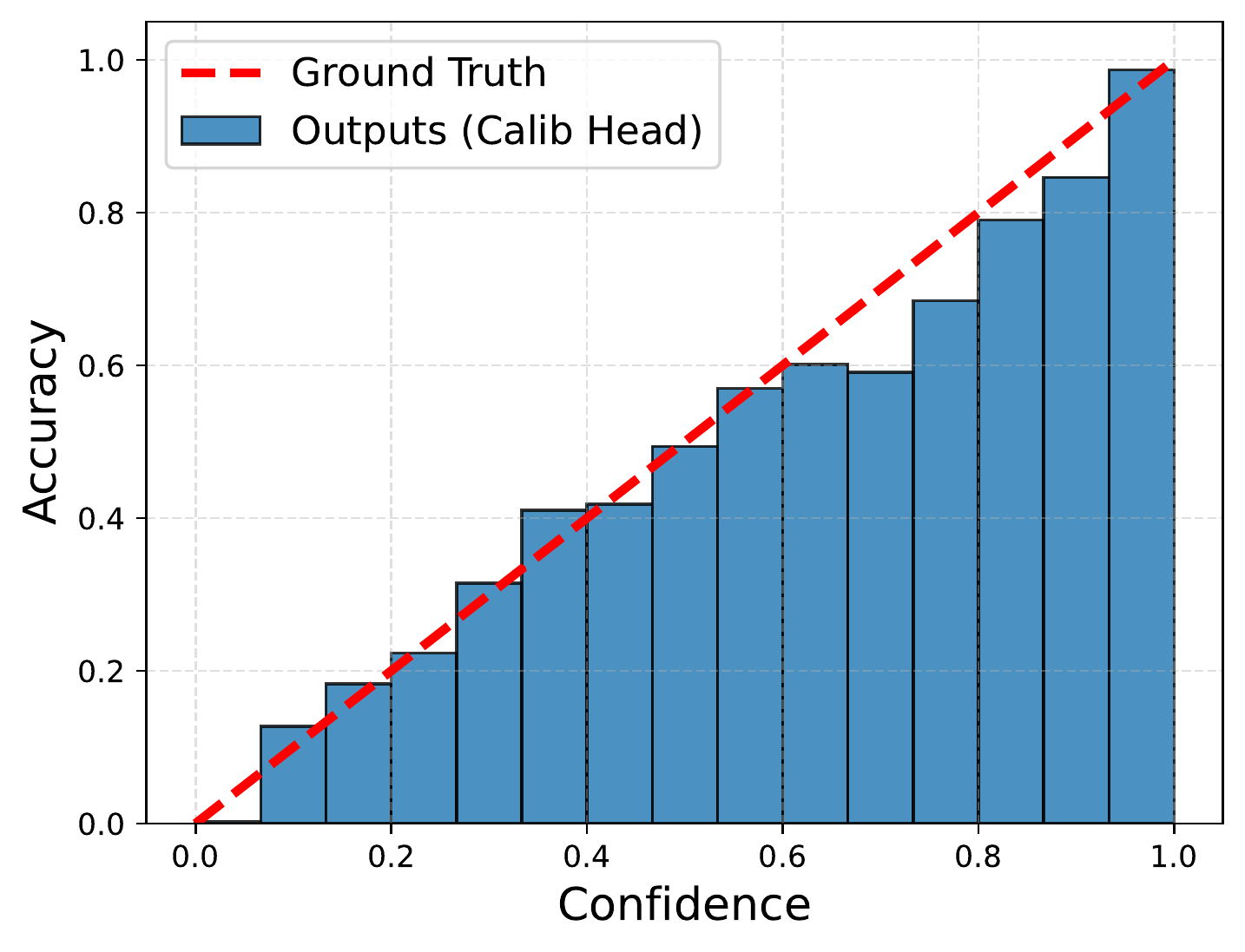} 
  \caption{$\beta = 1$}
  \label{fig: relib_beta_1}
\end{subfigure}
\begin{subfigure}{0.24\columnwidth}
  \centering
  \includegraphics[width=0.95\linewidth]{./results/figures/relib_hist_1.2.pdf}  
  \caption{$\beta = 1.2$}
  \label{fig: relib_beta_1.2}
\end{subfigure}
\begin{subfigure}{0.24\columnwidth}
  \centering
  \includegraphics[width=0.95\linewidth]{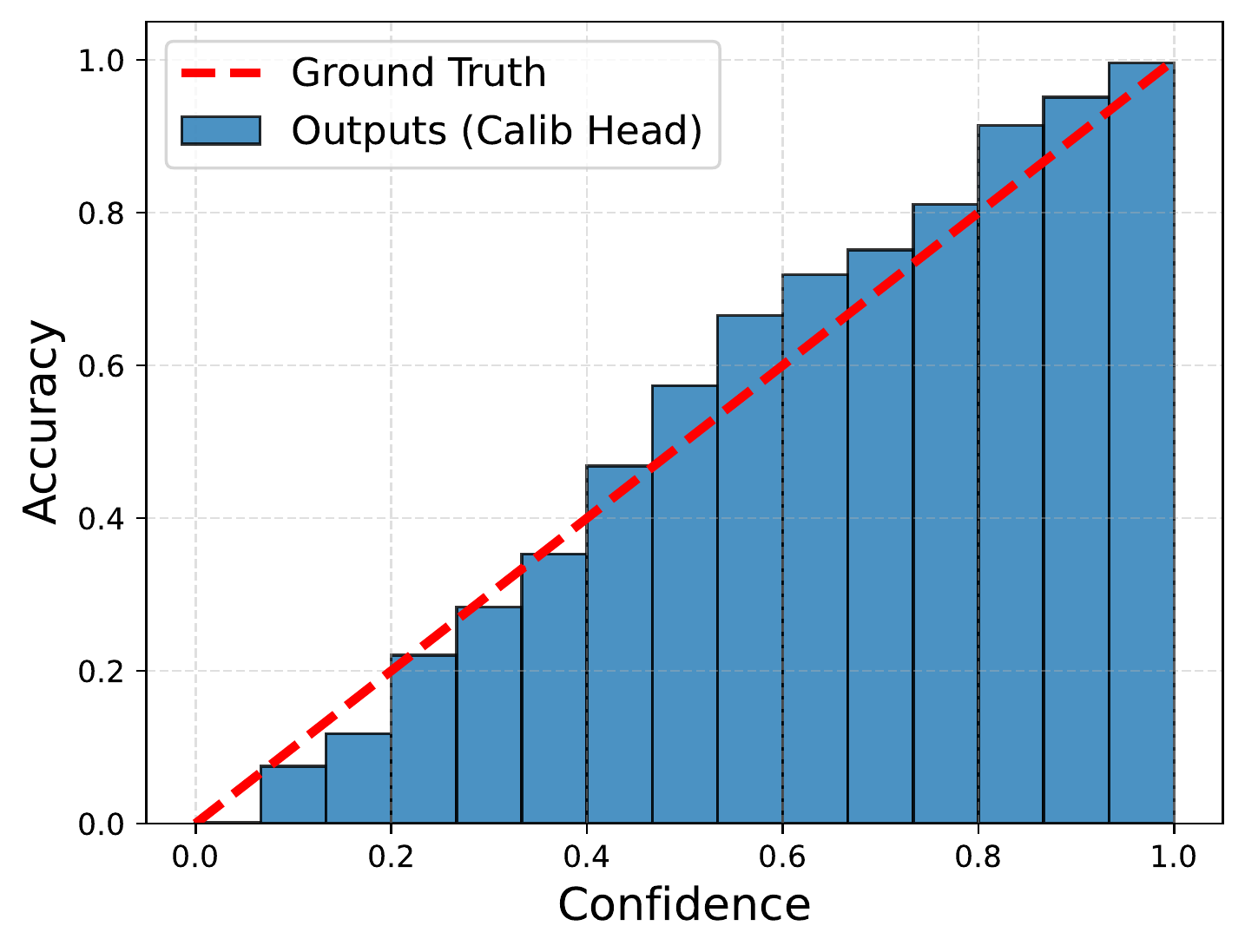} 
  \caption{$\beta=2.0$}
  \label{fig: relib_beta_2.0}
\end{subfigure}\\
\centering
\begin{subfigure}{0.24\columnwidth}
  \centering
  \includegraphics[width=0.95\linewidth]{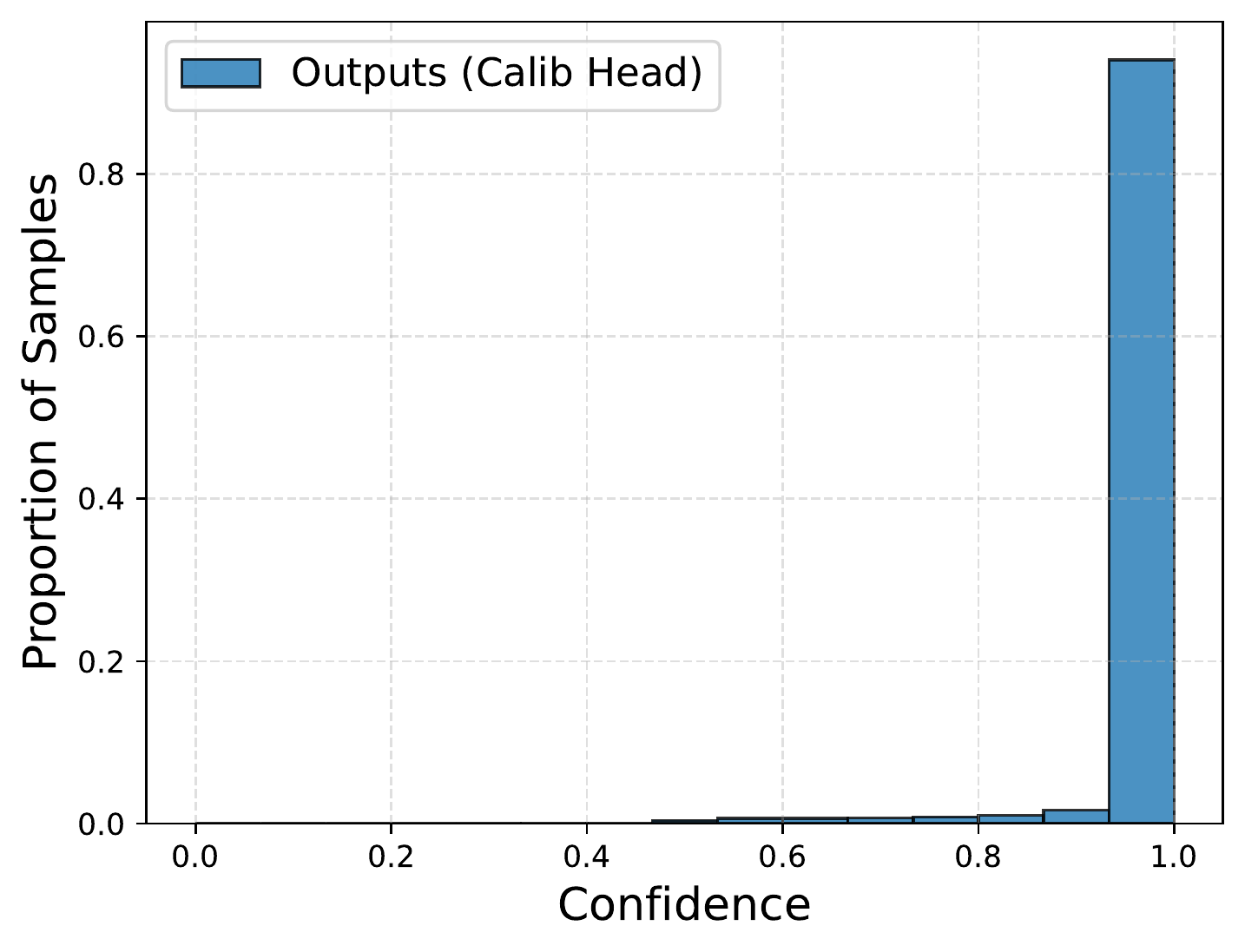}  
  \caption{$\beta = 0.1$}
  \label{fig: confid_beta_0.1}
\end{subfigure}
\begin{subfigure}{0.24\columnwidth}
  \centering
  \includegraphics[width=0.95\linewidth]{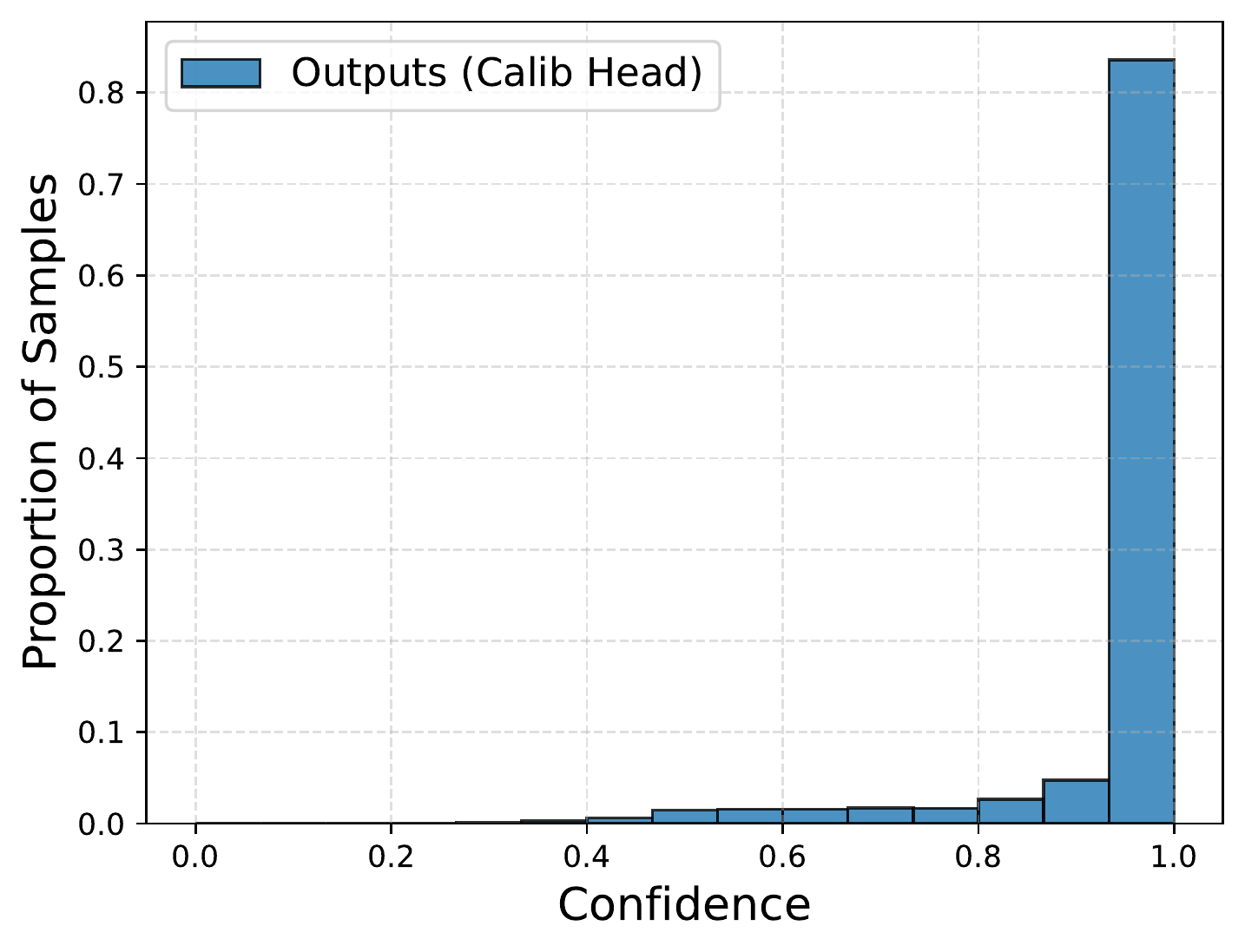} 
  \caption{$\beta = 1$}
  \label{fig: conf_beta_1}
\end{subfigure}
\begin{subfigure}{0.24\columnwidth}
  \centering
  \includegraphics[width=0.95\linewidth]{./results/figures/conf_hist_1.2.pdf}  
  \caption{$\beta = 1.2$}
  \label{fig: confid_beta_1.2}
\end{subfigure}
\begin{subfigure}{0.24\columnwidth}
  \centering
  \includegraphics[width=0.95\linewidth]{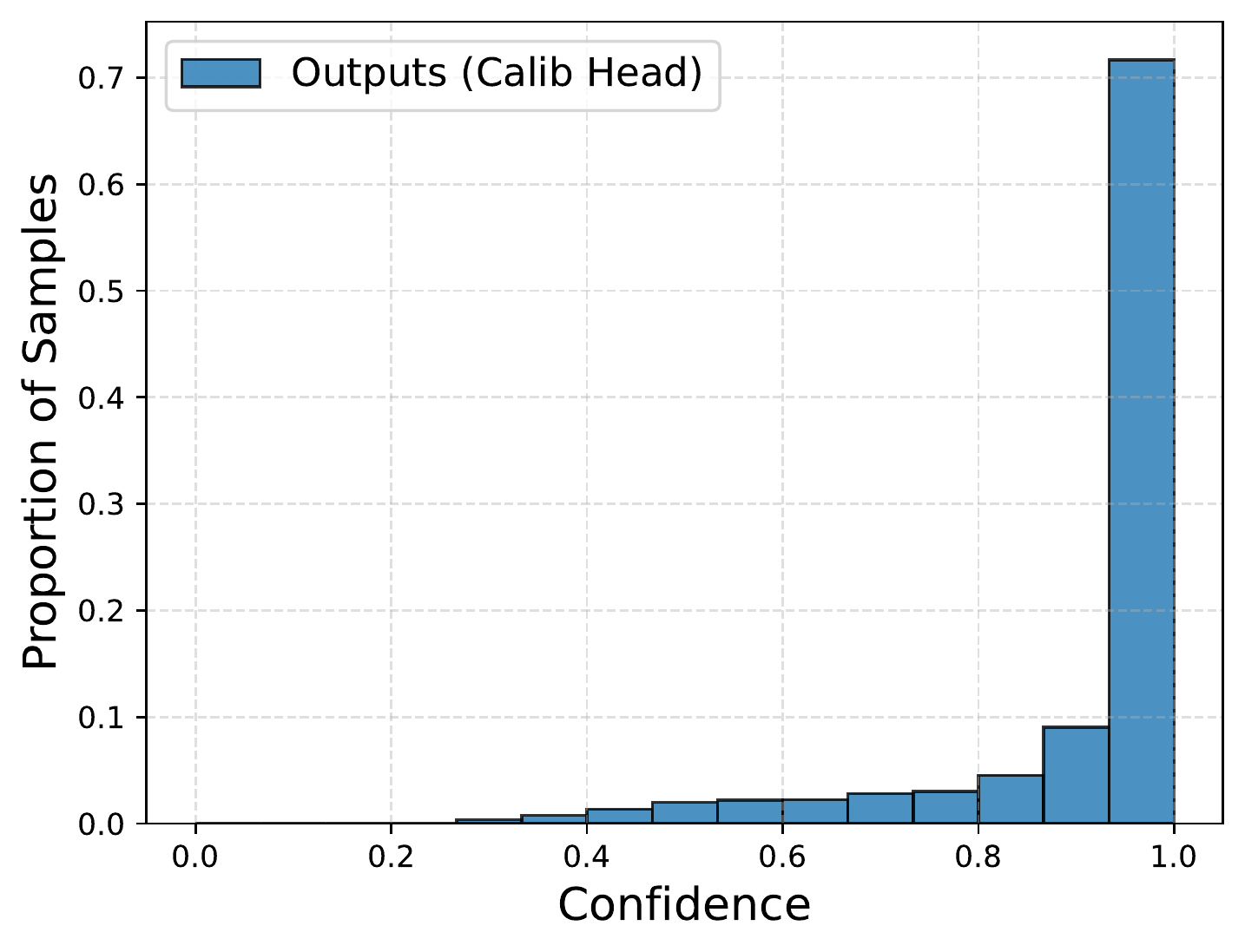} 
  \caption{$\beta=2.0$}
  \label{fig: confid_beta_2.0}
\end{subfigure}\\
\centering
\begin{subfigure}{0.24\columnwidth}
  \centering
  \includegraphics[width=0.95\linewidth]{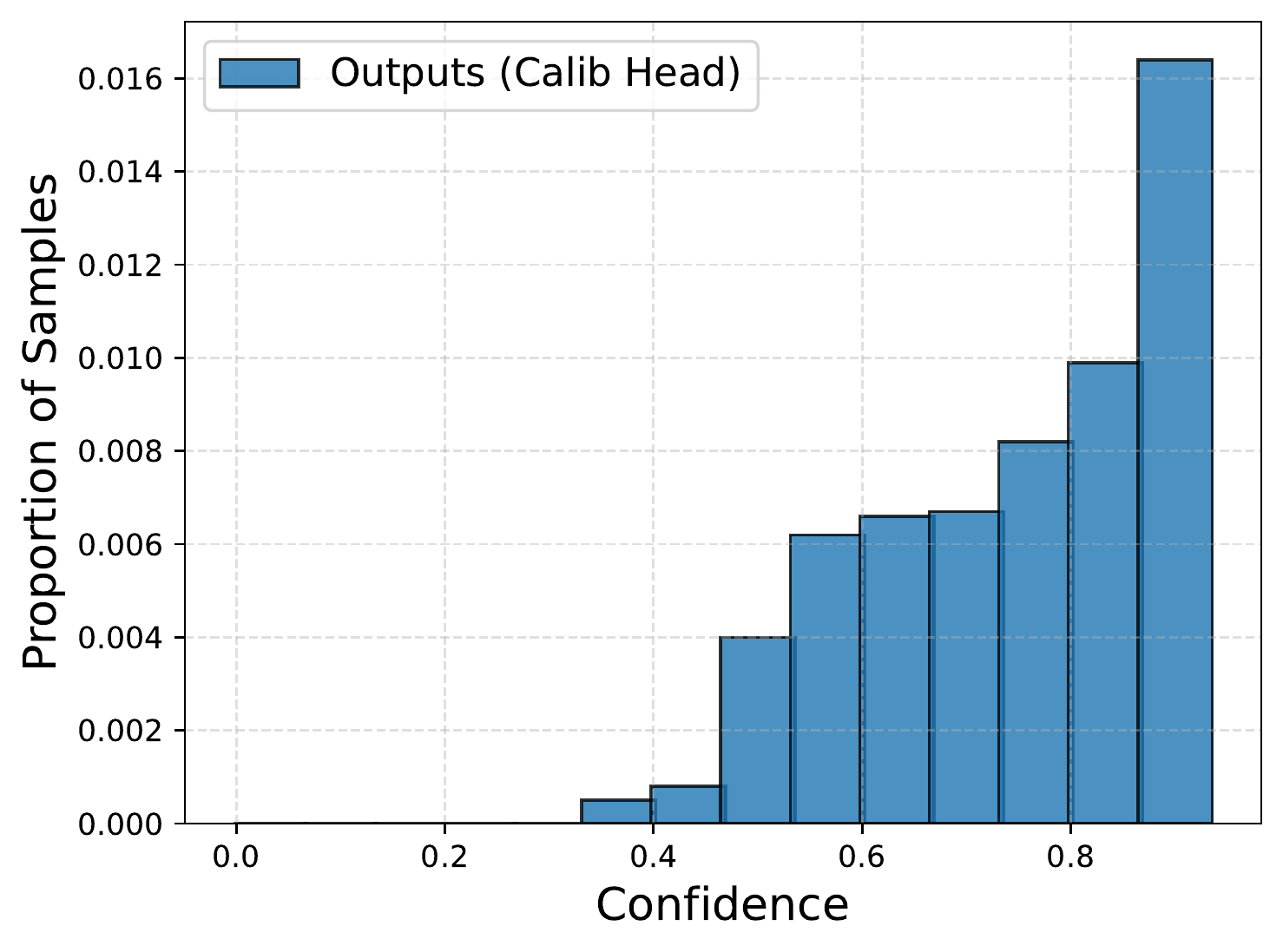}  
  \caption{$\beta = 0.1$}
  \label{fig: confid_beta_0.1_red}
\end{subfigure}
\begin{subfigure}{0.24\columnwidth}
  \centering
  \includegraphics[width=0.95\linewidth]{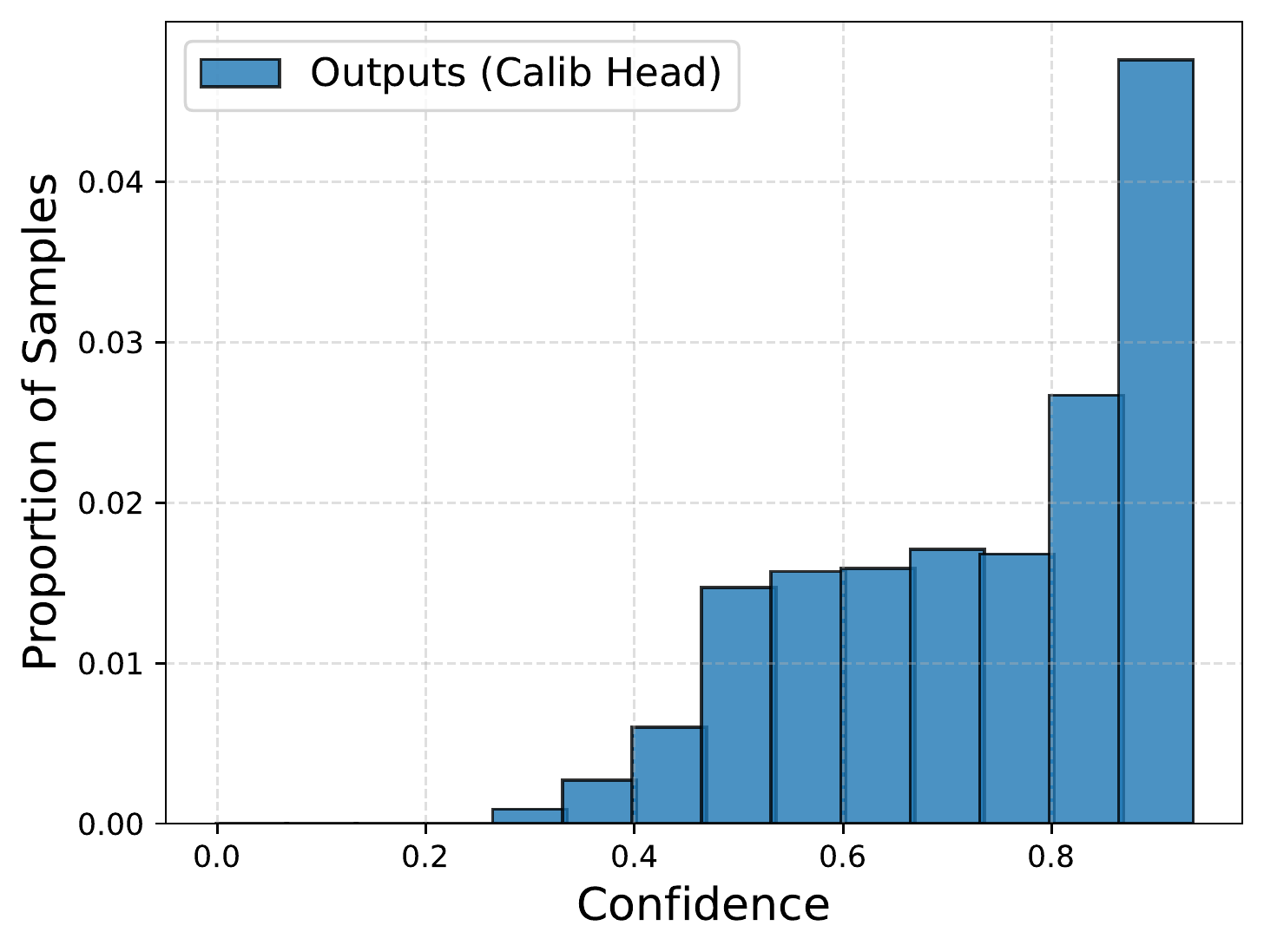} 
  \caption{$\beta = 1$}
  \label{fig: conf_beta_1.0_red}
\end{subfigure}
\begin{subfigure}{0.24\columnwidth}
  \centering
  \includegraphics[width=0.95\linewidth]{./results/figures/conf_hist_1.2_red.pdf}  
  \caption{$\beta = 1.2$}
  \label{fig: conf_beta_1.2_red}
\end{subfigure}
\begin{subfigure}{0.24\columnwidth}
  \centering
  \includegraphics[width=0.95\linewidth]{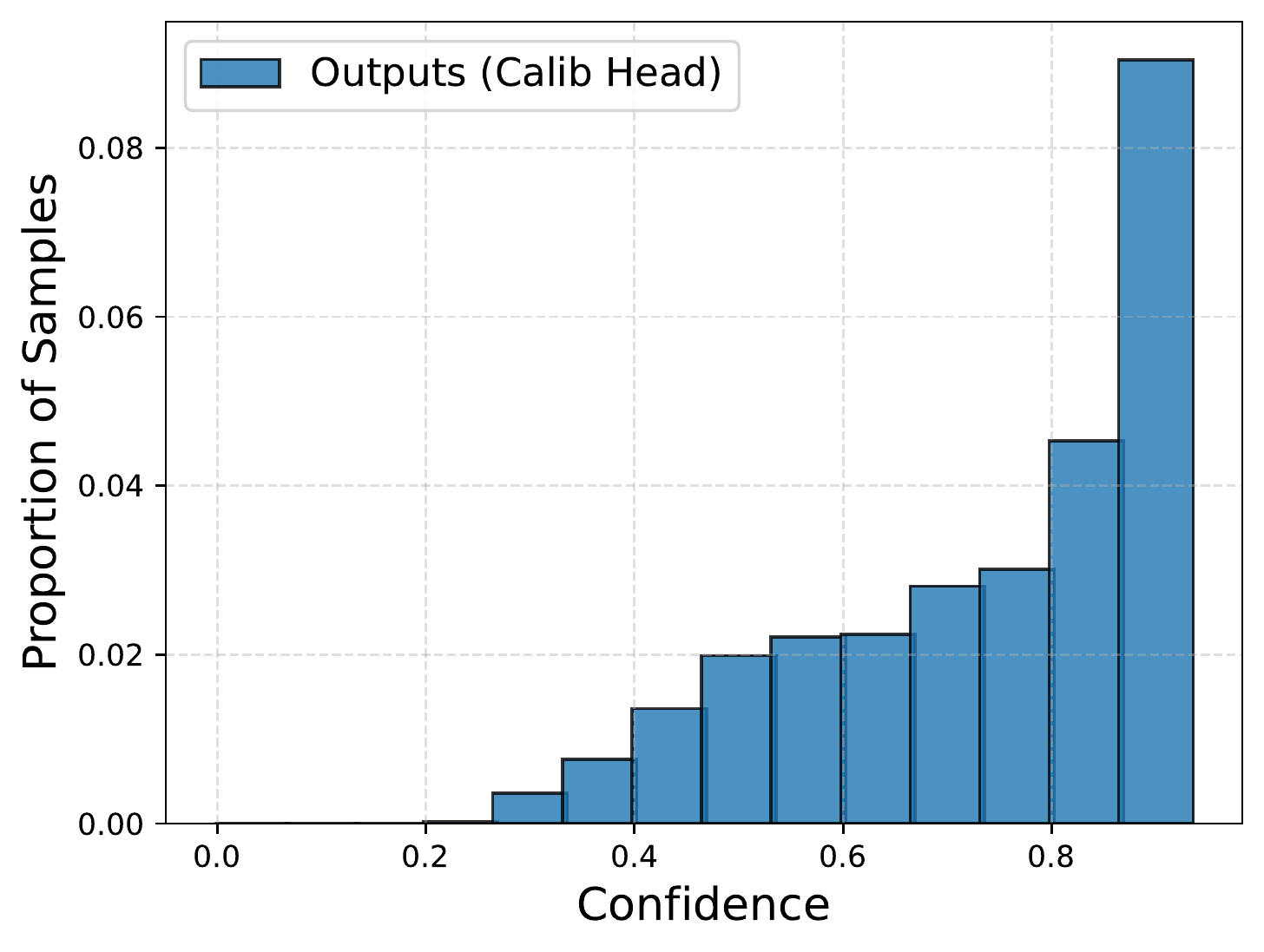} 
  \caption{$\beta=2.0$}
  \label{fig: conf_beta_2.0_red}
\end{subfigure}
\caption{With $\beta$ set to $0.1, $1, $1.2$, and $2.0$, the reliability diagrams for ADH are shown in (a), (b), (c), and (d), respectively. In (e), (f), (g), and (h), we show the confidence histograms for the same $\beta$ list. In (i), (j), (k), and (l), the identical histograms as in (e), (f), (g), and (h), are presented, but the last bin is removed.}
\label{fig: relib_annealing}
\vskip -0.2in
\end{figure}

\begin{figure}[!ht]
\vskip 0.2in
\centering
\begin{subfigure}{0.99\columnwidth}
\includegraphics[width=0.95\linewidth]{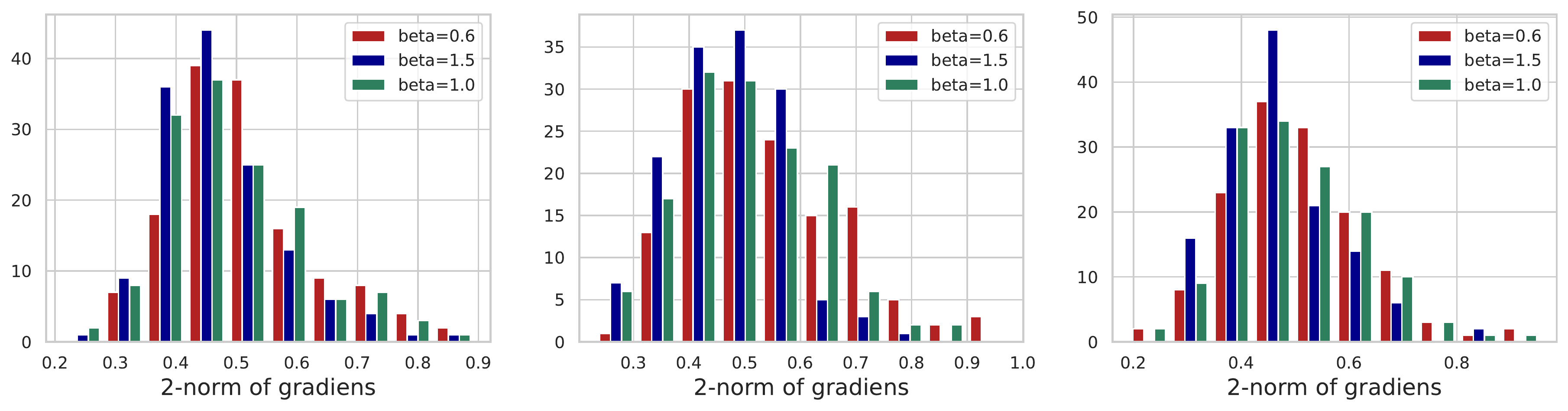}  
  \caption{Left, Middle, Right: Histograms of the $2$-norm of the gradients in the last $10^{\text{th}}$, the last $5^{\text{th}}$, and the final step.}
  \label{fig: grads_hist_end} 
\end{subfigure} \\
\centering
\begin{subfigure}{0.99\columnwidth}
\includegraphics[width=0.95\linewidth]{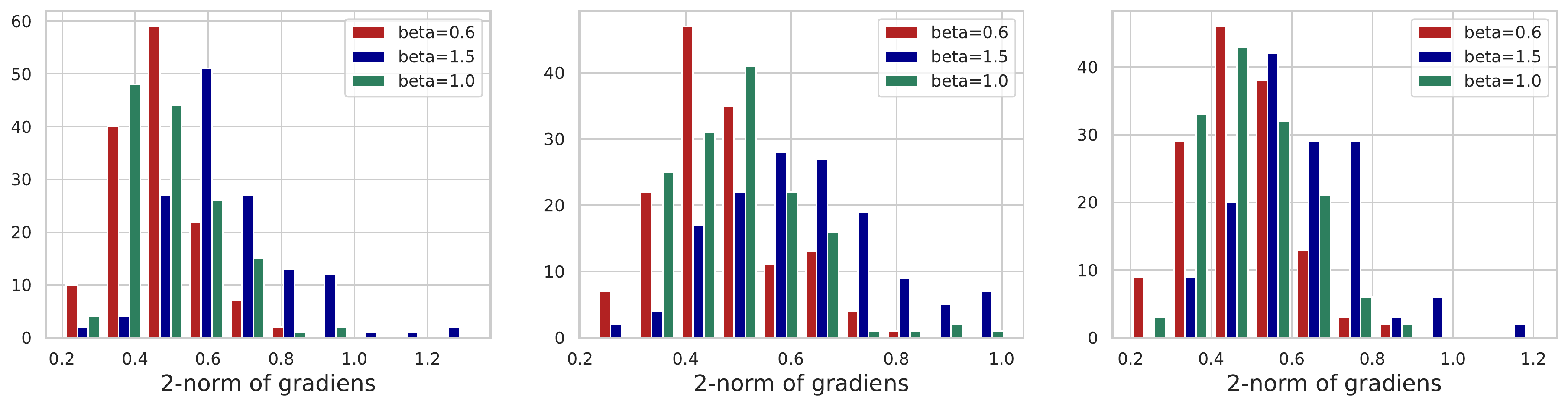}  
\caption{Left, Middle, Right: Histograms of the $2$-norm of the gradients in the first step, the $5^{\text{th}}$ step, and the $10^{\text{th}}$ step.}
\label{fig: grads_hist_start} 
\end{subfigure}
\caption{In this Figure, we present the histogram of the $2$-norm of the gradients of calibration head gathered in one specific step of training from Epoch $60$ to Epoch $200$ for ResNet-50 on CIFAR-10. For (a) and (b), the histograms of the gradients in the final few steps and initial few steps, respectively, are displayed.}
\label{fig: grads_hist}
\vskip -0.2in
\end{figure}

\section{Complexity and Implementation}
\label{subsec: complexity_app}
In Table \ref{tab: complexity}, we show the computation overhead and also ECE of our ADH method for various calibration period $k$ of ResNet-50 on CIFAR-10. It is shown that $k=70$ maintains good performance with only $1.28\%$ of overhead.

We provide a simple implementation of ADH in Pytorch, and about only thirty lines of codes are required as shown in Fig. \ref{fig: code_adh}.

\begin{table*}[ht]
\vskip 0.15in
\begin{center}
\begin{small}
\begin{sc}
\resizebox{0.55\textwidth}{!}
{%
\begin{tabular}{c c c c c c c}
\toprule 
Calibration Period & 10 & 20 & 30 & 40 & 50 & 70 \\
\midrule 
ECE & 0.97 & 0.98 & 1.01 & 0.97 & 1.28 & 1.29 \\
\midrule
Overhead & 8.86 & 3.80 & 2.53 & 1.28 & 1.28 & 1.28 \\
\bottomrule
\end{tabular}%
}
\end{sc}
\end{small}
\end{center}
\caption{This table displays the computation Overhead ($\%$) and ECE ($\%$) of ADH with different Calibration Period $k$ of ResNet-$50$ on CIFAR-10. }
\label{tab: complexity}
\vskip -0.1in
\end{table*}

\begin{figure}[!ht]
\vskip 0.2in
\centering
  \includegraphics[width=0.98\linewidth]{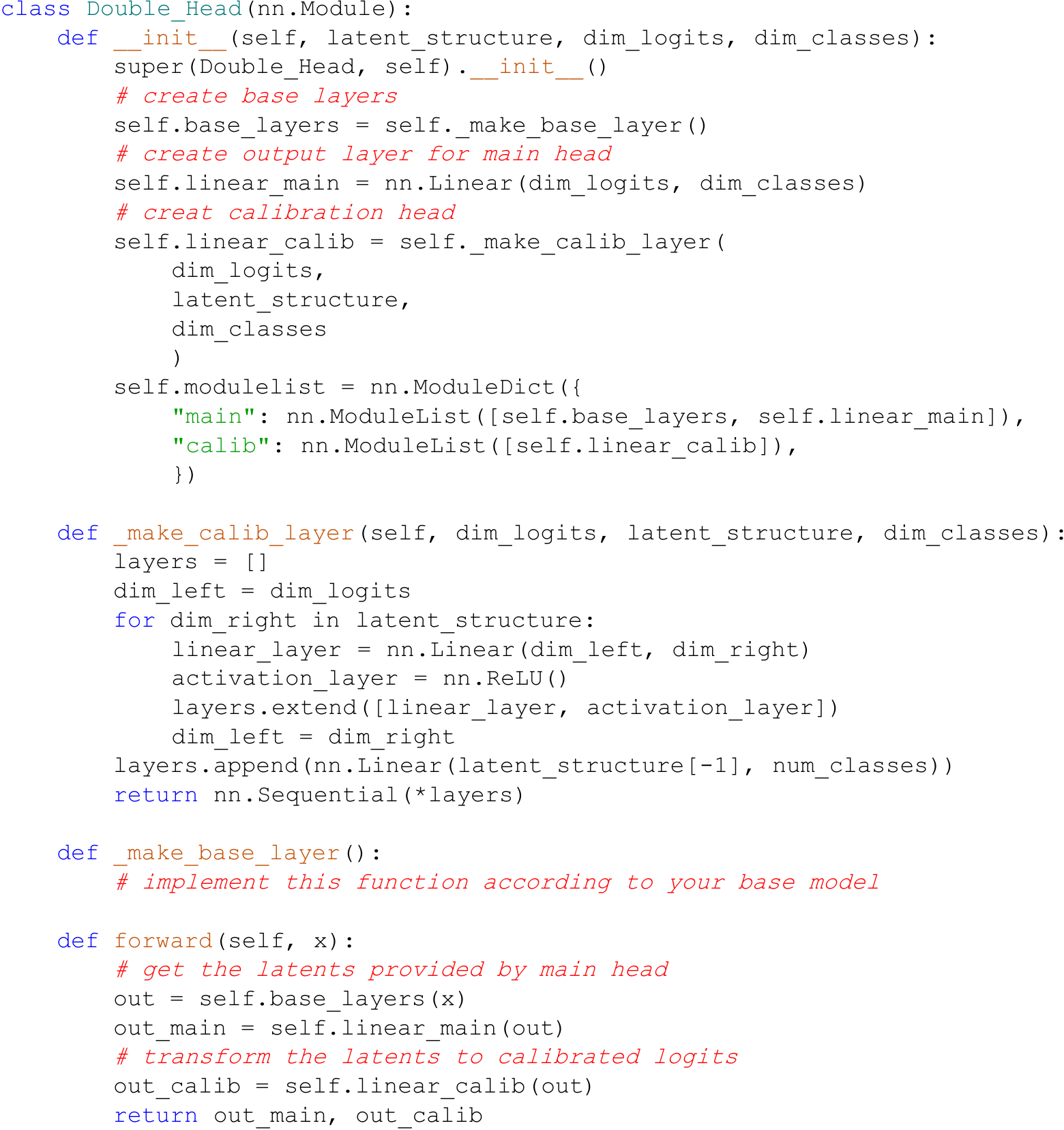}  
  \caption{A thirty-line Pytorch version of our Annealing Double-Head}
  \label{fig: code_adh}
\end{figure}

\section{Remark on ECE computation}
For the population construction of the confidence histogram, we have two options: aggregating all components, refered to as the all components method or just picking the largest component of each confidence vector, also known as the maximal component method. 
We report ECE generated from the two aforementioned implementations for completeness.
In Table \ref{tab: ece_max}, we display the ECE by maximal component method, meanwhile the ECE by all components method is presented in following Table \ref{tab: ece_all}.

\begin{table*}[!t]
\vskip 0.15in
\begin{centering}
\begin{small}
\begin{sc}
\resizebox{\textwidth}{!}{
\begin{tabular}{c c cc cc cc cc cc}
\toprule
\multirow{2}{*}{\textbf{Dataset}} & \multirow{2}{*}{\textbf{Model}} & \multicolumn{2}{c}{\textbf{Cross Entropy}} & \multicolumn{2}{c}{\textbf{MMCE}} & \multicolumn{2}{c}{\textbf{Brier Loss}} & \multicolumn{2}{c}{\textbf{Focal Loss}} & \multicolumn{2}{c}{\textbf{Double-Head}} \\
 &  & pre T & post T & pre T & post T & pre T & post T & pre T & post T & No Anneal & Anneal \\
\midrule
\multirow{4}{*}{CIFAR-10} & ResNet 50 & 8.43 & 1.64 (2.0) & 6.45 & 1.97 (1.6) & 4.95 & 3.10 (1.1) & 5.39 & 1.91 (1.3) & 2.70 & \bf{0.94} \\
& ResNet 101 & 8.04 & 1.38 (2.6) & 7.44 & 1.92 (2.2) & 5.36 & 1.97 (1.2) & 5.73 & 2.00 (1.3) & 2.16 & \bf{1.04} \\
& DenseNet 121 & 6.68 & 2.38 (1.5) & 6.13 & 2.26 (1.4) & 2.11 & 2.11 (1.0) & 5.28 & 2.28 (1.2) & 2.24 & \bf{1.88} \\
& Wide ResNet 28-10 & 6.58 & 2.08 (1.5) & 4.70 & 1.97 (1.4) & 2.02 & 2.02 (1.0) & 1.72 & 1.72(1.0) & 1.50 & \bf{1.35}\\
 \midrule
\multirow{4}{*}{CIFAR-100} & ResNet 50 & 16.15 & 7.29 (1.3) & 22.80 & 7.24 (1.5) & 10.39 & 6.31 (1.1) & 15.10 & 4.28 (1.2) & 2.52 & \bf{2.38}\\
& ResNet 101 & 20.39 & 6.58 (1.4) & 20.21 & 5.92 (1.4) & 8.50 & 8.50 (1.0) & 19.42 & 7.46 (1.2) & \bf{1.75} & 2.29\\
& DenseNet 121 & 17.46 & 7.38 (1.2) & 17.45 & 7.66 (1.3) & 9.81 & 8.56 (1.1) & 12.15 & 6.14 (1.2) & 4.79 & \bf{2.62}\\
& Wide ResNet 28-10 & 8.41 & 8.41 (1.0) & 11.10 & 7.63 (1.1) & 7.92 & 7.92 (1.0) & 7.27 & \bf{2.62} (1.6) & 3.11 & 3.04\\ \midrule
\multirow{4}{*}{SVHN} & ResNet 50 & 5.26 & 1.61 (1.6) & 4.93 & \bf{1.18} (1.6) & 1.80 & 1.80 (1.0) & 3.86 & 2.24 (1.2) & 1.31 & 1.30\\
& ResNet 101 & 5.59 & 1.78 (1.8) & 5.66 & 1.72 (1.6) & 1.63 & 1.63 (1.0) & 4.43 & 1.47 (1.2) & \bf{0.91} & 1.04 \\
& DenseNet 121 & 3.94 & 1.76 (1.3) & 4.68 & 1.57 (1.4) & 2.75 & 1.88 (0.9) & 2.69 & 2.0 (1.1) & 2.61 & \bf{0.74}\\
& Wide ResNet 28-10 & 4.03 & 1.83 (1.4) & 3.75 & 1.76 (1.2) & 2.23 & 1.07 (0.9) & 1.57 & 1.57 (1.0) & 1.07 & \bf{0.87}\\ \midrule
SST Fine Grained & TreeLSTM & 5.53 & 2.80 (1.2) & 5.26 & 5.26 (1.0) & 10.15 & 6.44 (0.8) & 10.21 & 8.15 (0.8) & 2.28 & \bf{1.55}\\
SST Binary & GP CNN & 14.68 & 3.21 (1.7) & 8.74 & 3.19 (1.4) & 7.50 & 3.30 (1.1) & 4.02 & 3.34 (1.1) & \bf{1.42} & 2.0\\
20 Newsgroups & GP CNN & 6.90 & \bf{3.70} (1.1) & 6.75 & 5.18 (1.1) & 8.07 & 4.74 (0.8) & 11.03 & 6.62 (1.1) & 5.59 & 4.92\\ 
\bottomrule
\end{tabular}
}


\end{sc}
\end{small}
\end{centering}
\caption{ECE ($\%$) of multiple methods: CE, MMCE, Brier, FL, and ours with and without TS on different model and dataset combinations is reported. The used temperature is specified between brackets. In the last two columns, the ECE of our own approach without Annealing and with Annealing is displayed.}
\label{tab: ece_all}
\vskip -0.1in
\end{table*}

\end{document}